\documentclass[mathpazo]{csam}
\usepackage{tabularx}
\usepackage{color} 
\usepackage{diagbox}
\usepackage{graphics}
\usepackage{multirow}
\usepackage{bm}
\usepackage{float} 
\usepackage{amsmath}
\usepackage{epsfig}
\usepackage{amssymb,amsfonts}
\usepackage{epstopdf}
\usepackage{tikz}
\usepackage{pgflibraryarrows}
\usepackage{pgflibrarysnakes}

\usepackage{mlmath}
\usepackage{color}
\usepackage{graphicx}
\usepackage{subfigure}
\usepackage{enumerate}
\usepackage{amsmath,amsfonts,amssymb,mathtools,amsthm}
\usepackage{algorithm}
\usepackage{algorithmic}

\usepackage[utf8]{inputenc} 
\usepackage[T1]{fontenc}    
\usepackage{hyperref}       
\usepackage{url}            
\usepackage{booktabs}       
\usepackage{amsfonts}       
\usepackage{nicefrac}       
\usepackage{microtype}      
\usepackage{xcolor}         

\newtheorem*{prop*}{Proposition}

\newtheorem*{theorem*}{Theorem}
\newtheorem*{lemma*}{Lemma}

\newtheorem*{property*}{Property}
\newtheorem*{definition*}{Definition}
\newtheorem*{corollary*}{Corollary}
\newtheorem{assumption}{Assumption}
\numberwithin{assumption}{section}

\newcommand{\rdeep}{\textnormal{deep}}
\newcommand{\rshal}{\textnormal{shal}}

\newcommand{\rlow}{\textnormal{low}}
\newcommand{\rup}{\textnormal{up}}
\newcommand{\rmin}{\textnormal{min}}
\newcommand{\rmax}{\textnormal{max}}

\makeatletter
\newcommand*\bigcdot{\mathpalette\bigcdot@{.5}}
\newcommand*\bigcdot@[2]{\mathbin{\vcenter{\hbox{\scalebox{#2}{$\m@th#1\bullet$}}}}}
\makeatother

\begin{document}
\title{Embedding Principle in Depth for the Loss Landscape Analysis of Deep Neural Networks}

\author[ ]{Zhiwei Bai
\affil{1}, Tao Luo\affil{1, 2}, Zhi-Qin John Xu\affil{1}\corrauth ~and Yaoyu Zhang\affil{1, 3}\secondcorrauth}
\address{
\affilnum{1}\ School of Mathematical Sciences, Institute of Natural Sciences, \\MOE-LSC, Shanghai Jiao Tong University, Shanghai 200240, P.R. China.\\
\affilnum{2}\ CMA-Shanghai, Shanghai
Artificial Intelligence Laboratory, Shanghai 200240, P.R. China\\
\affilnum{3}\ Shanghai Center for Brain Science and Brain-Inspired Technology, Shanghai 200240, P.R. China\\
}
\emails{{\tt bai299@sjtu.edu.cn} (Z.~Bai), {\tt luotao41@sjtu.edu.cn} (T.~Luo), {\tt xuzhiqin@sjtu.edu.cn} (Z.~Xu), {\tt zhyy.sjtu@sjtu.edu.cn} (Y.~Zhang)}


\begin{abstract}
In this work, we delve into the relationship between deep and shallow neural networks (NNs), focusing on the critical points of their loss landscapes.  We discover an embedding principle in depth that loss landscape of an NN ``contains'' all critical points of the loss landscapes for shallower NNs.  The key tool for our discovery is the critical lifting that maps any critical point of a network to critical manifolds of any deeper network while preserving the outputs. To investigate the practical implications of this principle, we conduct a series of numerical experiments. The results confirm that deep networks do encounter these lifted critical points during training, leading to similar training dynamics across varying network depths. We provide theoretical and empirical evidence that through the lifting operation, the lifted critical points exhibit increased degeneracy. This principle also provides insights into the optimization benefits of batch normalization and larger datasets, and enables practical applications like network layer pruning. Overall, our discovery of the embedding principle in depth uncovers the depth-wise hierarchical structure of deep learning loss landscape, which serves as a solid foundation for the further study about the role of depth for DNNs.

\end{abstract}

\ams{68T07
}
\keywords{Deep learning, loss landscape, embedding principle.}

\maketitle

\section{Introduction}
Deep neural networks (DNNs) have achieved remarkable success in various fields, such as computer vision~\cite{rastegari2016xnor}, natural language processing~\cite{collobert2008unified}, and numerous scientific computing applications~\cite{han2018solving, cai2020deep, wang2020mesh}. Despite their widespread adoption and empirical achievements, our theoretical understanding of DNNs, particularly regarding their loss landscape and training dynamics, remains limited. The loss landscape of a DNN essentially characterizes the optimization problem encountered during the network's training process. The study of this landscape is of paramount importance as it directly influences not only the efficiency and final outcome of the training process, but also the generalization in overparametered case. Regrettably, the high-dimensionality and non-convex nature of DNNs render their loss landscapes notoriously challenging to comprehend and navigate. The recent discovery of the embedding principle~\cite{zhang2021embedding,embedding2021long,fukumizu2019semi,csimcsek2021geometry} offers insights for analyzing the loss landscape of networks and establishes connections between the loss landscapes of neural networks with varying widths. However, considering the extreme importance of depth for DNNs, it prompts us to question whether a relationship exists between the loss landscapes of networks with different depths. In this paper, we strive to address this fundamental question by conducting a thorough analysis of critical points across varying network depths.

\begin{figure}[ht]
	\centering
\subfigure[Loss (Iris)]{\includegraphics[height=0.27\textwidth]{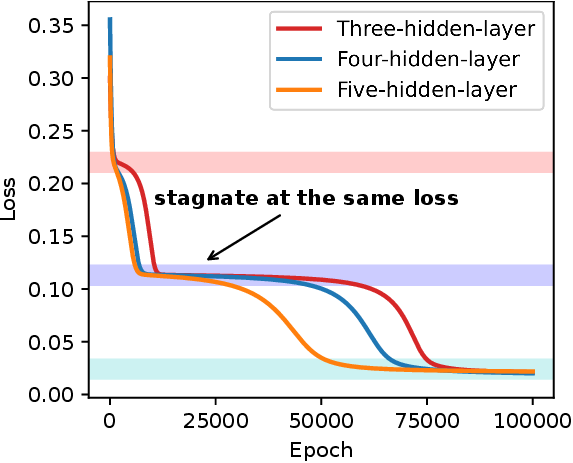}}\hspace{1.5cm}
\subfigure[Accuracy (Iris)]{\includegraphics[height=0.27\textwidth]{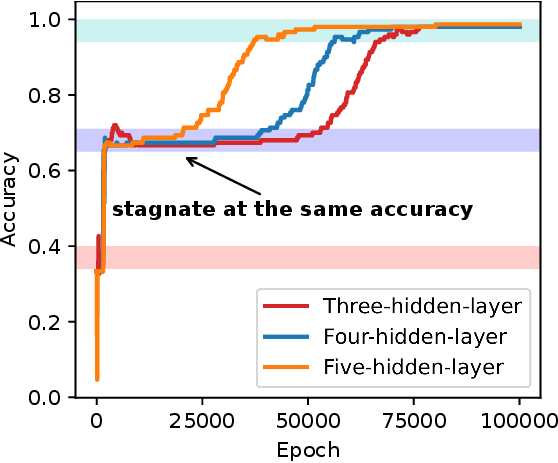}}\\
\subfigure[Loss (MNIST)]{\includegraphics[height=0.27\textwidth]{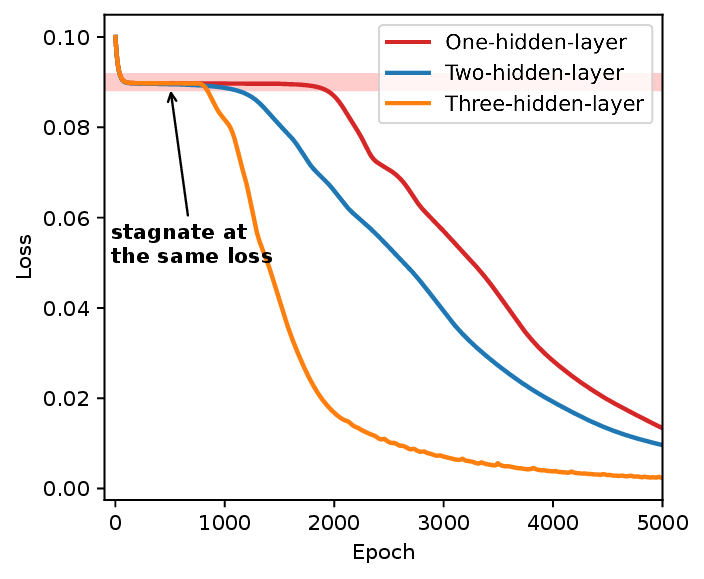}}\hspace{1.5cm}
\subfigure[Accuracy (MNIST)]{\includegraphics[height=0.27\textwidth]{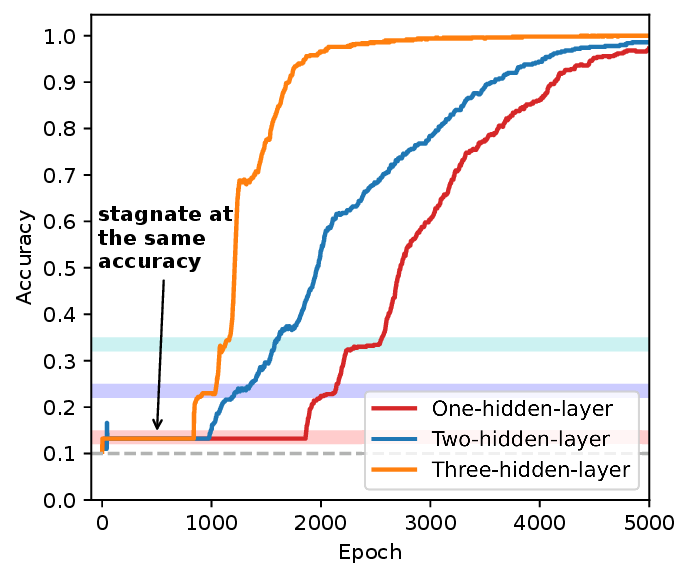}}
	\caption{\textbf{The training dynamics of networks of different depths exhibit similarity.} (a, c) The training loss for NNs of varying depths on the Iris and MNIST datasets, respectively. (b, d) The corresponding training accuracy for NNs of varying depths on the Iris and MNIST datasets, respectively. The color-coded areas indicate periods of slow change in training loss or training accuracy, indicating a possible encounter with a saddle point.
\label{fig:syntraining}}
\end{figure} 

Our theoretical investigation is motivated by the following experimental observations, which hint at the existence of an embedding relationship in depth. As illustrated in Fig.~\ref{fig:syntraining}, the training of NNs with varying hidden layers, learning the Iris and MNIST datasets with small initialization and a small learning rate, exhibit a similar behavior. Specifically, in Fig.~\ref{fig:syntraining}(a, c), we notice that network trajectories of different depths appear to stagnate at almost the same loss values, with virtually the same training accuracy, as demonstrated in Fig.~\ref{fig:syntraining}(b, d). This intriguing observation suggests that the loss landscapes of NNs of varying depths may share a set of critical functions (i.e., output functions of critical points), by which a deep NN can experience a training process similar to that of a shallower one.

Motivated by these observations, we prove in this work an embedding principle in depth for fully-connected NNs, which can be \emph{intuitively} stated as follows:

\emph{\textbf{Embedding Principle in Depth}: the loss landscape of any network  ``contains'' all critical points of all shallower networks.}

Central to our proof of the embedding principle in depth is the introduction of a critical lifting operator. This operator, as proposed in this work, maps any critical point of a shallower NN to critical manifolds (i.e., manifolds consisting of critical points sharing the same loss value) of a target NN, while preserving outputs on the training inputs. Our critical lifting operator predicts a rich class of ``simple'' critical points, which are derived from shallower NNs and embedded in the loss landscapes of deeper NNs. This thereby explicitly unveils the depth-wise hierarchical structure within the loss landscape of deep learning.

To evaluate the practical implications of the embedding principle in depth, we conduct a comprehensive set of numerical experiments. These experiments reveal that the practical training dynamics of deep NNs indeed encounter these lifted critical points, resulting in similar training dynamics between deep and shallow networks. Furthermore, we observe that through the critical lifting process, lifted critical points exhibit increased degeneracy, which aligns with the empirically observed highly degenerate critical points within the loss landscape~\cite{sagun2016singularity}. The embedding principle in depth also provides new understanding to the optimization benefits of batch normalization~\cite{ioffe2015batch} and the use of larger datasets. In the final part of our experimental study, we explore the aspect of network compression, proposing a method for layer pruning.

The remainder of the paper is organized as follows. In Section~\ref{sec:related-works}, we review related works. In Section~\ref{sec:pre}, we provide a brief introduction to deep neural networks and the back propagation process. In Section~\ref{sec:theory-embedding}, we lay out the theory of the embedding principle in depth. Section~\ref{sec:exp} presents a range of practical effects to corroborate our theoretical insights. In Section~\ref{sec:discussion}, we contrast the differences between the embedding principles in width and depth and discuss other network architectures beyond fully-connected networks. We conclude the paper in Section~\ref{sec:conclusion}. Detailed proofs are provided in Appendix~\ref{app:proofs}. 

\section{Related works}
\label{sec:related-works}
The loss landscape of deep neural networks is notoriously complex due to its high-dimensionality and non-convex nature \cite{skorokhodov2019loss}. Certain directions of a minimum can exhibit markedly different sharpness \cite{he2019asymmetric}. Moreover, different training algorithms find global minima with different properties, such as SGD often finds a flatter minimum compared with GD  \cite{keskar2017large,wu2017towards}. Although previous studies have provided detailed investigations on the loss landscape of shallow NNs with specific activations \cite{du2018power,soltanolkotabi2018theoretical,cheridito2021landscape}, the relationship between critical points across different network architectures remains largely unexplored.

The recent work ~\cite{zhang2021embedding} introduced an embedding principle (in width) that establishes a relationship between the critical points of a network and its wider counterparts. The principle, which leverages one-step embeddings and their multi-step composition, suggests that the critical points of a network can be embedded into the loss landscape of wider NNs. Similar findings about these composition embeddings have been studied~\cite{fukumizu2000local,fukumizu2019semi,csimcsek2021geometry}. Different from these works studying the effect of width, our work for the first time establishes the embedding relation regarding the extremely important hyperparameter of depth for DNNs.

Using a deeper NN has many advantages. In approximation, a deeper NN has more expressive power \cite{telgarsky2016benefits,eldan2016power,weinan2018exponential}. In optimization, a deeper NN can learn data faster ~\cite{he2016deep,arora2018optimization,DBLP:conf/aaai/XuZ21}. In generalization, it has been widely observed that overparameterized deep neural networks often generalize well in practice~\cite{zhang2016understanding} and a deeper NN may achieve better generalization for real-world problems~\cite{he2016deep}. Therefore, it is important to understand the effect of depth to the DNN loss landscapes.

The proposed embedding principle in depth suggests a simplicity bias in depth, which is consistent with previous works, for example, the frequency principle ~\cite{xu_training_2018,xu2019frequency,rahaman2018spectral,zhang2021linear}, which states that DNNs often fit target functions from low to high frequencies during the training, and the block structure \cite{nguyenwide}, which identifies similar representations across many layers in overparameterized networks.

\section{Preliminaries}\label{sec:pre}
\paragraph{Deep neural networks.} Consider a fully connected neural network (NN) with $L (L\geq 1)$ layers. Let $i, k \in \mathbb{N}$, and for $i < k$, denote $[i:k]={i, i+1, \ldots, k}$. Specifically, denote $[k]:={1, 2, \ldots, k}$. The input is treated as layer $0$ and the output as layer $L$. The width of layer $l$ is represented by $m_l$, with $m_0=d$ and $m_L=d'$. 

For any parameter $\boldsymbol{\theta}$ of the NN, we consider it as a $2L$-tuple
$$\vtheta=\bigl(\vtheta|_1,\cdots,\vtheta|_L\bigr)=\bigl(\mW^{[1]},\vb^{[1]},\ldots,\mW^{[L]},\vb^{[L]}\bigr),$$ 
where $\boldsymbol{W}^{[l]} \in \mathbb{R}^{m_l\times m{l-1}}$ and $\boldsymbol{b}^{[l]} \in\mathbb{R}^{m{l}}$ represent the weight and bias of layer $l$, respectively. The parameters of layer $l$ in $\boldsymbol{\theta}$ are given as an ordered pair $\boldsymbol{\theta}|_{l}=\bigl(\boldsymbol{W}^{[l]}, \boldsymbol{b}^{[l]}\bigr),$ for $l\in[L]$. We may use notation interchangeably and identify $\vtheta$ with its vectorization $\mathrm{vec}(\vtheta)\in \sR^M$ with  $M=\sum_{l=0}^{L-1}(m_l+1) m_{l+1}$.

Given the parameter vector $\boldsymbol{\theta}$, the neural network function $\vf_{\vtheta}(\cdot)$ can be defined recursively. First, let $\vf^{[0]}_{\vtheta}(\vx)=\vx$ for all $\vx\in \mathbb{R}^d$. Then for $l \in [L-1]$, $\vf^{[l]}_{\vtheta}$ is defined recursively as 
$$\vf^{[l]}_{\vtheta}(\vx)=\sigma (\mW^{[l]} \vf^{[l-1]}_{\vtheta}(\vx)+\vb^{[l]}).$$
Finally, we denote $\vf_{\vtheta}(\vx)=\vf(\vx; \vtheta)=\vf^{[L]}_{\vtheta}(\vx)=\mW^{[L]} \vf^{[L-1]}_{\vtheta}(\vx)+\vb^{[L]}.$

In the case of residual neural networks (ResNets), if the $l$-th layer employs a skip connection, then 
$$
\boldsymbol{f}_{\boldsymbol{\theta}}^{[l]}(\boldsymbol{x}) = \sigma\left(\boldsymbol{W}^{[l]} \boldsymbol{f}_{\boldsymbol{\theta}}^{[l-1]}(\boldsymbol{x})+\boldsymbol{b}^{[l]}\right) + \boldsymbol{f}_{\boldsymbol{\theta}}^{[l-1]}(\boldsymbol{x}).
$$

To enable a comprehensible comparison between deep and shallow networks, we provide a precise definition for the terms ``deeper" and ``shallower".

\begin{definition}[\textbf{deeper/shallower}]
\label{def:deeper}
Given two NNs, $\mathrm{NN}\bigl(\bigl\{m_{l}\bigr\}_{l=0}^{L}\bigr\}\bigl)$ and $\mathrm{NN}^{\prime}(\{m_l^\prime\}$, $l \in \{0, 1, \cdots, q, $ $\hat{q},  q+1, \cdots, L\})$. If $m_1^\prime = m_1, \cdots, m_{q}^\prime = m_q, m_{\hat{q}}^\prime\geq \min\{m_{q}, m_{q+1}\}, m_{q+1}^\prime = m_{q+1}, \cdots, m_L^\prime = m_L$, then we say $\mathrm{NN}^{\prime}$ is one-layer deeper than $\mathrm{NN}$, and conversely, $\mathrm{NN}$ is one-layer shallower than $\mathrm{NN}^{\prime}$. $J$-layer deeper (or shallower) is defined by the composition of one-layer deeper (or shallower).
\end{definition}

\begin{remark}
If an NN is termed ``deeper" than another NN, it signifies that the former can be derived by incorporating additional layers of adequate widths to the latter. Note that, for the sake of notational convenience, the layer index $l$ for the deeper NN is utilized as a placeholder index, adhering to a specific order of $\{0, 1, 2, \cdots, q, \hat{q}, q+1, \cdots, L\}$.
\end{remark}

\paragraph{Loss function.} We designate the training data and training inputs as $S=\{(\vx_i,\vy_i)\}_{i=1}^n$ and $S_{\vx}=\{\vx_i\}_{i=1}^n$, respectively, where $\vx_i\in\sR^d$ and $\vy_i\in \sR^{d'}$. For the sake of convenience, we presuppose an unknown function $\vf^*$ such that $\vf^*(\vx_i)=\vy_i$ holds for $i\in[n]$. The empirical risk can be expressed as
$$\RS(\vtheta)=\frac{1}{n}\sum_{i=1}^n\ell(\vf(\vx_i,\vtheta),\vf^*(\vx_i))=\Exp_S\ell(\vf(\vx,\vtheta),\vf^*(\vx)),$$
where the expectation $\Exp_S h(\vx):=\frac{1}{n}\sum_{i=1}^n h(\vx_i)$ is defined for any function $h:\sR^d\to \sR$. 
We denote the derivative of the loss function $\ell$ with respect to its first argument as $\nabla\ell(\vy,\vy^*)$. The training dynamics are treated as the gradient flow of $\RS(\vtheta)$, i.e., 
\begin{equation*}
    \left\{
    \begin{aligned}
        & \dfrac{\D \vtheta}{\D t}
          = -\nabla_{\vtheta} \RS(\vtheta), \\
        & \vtheta(0)
         = \vtheta_0.
    \end{aligned}
    \right.\label{eq..TrainingL2Loss}
\end{equation*}

\paragraph{Back propagation.} For every $l\in [L]$, we define the error vectors $\vz_{\vtheta}^{[l]}=\nabla_{\vf^{[l]}}\ell$ and the feature gradients $\vg_{\vtheta}^{[L]}=\boldsymbol{1}$ along with $\vg^{[l]}_{\vtheta} =\sigma^{(1)}\bigl(\mW^{[l]} \vf^{[l-1]}_{\vtheta}+\vb^{[l]}\bigr)$ for $l\in[L-1]$, where $\sigma^{(1)}$ signifies the first derivative of $\sigma$. Furthermore, $\vf^{[l]}_{\vtheta}$, for $l\in[L]$, are referred to as feature vectors. We denote the collections of feature vectors, feature gradients, and error vectors by $\vF_{\vtheta}= \{\vf^{[l]}_{\vtheta}\}_{l=1}^L,
    \vG_{\vtheta}
    = \{\vg^{[l]}_{\vtheta}\}_{l=1}^L,
    \vZ_{\vtheta}
    = \{\vz^{[l]}_{\vtheta}\}_{l=1}^L$, respectively. The gradients can be computed employing backpropagation as follows:
\begin{equation}
\left\{\begin{aligned}
& \boldsymbol{z}_{\boldsymbol{\theta}}^{[L]} =\nabla \ell, \\
& \boldsymbol{z}_{\boldsymbol{\theta}}^{[l]} =\left(\boldsymbol{W}^{[l+1]}\right)^{\top}\left(\boldsymbol{z}_{\boldsymbol{\theta}}^{[l+1]} \circ \boldsymbol{g}_{\boldsymbol{\theta}}^{[l+1]}\right), \quad l \in[L-1], \\
& \nabla_{\boldsymbol{W}^{[l]}} \ell =\left(\boldsymbol{z}_{\boldsymbol{\theta}}^{[l]} \circ \boldsymbol{g}_{\boldsymbol{\theta}}^{[l]}\right)\left(\boldsymbol{f}_{\boldsymbol{\theta}}^{[l-1]}\right)^{\top}, \quad l \in[L], \\
& \nabla_{\boldsymbol{b}^{[l]}} \ell =\boldsymbol{z}_{\boldsymbol{\theta}}^{[l]} \circ \boldsymbol{g}_{\boldsymbol{\theta}}^{[l]}, \quad l \in[L].
\end{aligned}\right.
\label{back propoga}
\end{equation}

\section{Theory of Embedding Principle in Depth}
\label{sec:theory-embedding}
Consider a neural network $\vf_{\vtheta}(\vx)$, where $\vtheta$ represents the set of all network parameters and $\vx\in\sR^{d}$ is the input. We summarize the assumptions for all our theoretical results in this work as follows.

\begin{assumption}
(i) $L$-layer ($L\geq 1$) fully-connected NN.

(ii) Training data $S=\{(\vx_i,\vy_i)\}_{i=1}^n$ for $n\in\sZ^+$. 

(iii) Empirical risk  $\RS(\vtheta)=\Exp_S\ell(\vf_{\vtheta}(\vx),\vy)$.

(iv) Activation function $\sigma$ has a non-constant linear segment, e.g., ReLU, leaky-ReLU and ELU.

(v) Loss function $\ell$ and activation function $\sigma$ are subdifferentiable, i.e., a unique subgradient can be assigned to each non-differentiable point.
\end{assumption}

\begin{remark}
For general smooth activations without a linear segment, e.g. tanh, our results hold in the sense of approximation because they are arbitrarily close to linear in a sufficiently small interval (for instance, around $0$). Therefore, we also demonstrate our results using the tanh activation in the subsequent numerical experiments.
\end{remark}

\begin{definition}[\textbf{affine subdomain}]
For an activation $\sigma$ with a non-constant linear segment, an affine subdomain of $\sigma$ is an open interval $(a, b)$ satisfying that there exist $\lambda,\mu \in \mathbb{R}$ ($\lambda\neq0$),
$\sigma(x) = \lambda x + \mu, \text{ for any } x\in (a, b).$
\end{definition}

\subsection{Lifting operator}
We begin by introducing a lifting operator, as illustrated in Fig.~\ref{fig:onestep}.

\begin{definition}[\textbf{one-layer lifting}]
\label{def:one-layer-lifting}

Given data $S$, consider an $\mathrm{NN}\bigl(\left\{m_{l}\right\}_{l=0}^{L}\bigr)$ and its one-layer deeper counterpart, $\mathrm{NN}^{\prime}\big(\{m_l^\prime\}$, $l\in \{0, 1, 2, \cdots, q, \hat{q}, q+1, \cdots, L\}\big)$. The one-layer lifting, denoted as $\fT_S$,  is a function that transforms any parameter $\vtheta=\left(\boldsymbol{W}^{[1]}, \boldsymbol{b}^{[1]}, \cdots, \boldsymbol{W}^{[L]}, \boldsymbol{b}^{[L]}\right)$ of $\mathrm{NN}$ into a set $\fM$ within the parameter space of $\mathrm{NN}^{\prime}$. Formally, 
$\fM$ (where $\fM:= \fT_S(\vtheta))$ represents a collection of all possible parameters $\vtheta'$ of $\mathrm{NN}^{\prime}$
that satisfying the following three conditions:

(i) local-in-layer condition: weights of each layer in $\mathrm{NN}^{\prime}$ are inherited from $\mathrm{NN}$ except for layer $\hat{q}$ and $q+1$, i.e.,  
\begin{equation}
   \left\{ \begin{aligned}
& \vtheta^{\prime}|_l = \boldsymbol{\theta} |_l,\quad \text{for}\quad  l \in [q]\cup [q+2:L],\\
& \vtheta^{\prime}|_{\hat{q}} = \bigl(\boldsymbol{W'}^{[\hat{q}]},  \boldsymbol{b'}^{[\hat{q}]}\bigr)\in\mathbb{R}^{m'_{\hat{q}}\times m'_{q-1}}\times \mathbb{R}^{m'_{\hat{q}}},\\
& \vtheta^{\prime}|_{q+1} = \bigl(\boldsymbol{W'}^{[q+1]},  \boldsymbol{b'}^{[q+1]}\bigr)\in\mathbb{R}^{m'_{q+1}\times m'_{\hat{q}}}\times \mathbb{R}^{m'_{q+1}}.\\
\end{aligned}
\right. 
\end{equation}

(ii) layer linearization condition: for any $j\in [m_{\hat{q}}]$, there exists an affine subdomain $(a_j, b_j)$ of $\sigma$ associated with $\lambda_j, \mu_j$ such that the $j$-th component $\big(\boldsymbol{W'}^{[\hat{q}]}\boldsymbol{f}_{\boldsymbol{\theta}'}^{[q]}(\boldsymbol{x}) + \boldsymbol{b'}^{[\hat{q}]}\big)_j \in (a_j, b_j)$ for any $\vx \in S_{\vx}.$

(iii) output preserving condition:
\begin{equation}
\label{eq:output-pres}
\left\{\begin{aligned}
&  \boldsymbol{W'}^{[q+1]}\operatorname{diag}(\boldsymbol{\lambda})\boldsymbol{W'}^{[\hat{q}]} = \boldsymbol{W}^{[q+1]},\\
& \boldsymbol{W'}^{[q+1]}\operatorname{diag}(\boldsymbol{\lambda})\boldsymbol{b'}^{[\hat{q}]}+\boldsymbol{W'}^{[q+1]}\boldsymbol{\mu} + \boldsymbol{b'}^{[q+1]} = \boldsymbol{b}^{[q+1]},   
\end{aligned}\right.  
\end{equation}
where $\vlambda = [\lambda_1, \lambda_2, \cdots, \lambda_{m_{\hat{q}}}]^\top\in \mathbb{R}^{m^{\prime}_{\hat{q}}}, \vmu = [\mu_1, \mu_2, \cdots, \mu_{m_{\hat{q}}}]^\top\in\mathbb{R}^{m^{\prime}_{\hat{q}}}$, and $\operatorname{diag}(\boldsymbol{\lambda})$ denotes the diagonal matrix formed by vector $\vlambda$.
\end{definition}

\begin{figure*}[t]
    \centering
    \includegraphics[width=1\textwidth]{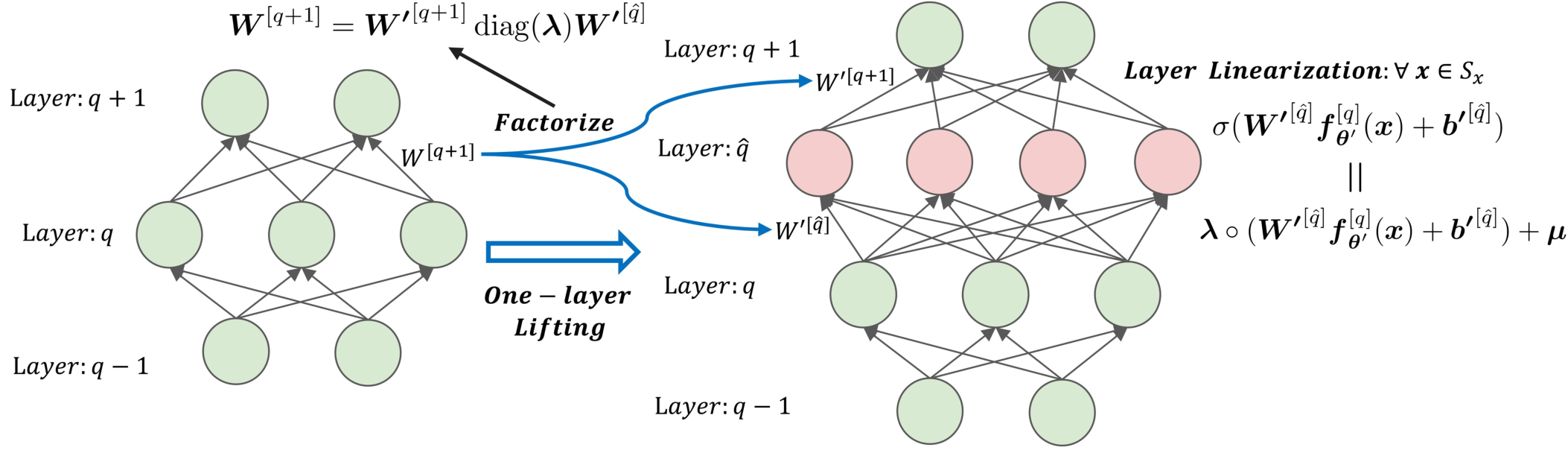} 
\caption{\textbf{Illustration of one-layer lifting.} The pink layer is inserted into the left network to get the right network. The input parameters $\vW^{\prime[\hat{q}]}$ and output parameters $\vW^{\prime[q+1]}$ of the inserted layer are obtained by factorizing the input parameters $\vW^{[ q+1]}$ of $(q+1)$-th layer in the left network to satisfy layer linearization and output preserving conditions. }
\label{fig:onestep}
\end{figure*}

\begin{remark}
Intuitively, the lifting operator described above elevates a point $\vtheta$ from $\mathrm{NN}$ to a set $\fM$ within a higher-dimensional space. In general, $\fM$ is a finite union of manifolds instead of a set of isolated points. To highlight this point, we refer to the mapped set $\fM := \fT_S(\vtheta)$ as a `manifold' with slight abuse of terminology throughout this paper. For a similar reason, we also refer to the set of critical points sharing the same loss value as `critical manifold'.
\end{remark}

\begin{remark}
The one-layer lifting operator describes the generic process of mapping a point from a lower-dimensional space to a set within a higher-dimensional space. Any $\vtheta'$ that belongs to the mapped set as outlined in Def.~\ref{def:one-layer-lifting} is termed a `lifted point'. For clarity, when discussing specific operations that map a point $\vtheta$ of $\mathrm{NN}$
to a specific point $\vtheta'\in \fT_S(\vtheta)$ of $\mathrm{NN}^\prime$, we use the term `embedding' to describe this operation.
\end{remark}

As illustrated in Fig.~\ref{fig:onestep}, a one-layer lifting is realized by inserting a hidden layer, depicted here as the pink layer ($\hat{q}$-th layer) in the right network. The outcome of a one-layer lifting is a manifold of the parameter space of the right network, which comprises each parameter vector that satisfies the following constraints: the parameters $\boldsymbol{W'}^{[\hat{q}]}, \boldsymbol{b'}^{[\hat{q}]}$ of the inserted layer meet the layer linearization condition, ensuring this layer operates like a linear layer \emph{when applied to the training inputs $S_{\vx}=\{\vx_i\}_{i=1}^n$}. Additionally, the parameters $\boldsymbol{W'}^{[q+1]}, \boldsymbol{b'}^{[q+1]}$ of $(q+1)$-th layer fulfill the output preserving condition, making the composition of the $\hat{q}$-th layer and the $(q+1)$-th layer in the right network equivalent to the $(q+1)$-th layer in the left network.

Because the factorized weights satisfying both layer linearization and output preserving conditions always exist, we have the following existence result for one-layer lifting:
\begin{lemma}[\textbf{existence of one-layer lifting}]
\label{lemma:existence of lifting} (see Appendix~\ref{app:proofs}: Lem.~\ref{APP:existence} for proof) Given data $S$, an $\mathrm{NN}\bigl(\left\{m_{l}\right\}_{l=0}^{L}\bigr)$ and its one-layer deeper counterpart,  $\mathrm{NN}^{\prime}\big(\{m_l^\prime\}$, $l\in \{0, 1, 2, \cdots, q, \hat{q}, q+1, \cdots, L\}\big)$, the one-layer lifting $\fT_S$ exists, i.e., 
$\fT_S(\vtheta_{\rshal})$ is not empty for any parameter $\vtheta_{\rshal}$ of $\mathrm{NN}$. 
\end{lemma}

\subsection{Embedding Principle in Depth}
The multi-layer lifting is defined as the composition of multiple one-layer liftings. Consequently, the parameters of any neural network can be lifted to a parameter manifold of any deeper neural network through a multi-layer lifting. Both one-layer and multi-layer liftings exhibit two key attributes: network properties preservation and criticality preservation. To demonstrate these two properties, we first consider the following lemma:
\begin{lemma}[\textbf{computation of feature vectors, feature gradients and error vectors}]
\label{lemma-gradients} (see Appendix~\ref{app:proofs}: Lem.~\ref{APP:lemma-gradients} for proof)
Given data $S$, consider an $\mathrm{NN}\bigl(\left\{m_{l}\right\}_{l=0}^{L}\bigr)$ and its one-layer deeper counterpart, $\mathrm{NN}^{\prime}\big(\{m_l^\prime\}$, $l\in \{0, 1, 2, \cdots, q, \hat{q}, q+1, \cdots, L\}\big)$. Let $\fT_S$ denote the one-layer lifting and $\vtheta_{\rshal}$
be any parameter of $\mathrm{NN}$. Then, for any lifted point
$\boldsymbol{\theta}'_{\rdeep} \in \fT_{S}\left(\boldsymbol{\theta}_{\rshal}\right)$, the following conditions hold: there exist $\vlambda, \vmu \in \mathbb{R}^{m^{\prime}_{\hat{q}}}$ such that for any $\vx\in S_{\vx}$,

(i) feature vectors in $\boldsymbol{F}_{\boldsymbol{\theta}'_{\rdeep}}: \boldsymbol{f}_{\boldsymbol{\theta}'_{\rdeep}}^{\left[l\right]}(\vx)=\boldsymbol{f}_{\boldsymbol{\theta}_{\rshal}}^{\left[l\right]}(\vx)$ for $l \in[L]$ and $ \boldsymbol{f}_{\boldsymbol{\theta}'_{\rdeep}}^{[\hat{q}]}(\vx) = \operatorname{diag}(\vlambda) (\boldsymbol{W}^{\prime[\hat{q}]}\boldsymbol{f}_{\vtheta_{\rshal}}^{[q]}(\vx) + \boldsymbol{b}^{\prime[\hat{q}]}) + \boldsymbol{\mu};$

(ii) feature gradients in $\boldsymbol{G}_{\boldsymbol{\theta}'_{\rdeep}}: \boldsymbol{g}_{\boldsymbol{\theta}'_{\rdeep}}^{\left[l\right]}(\vx)=\boldsymbol{g}_{\boldsymbol{\theta}_{\rshal}}^{\left[l\right]}(\vx)$ for $l \in[L]$ and $ \boldsymbol{g}_{\boldsymbol{\theta}'_{\rdeep}}^{[\hat{q}]}(\vx) = \boldsymbol{\lambda};$

(iii) error vectors in $\boldsymbol{Z}_{\boldsymbol{\theta}'_{\rdeep}}: \boldsymbol{z}_{\boldsymbol{\theta}'_{\rdeep}}^{\left[l\right]}(\vx)=\boldsymbol{z}_{\boldsymbol{\theta}_{\rshal}}^{\left[l\right]}(\vx)$ for $l \in [q-1]\cup[q+1: L]$ and $\boldsymbol{z}_{\boldsymbol{\theta}'_{\rdeep}}^{\left[\hat{q}\right]}(\vx) = \bigl(\boldsymbol{W}^{'[q+1]}\bigr)^{\top}\bigl(\boldsymbol{z}_{\boldsymbol{\theta}_{\rshal}}^{[q+1]}(\vx) \circ \boldsymbol{g}_{\boldsymbol{\theta}_{\rshal}}^{[q+1]}(\vx)\bigr),  \boldsymbol{z}_{\boldsymbol{\theta}'_{\rdeep}}^{\left[q\right]}(\vx) = \bigl(\boldsymbol{W}^{'[\hat{q}]}\bigr)^{\top}\bigl(\boldsymbol{z}_{\boldsymbol{\theta}'_{\rdeep}}^{[\hat{q}]}(\vx) \circ \boldsymbol{\lambda}\bigr).$
\end{lemma}

Using Lem.~\ref{lemma-gradients}, we can promptly derive the property of network preservation. This underlines the preservation of the output function by the lifting process during the transformation from a shallower to a deeper neural network.

\begin{proposition}[\textbf{network properties preserving}]
\label{prop:network preserving} (see Appendix~\ref{app:proofs}: Prop.~\ref{APP:prop-network-preserve} for proof)  Given data $S$, consider an $\mathrm{NN}\bigl(\left\{m_{l}\right\}_{l=0}^{L}\bigr)$ and its one-layer deeper counterpart, $\mathrm{NN}^{\prime}\big(\{m_l^\prime\}$, $l\in \{0, 1, 2, \cdots, q, \hat{q}, q+1, \cdots, L\}\big)$. Let $\fT_S$ denote the one-layer lifting and $\vtheta_{\rshal}$
be any parameter of $\mathrm{NN}$. Then, for any lifted point
$\boldsymbol{\theta}'_{\rdeep} \in \fT_{S}\left(\boldsymbol{\theta}_{\rshal}\right)$, the following conditions hold:

    (i) outputs are preserved: $f_{\vtheta'_{\rdeep}}(\vx)=f_{\vtheta_{\rshal}}(\vx)$  for $\vx\in S_{\vx}$;
    
    (ii) empirical risk is preserved: $\RS(\vtheta'_{\rdeep})=\RS(\vtheta_{\rshal})$;
    
    (iii) network representations are preserved for all layers:  
    $$
    \operatorname{span}\bigl\{\bigl\{\bigl(\vf_{\vtheta_{\rdeep}^\prime}^{[\hat{q}]}(\vX)\bigr)_j\bigr\}_{j \in[m^{\prime}_{\hat{q}}]}\cup \{ \boldsymbol{1}\} \bigr\} = \operatorname{span}\bigl\{\bigl\{\bigl(\vf_{\vtheta_{\rshal}}^{[q]}(\vX)\bigr)_j\bigr\}_{j \in[m_{q}]}\cup \{\boldsymbol{1}\} \bigr\},  
    $$
    
    and for the other index $l\in [L]$,
    $$
    \operatorname{span}\bigl\{\bigl\{\bigl(\vf_{\vtheta_{\rdeep}^\prime}^{[l]}(\vX)\bigr)_j\bigr\}_{j \in[m^{\prime}_l]}\cup \bigl\{\boldsymbol{1}\bigr\} \bigr\}=\operatorname{span}\bigl\{\bigl\{\bigl(\vf_{\vtheta_{\rshal}}^{[l]}(\vX)\bigr)_j\bigr\}_{j \in[m_l]}\cup \bigl\{\boldsymbol{1}\bigr\} \bigr\},    
    $$
    
    where $\vf_{\vtheta}^{[l]}(\vX)=\bigl[\vf_{\vtheta}^{[l]}(\vx_1), \vf_{\vtheta}^{[l]}(\vx_2), \cdots, \vf_{\vtheta}^{[l]}(\vx_n)\bigr]^\top \in \sR^{n\times m^{\prime}_{\hat{q}}}$ and $\boldsymbol{1}\in\sR^{n}$ is the all-ones vector.
\end{proposition}

The most significant characteristic of the lifting operator is its preservation of criticality. That is to say, if a shallow network is at a critical point, it will still be at a critical point when transformed into a deeper network through the lifting operator. 

\begin{proposition}[\textbf{criticality preserving}]\label{prop:criticality-preserving}(see Appendix~\ref{app:proofs}: Prop.~\ref{APP:criticality-preserve} for proof)
Given data $S$, consider an $\mathrm{NN}\bigl(\left\{m_{l}\right\}_{l=0}^{L}\bigr)$ and its one-layer deeper counterpart, $\mathrm{NN}^{\prime}\big(\{m_l^\prime\}$, $l\in \{0, 1, 2, \cdots, q, \hat{q}, q+1, \cdots, L\}\big)$. Let $\fT_S$ denote the one-layer lifting and $\vtheta_{\rshal}$
be any parameter of $\mathrm{NN}$. If $\vtheta_{\rshal}$ of $\mathrm{NN}$ satisfies $\nabla_{\boldsymbol{\theta}} R_{S}\left(\boldsymbol{\theta}_{\rshal}\right)=\boldsymbol{0}$, then $\nabla_{\boldsymbol{\theta}^{\prime}} R_{S}\bigl(\boldsymbol{\theta}'_{\rdeep}\bigr)=\boldsymbol{0}$ for any lifted point $\boldsymbol{\theta}'_{\rdeep} \in \fT_{S}\left(\boldsymbol{\theta}_{\rshal}\right)$.
\end{proposition}

Owing to the criticality preserving property, we term one-layer or multi-layer lifting as \emph{critical lifting}. With the aforementioned results, we now establish an embedding principle in depth, which can be intuitively described as follows: \emph{the loss landscape of any DNN contains a hierarchy of critical manifolds, each of which is lifted from the critical points of the loss landscapes of all its shallower counterparts.}
\begin{theorem}[\textbf{embedding principle in depth}]
\label{thm:embedding-principle} (see Appendix~\ref{app:proofs}: Thm.~\ref{APP:thm:embedding-principle} for proof)
Given data $S$ and an $\mathrm{NN'}\bigl(\left\{m'_{l}\right\}_{l=0}^{L'}\bigr)$, for any parameter $\vtheta_{\rc}$ of any  shallower $\mathrm{NN}\bigl(\left\{m_{l}\right\}_{l=0}^{L}\bigr)$ satisfying $\nabla_{\boldsymbol{\theta}} R_{S}\left(\boldsymbol{\theta}_{\rc}\right)=\boldsymbol{0}$, there exists parameter $\vtheta'_{\rc}$ in the loss landscape of  $\mathrm{NN'}\bigl(\left\{m'_{l}\right\}_{l=0}^{L'}\bigr)$ satisfying the following conditions:
\begin{itemize}
    \item[(i)] $\vf_{\vtheta'_{\rc}}(\vx)=\vf_{\vtheta_{\rc}}(\vx)$ for $\vx\in S_{\vx}$;
    \item[(ii)] $\nabla_{\boldsymbol{\theta}'} R_{S}\left(\boldsymbol{\theta'}_{\rc}\right)=\boldsymbol{0}$.
\end{itemize}
\end{theorem}

\begin{remark}
Although we only prove output preserving for training inputs,  it is important to note that the output function of the neural network is indeed preserved over a broader area of the input space including at least a neighbourhood of each training input (see Appendix~\ref{app:proofs}: Prop.~\ref{APP:prop-generalization} for proof). Consequently, if the training dataset is sufficiently large and representative, then the lifting operator effectively preserves the generalization performance. 
\end{remark}

Leveraging critical lifting, the aforementioned embedding principle in depth provides a clear picture about the hierarchical structure of critical points/manifolds in depth within a DNN's loss landscape. This hierarchical structure profoundly influences the nonlinear training behavior of a deep network, as any nearby training trajectory tends to gravitate towards these points/manifolds.

Critical lifting delineates the relationship between the critical points of deep networks and their shallow counterparts. Notably, the critical embedding resulting from critical lifting also preserves the positive and negative inertia indices of the Hessian matrix.
\begin{proposition}[\textbf{positive and negative index of inertia preserving}]\label{prop:inertia-preserving}(see Appendix~\ref{app:proofs}: Prop.~\ref{app:inertia-preserving} for proof)
Given data $S$, consider an $\mathrm{NN}\bigl(\left\{m_{l}\right\}_{l=0}^{L}\bigr)$, and its deeper counterpart $\mathrm{NN'}\bigl(\left\{m'_{l}\right\}_{l=0}^{L'}\bigr)$. Let $\fT_S$ denote the corresponding critical lifting and $\vtheta_{\rshal}$
be a critical point of $\mathrm{NN}$. 
For any critical embedding $\fE: \sR^M \to \sR^{M'}$ resulting from $\fT_S$ (i.e., $\mathcal{E}$ is a differentiable point-to-point critical mapping with full column rank Jacobian $\boldsymbol{J}_{\mathcal{E}(\boldsymbol{\theta}_{\text{shal}})}$). Denote $\boldsymbol{\theta}_{\text{deep}}:=\mathcal{E}(\boldsymbol{\theta}_{\text{shal}})$. Then the number of positive and negative eigenvalues of the Hessian matrix $\boldsymbol{H}_S(\boldsymbol{\theta}_{\text{deep}})$ equals the counterparts of $\boldsymbol{H}_S(\boldsymbol{\theta}_{\text{shal}})$.
\end{proposition}

Utilizing Prop~\ref{prop:inertia-preserving}, we immediately conclude that through critical embedding, the degeneracy of a critical point will increase.

\begin{corollary}[\textbf{incremental degeneracy of critical point through lifting}]\label{cor:incremental_degeneracy}(see Appendix~\ref{app:proofs}: Cor.~\ref{app:cor:incremental_degeneracy} for proof)
Given data $S$, consider an $\mathrm{NN}\bigl(\left\{m_{l}\right\}_{l=0}^{L}\bigr)$, and its deeper counterpart $\mathrm{NN'}$ $\bigl(\left\{m'_{l}\right\}_{l=0}^{L'}\bigr)$. Let $\fT_S$ denote the corresponding critical lifting and $\vtheta_{\rshal}$
be a critical point of $\mathrm{NN}$. 
For any critical embedding $\fE: \sR^M \to \sR^{M'}$ resulting from $\fT_S$ (i.e., $\mathcal{E}$ is a differentiable point-to-point critical mapping with full column rank Jacobian $\boldsymbol{J}_{\mathcal{E}(\boldsymbol{\theta}_{\text{shal}})}$). Denote $\boldsymbol{\theta}_{\text{deep}}:=\mathcal{E}(\boldsymbol{\theta}_{\text{shal}})$.
Then, $\vtheta_{\rdeep}$ possesses $M'- M$ additional degrees of degeneracy in comparison to $\vtheta_{\rshal}$.
\end{corollary}

\begin{example}[\textbf{preservation or change of positive and negative indices of Inertia}]
Consider a simple two-layer linear network $\boldsymbol{f}(\boldsymbol{\theta}_{\text{shal}}) = \boldsymbol{W}_1 \boldsymbol{W}_2$ and $\vtheta_{\rshal} = (\vW_1, \vW_2)=(\vzero, \vzero)$ is a critical point. We lift this to a three-layer network $\boldsymbol{f}'(\boldsymbol{\theta}'_{\text{deep}}) = \boldsymbol{W}'_1 \boldsymbol{W}'_2 \boldsymbol{W}'_3$.
While $\vtheta_{\rdeep}=(\boldsymbol{W}'_1, \boldsymbol{W}'_2, \boldsymbol{W}'_3)=(\vzero, \vzero, \vI)$ and $(\vzero, \vzero, \vzero)$ are both lifted critical points, only the former corresponds to an embedded critical point via a differentiable critical embedding and therefore preserves the inertia index, whereas the latter does not and may alter the Hessian's spectral properties.

Specifically, consider a critical mapping $\mathcal{E}: (\boldsymbol{W}_1, \boldsymbol{W}_2) \hookrightarrow (\boldsymbol{W}_1, \boldsymbol{h}(\boldsymbol{W}_2), \boldsymbol{g}(\boldsymbol{W}_2))$ such that $\boldsymbol{h}(\boldsymbol{W}_2)\boldsymbol{g}(\boldsymbol{W}_2)=\boldsymbol{W}_2$.
We analyze two cases:

(1) Set $\boldsymbol{h}(\boldsymbol{W}_2)=\boldsymbol{W}_2$ and $\boldsymbol{g}(\boldsymbol{W}_2)=\boldsymbol{I}$.
One can verify that both $\boldsymbol{h}$ and $\boldsymbol{g}$ are differentiable, and $\mathcal{E}$ has a Jacobian $\boldsymbol{J}_{\mathcal{E}(\boldsymbol{\theta}_{\text{shal}})}$ with full column rank. Therefore, by Proposition~\ref{prop:inertia-preserving}, the positive and negative inertia exponents are preserved under this mapping.

(2) Consider any functions $\boldsymbol{h}$ and $\boldsymbol{g}$ satisfying $\boldsymbol{h}(\boldsymbol{W}_2)\boldsymbol{g}(\boldsymbol{W}_2) = \boldsymbol{W}_2$ with the constraint that $\boldsymbol{h}(\boldsymbol{0})=\boldsymbol{g}(\boldsymbol{0})=\boldsymbol{0}$.
One can check that $\boldsymbol{h}$ and $\boldsymbol{g}$ cannot both be differentiable in this constraint~\footnote{For simplicity, consider the one-dimensional case where $h(y)g(y) = y$ and $h(0)=g(0)=0$. At $y=0$, differentiability demands $g(0) = \lim_{y\to0} g(y) = \lim_{y\to0} \frac{y}{h(y)} = \frac{1}{h'(0)} \neq 0$, contradicting $g(0)=0$.}. Therefore, there exists no differentiable critical embedding that maps the origin $(\boldsymbol{0}, \boldsymbol{0})$ to $(\boldsymbol{0}, \boldsymbol{0}, \boldsymbol{0})$ in the deeper network. This illustrates a strict saddle point in the shallow network can be transformed into a degenerate saddle point in the deeper network.
\end{example}

It is important to note that critical lifting is data-dependent, and the nature of this data-dependence is characterized by the following proposition:
\begin{proposition}[\textbf{data dependency of critical lifting}]\label{thm:data-depen} (see Appendix~\ref{app:proofs}: Prop.~\ref{APP:thm:data-depen} for proof)
Given data $S$ and $S'$, consider an $\mathrm{NN}\bigl(\left\{m_{l}\right\}_{l=0}^{L}\bigr)$ and its deeper counterpart, $\mathrm{NN}^{\prime}\big(\{m_l^\prime\}_{l=0}^{L'}$. Let $\fT_S$ and $\fT_{S'}$ denote the respective critical liftings and $\vtheta_{\rshal}$
be any parameter of $\mathrm{NN}$. If data $S' \subseteq S$, then $\fT_{S}(\vtheta_{\rshal})\subseteq\fT_{S'}(\vtheta_{\rshal})$.
\end{proposition}

This result indicates that increasing training data shrinks any lifted manifold to its subset. The implication is that enlarging the training dataset is a viable strategy for diminishing the lifted critical manifolds, which can consequently expedite the decay of the training loss, as demonstrated in the subsequent experimental study.

\section{Numerical experiments}\label{sec:exp}

The theory of the embedding principle in depth highlights the existence of a class of "simple" critical points inherited from shallower neural networks. A natural question arises as to whether deep networks encounter these lifted critical points. Moreover, it is essential to understand the influence of these critical points on the training dynamics. In this section, we conduct comprehensive experiments to study these questions. In section~\ref{sec:experiment-setup}, we briefly describe the experimental setup. Section~\ref{sec:training-dynamics} is dedicated to a detailed comparison of the training dynamics between deep and shallow networks. Section~\ref{sec:BN} investigates the impact of batch normalization and larger datasets on optimization. Lastly, in section~\ref{sec:network-pruning}, we explore the practical application of layer-wise network pruning.

\subsection{Experimental setup}\label{sec:experiment-setup}
\paragraph{Measuring layer linearization by Minimal Pearson Correlation (MPC).} To detect lifted critical points in our experiments, we propose a method to measure their key feature, i.e., layer linearization. Let $\tilde{\vf}^{[l]}_{\vtheta}=\mW^{[l]} \vf^{[l-1]}_{\vtheta}+\vb^{[l]} \in \sR^{m_l }$ and $\vf^{[l]}_{\vtheta} = \sigma(\tilde{\vf}^{[l]}_{\vtheta})\in \sR^{m_l}$ denote the input and output of neurons in layer $l$, respectively. For each neuron in a layer, the absolute value of the Pearson correlation coefficient is utilized to measure the extent of linearization for each neuron. By \emph{taking the minimum over the whole layer}, we obtain the following measure of the extent of linearization for the $l$-th layer:
\begin{equation}
\operatorname{MPC}(\vf^{[l]}_{\vtheta}, \tilde{\vf}^{[l]}_{\vtheta}) = \min_{j\in[m_l]} \bigl|\rho\bigl(\bigl(\vf^{[l]}_{\vtheta}\bigr)_j, \bigl(\tilde{\vf}^{[l]}_{\vtheta}\bigr)_j\bigr)\bigr| \in [0, 1],
\label{eq:pear}
\end{equation}
where $\bigl(\vf^{[l]}_{\vtheta}\bigr)_j, \bigl(\tilde{\vf}^{[l]}_{\vtheta}\bigr)_j$ represent the $j$-th components of $\vf^{[l]}_{\vtheta}$ and $\tilde{\vf}^{[l]}_{\vtheta}$, respectively, and $\rho\bigl(\bigl(\vf^{[l]}_{\vtheta}\bigr)_j$, $\bigl(\tilde{\vf}^{[l]}_{\vtheta}\bigr)_j\bigr)$ denotes the Pearson correlation coefficient.
\begin{remark}
To determine whether a critical point in an experiment is a lifted critical point, two conditions must be met: (i) the $\operatorname{MPC}$ of a layer equals 1, and (ii) merging that layer with the subsequent layer leads to a critical point of the shallower neural network.
\end{remark}

In the subsequent experiments conducted, fully-connected networks are predominantly utilized. The input dimension, denoted as $d$, and output dimension, represented by $d^\prime$, are determined by the respective training dataset. Each hidden layer has the same width $m$. All parameters are initialized by a Gaussian distribution with mean zero and variance specified in each experiment. To investigate the training behavior of DNNs with feature learning, we employ relatively small initializations to enhance the nonlinearity of training, which stays away from the Neural Tangent Kernel (NTK) regime. To meticulously examine the dynamics during the training process, a full-batch gradient descent approach is employed, combined with a small learning rate. More details of experiments are presented in Appendix~\ref{app:tra-details}.

\subsection{Training dynamics of deep and shallow neural networks}\label{sec:training-dynamics}

\subsubsection{Deep neural networks encounter lifted critical points during practical training} 
To investigate whether deep neural networks encounter lifted critical points during training, we train tanh NNs with different depths (width $m=50$) on the data shown in Fig.~\ref{fig:3-hidden-nonlinear}(b) and the Iris dataset in Fig.~\ref{fig:3-hidden-nonlinear-iris}. During training, we trace the evolution of the $\operatorname{MPC}$ for each hidden layer computed using Eq.~\eqref{eq:pear}. 

As depicted in Fig.~\ref{fig:3-hidden-nonlinear}(a), the three-hidden-layer NN first stagnates at the same loss value as the single-hidden-layer NN, displaying nearly the same output function as illustrated in Fig.~\ref{fig:3-hidden-nonlinear}(b). According to Fig.~\ref{fig:3-hidden-nonlinear}(c), the first two hidden layers exhibit strong linearity during stagnation. In Fig.~\ref{fig:3-hidden-nonlinear}(d), we observe that merging the effectively linear layers using Eq.~\eqref{eq:output-pres} results in a critical point of the single-hidden-layer NN. 

\begin{figure}[t]
\centering
\subfigure[Training loss]{
\includegraphics[height=0.24\textwidth]{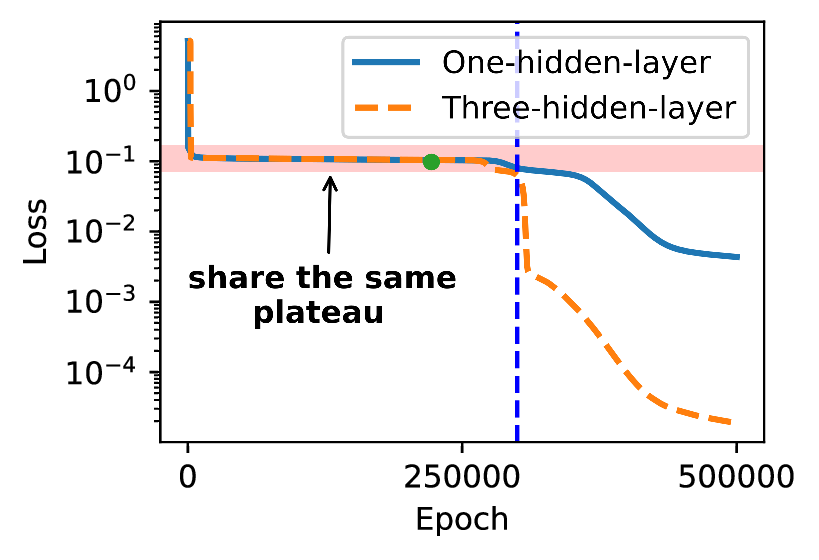}
}\hspace{1.5cm}
\subfigure[Output at the plateau]{
\includegraphics[height=0.24\textwidth]{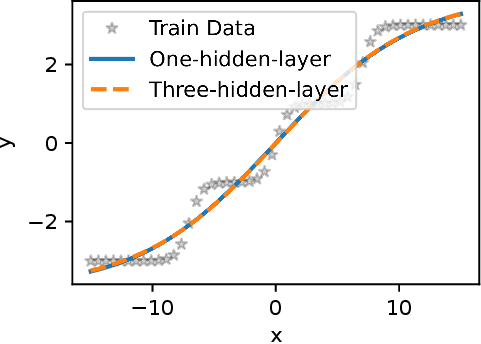}
}\\
\subfigure[Evolution of $\operatorname{MPC}$]{
\includegraphics[height=0.25\textwidth]{figures/Tanh_pear.eps}
}\hspace{1.5cm}
\subfigure[Loss of reduced NN]{
\includegraphics[height=0.25\textwidth]{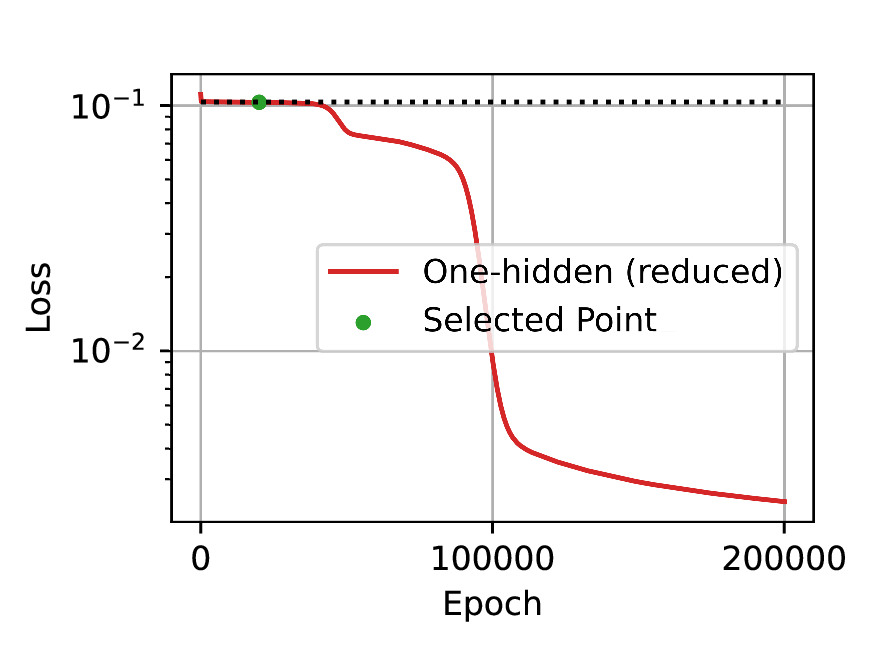}
}
\caption{\textbf{Deep neural networks encounter lifted critical points during training on synthetic data.} (a) The training loss for single-hidden-layer and three-hidden-layer NNs with width $m=50$. (b) The outputs of NNs with different depths at the same loss value indicated by the colored span in (a). (c) The extent of layer linearization for different hidden layers during the training process of the three-hidden-layer NN. (d) Training loss trajectory of the reduced single-hidden- layer NN. The green dot in (a) and (c) is selected as a representative for comparison.}
\label{fig:3-hidden-nonlinear}
\end{figure}

A similar phenomenon is observed for the Iris dataset in Fig.~\ref{fig:3-hidden-nonlinear-iris}. As shown in Fig.~\ref{fig:3-hidden-nonlinear-iris}(a), the three-hidden-layer NN first stagnates at the same loss value as the single-hidden-layer NN, displaying nearly the same training and test accuracy as illustrated in Fig.~\ref{fig:3-hidden-nonlinear-iris}(b). Upon examining the confusion matrices of the networks at this plateau, we observe that both networks correctly classify two of the three classes and completely misclassify the third class, achieving an accuracy of 66.7\%. According to Fig.~\ref{fig:3-hidden-nonlinear-iris}(c), the last two hidden layers exhibit strong linearity during stagnation. In Fig.~\ref{fig:3-hidden-nonlinear-iris}(d), we find that merging the effectively linear layers using Eq.~\eqref{eq:output-pres} results in a critical point of the single-hidden-layer NN. 

\begin{figure}[t]
\centering
\subfigure[Training loss]{
\includegraphics[height=0.25\textwidth]{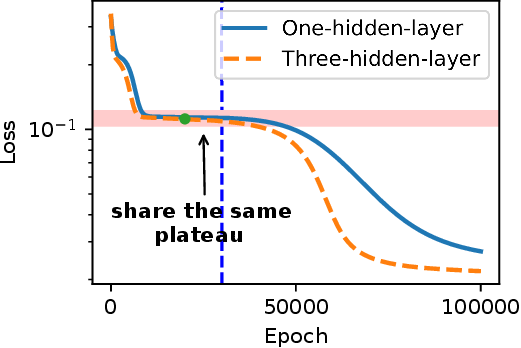}
}\hspace{1.5cm}
\subfigure[Output at the plateau]{
\includegraphics[height=0.25\textwidth]{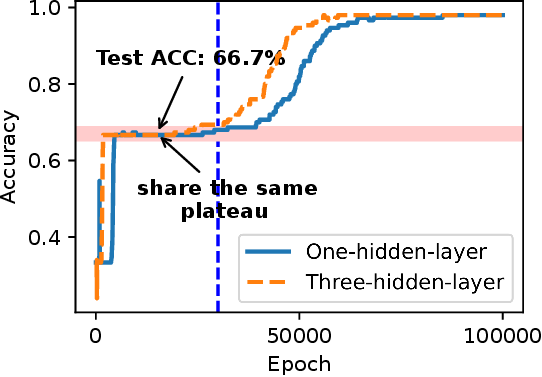}
}\\
\subfigure[Evolution of $\operatorname{MPC}$]{
\includegraphics[height=0.25\textwidth]{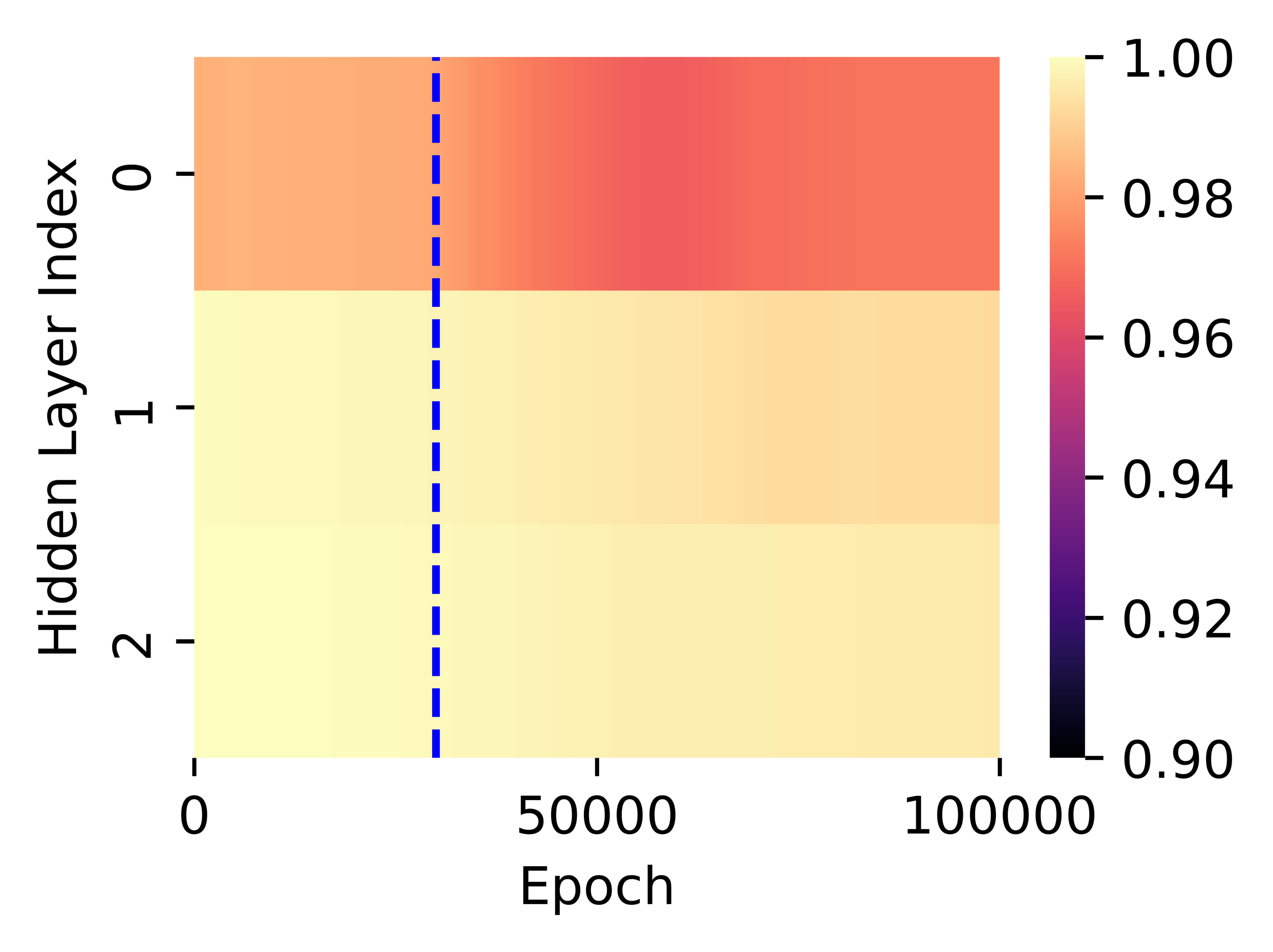}
}\hspace{1.5cm}
\subfigure[Loss of reduced NN]{
\includegraphics[height=0.25\textwidth]{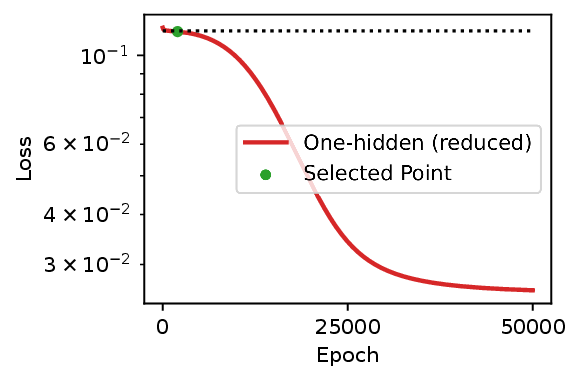}
}
\caption{
\textbf{Deep neural networks encounter lifted critical points during training on Iris data.} (a) The training loss for single-hidden-layer and three-hidden-layer NNs with width $m=50$. (b) The training accuracy of NNs with different depths at the same loss value indicated by the colored span in (a). The accuracy plateau is at 66.7\% for both train and test sets. (c) The extent of layer linearization for different hidden layers during the training process of the three-hidden-layer NN. (d) Training loss trajectory of the reduced single-hidden- layer NN. The green dot in (a) and (c) is selected as a representative for comparison.}
\label{fig:3-hidden-nonlinear-iris}
\end{figure}

We observe similar phenomena for ReLU NNs and residual-connected NNs (ResNets) in Fig.~\ref{fig:ReLU} and Fig.~\ref{fig:res-nonlinear} in Appendix~\ref{app:sup-exp}. These results confirm that deep NNs indeed encounter critical points lifted from shallower NNs with small initialization. Moreover, the study conducted by~\cite{kudugunta2019investigating} utilized linear centered kernel alignment (CKA) as a metric to assess the similarity between different layers. Their experiments on CIFAR-10 and ImageNet-1000 revealed that a "block structure" emerges when the network size is much larger than the dataset size, with many layers exhibiting a high degree of similarity. In Prop.~\ref{APP-prop-CKA} of Appendix~\ref{app:proofs}, we prove that this similarity between layers indeed reflects the degree of linear correlation between representations across layers. We further investigate the impact of initialization and dataset size on layer linearization in Appendix~\ref{app:sup-exp} (see Fig.~\ref{fig:Initialization}). Generally speaking, it is common to observe layer linearization when a deep network is trained on a simple task with small initialization. This occurrence of layer linearization significantly reduces the network's complexity and thus may be crucial in contributing to the generalization of deep networks.
\begin{remark}
Generally, reducing the scale of initialization enhances the nonlinear feature learning dynamics of DNNs. With a properly small initialization, the training dynamics tend to experience lifted critical points from shallower NNs, which implicitly help control the complexity of DNNs during training. However, a too small initialization scale may result in undesired prolonged stagnation at critical points. Remarkably, standard initialization schemes like Xavier/Kaiming-He strike a balance between feature learning and training efficiency with a proper initialization scale, ensuring that the training dynamics are both efficient and sufficiently nonlinear.
\end{remark}

\begin{figure}[t]
\centering
\subfigure[Synthetic data (ReLU)]{
\includegraphics[height=0.32\textwidth]{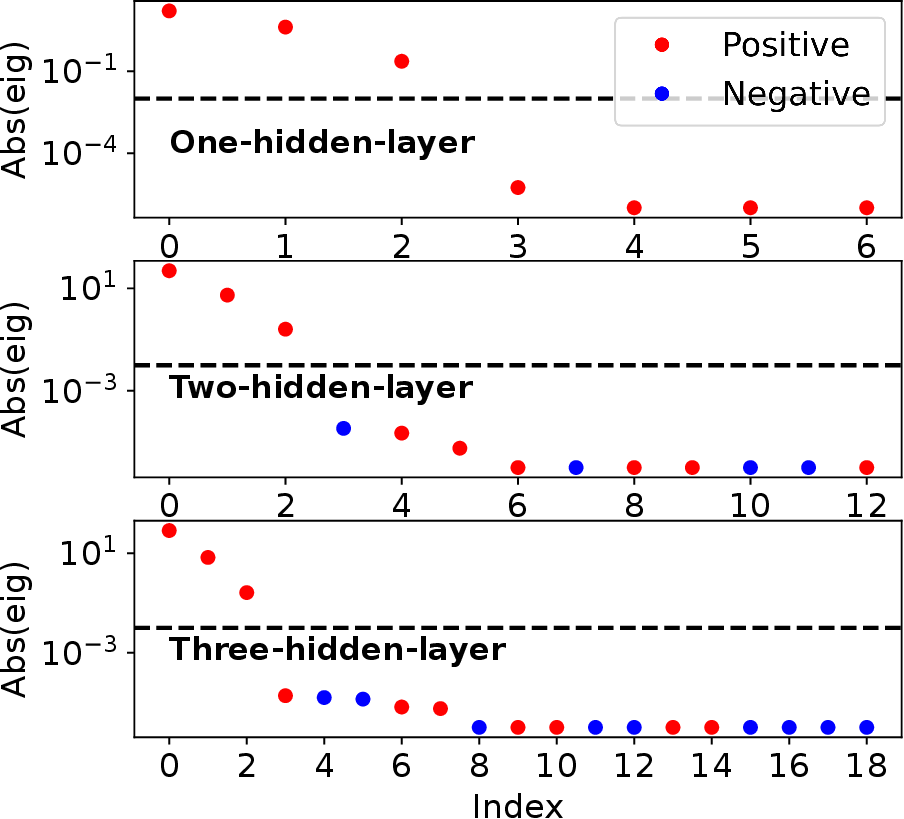}
}\hspace{1.cm}
\subfigure[Iris data (ReLU)]{
\includegraphics[height=0.32\textwidth]{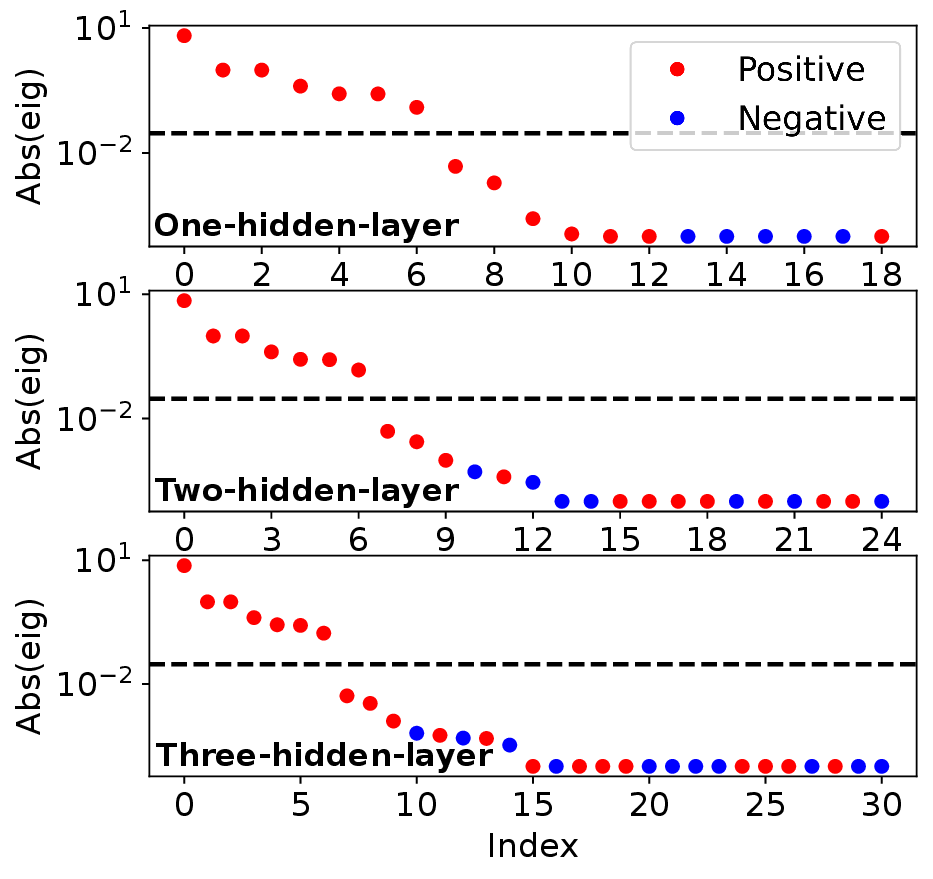}
}
\caption{\textbf{Incremental degeneracy of critical points through embedding.} (a, b) The eigenvalues of Hessian of ReLU NNs at the critical points embedded from the single hidden layer NN for learning data in Fig.~\ref{fig:3-hidden-nonlinear}(b) and Iris data, respectively. The results for each plot are averaged over 100 random orthogonal similarity transformations. The auxiliary dashed lines in (a, b) delineate the empirical boundary between zero and non-zero eigenvalues. We perform the embedding operation by factorizing one hidden layer into $k$ hidden layers ($k=2, 3$), whose input weights are identity and biases are selected to translate the input range into the affine subdomain.}

\label{fig:eigenvalues-new}
\end{figure}

\subsubsection{Incremental degeneracy of critical points through embedding} Empirical studies have shown that the Hessian matrix of the minimizer, derived from training, possesses a significant count of zero or near-zero eigenvalues, highlighting the existence of highly degenerate critical points within the loss landscape~\cite{sagun2016singularity}. Our findings, as illustrated in Figs.~\ref{fig:3-hidden-nonlinear} and~\ref{fig:3-hidden-nonlinear-iris}, imply that deep networks often encounter critical points inherited from their shallower counterparts. To empirically validate that lifted critical points have higher degeneracy, we design experiments to compute the eigenvalues of their respective Hessian matrices at the empirical critical points.

Specifically, we train a single-hidden-layer ReLU NN with width $m=2$ to learn the data in Fig. \ref{fig:3-hidden-nonlinear}(b) shown in Fig. \ref{fig:eigenvalues-new}(a) or the Iris dataset in Fig. \ref{fig:eigenvalues-new}(b) to a empirical critical point (the $L_1$ norm of the gradient $\leq 10^{-4}$). We then embed this critical point through a one-layer embedding and a two-layer embedding to NNs with 2 hidden layers and 3 hidden layers, with each hidden layer having the same width, respectively. To compute the eigenvalues of the Hessian matrix with a large condition number accurately, we conducted 100 random orthogonal similarity transformations on the matrix. We took the average from these 100 trials to obtain a more reliable set of eigenvalues. We then pinpoint locations where there are evident gaps in eigenvalue magnitudes, as delineated by the auxiliary line in Fig.~\ref{fig:eigenvalues-new}. This allows for differentiation between zero and non-zero eigenvalues, serving as a mechanism to ascertain empirical degeneracy. Detailed methodology can be found in Appendix~\ref{app:tra-details}. As illustrated in Fig.~\ref{fig:eigenvalues-new}, each embedding step introduces six more zero eigenvalues to the Hessian matrix due to the introduction of two neurons, resulting in six additional parameters. This finding is consistent with Prop.~\ref{prop:inertia-preserving} and potentially elucidates the origin of a particular type of degeneracy at critical points within the loss landscape.

\subsection{The effects of batch normalization and larger dataset}
\label{sec:BN}
Since layer linearization is a key feature in encountering lifted critical points and results in slower training, we investigate how batch normalization and larger training sets can affect layer linearization and subsequently lead to accelerated training.
\subsubsection{Batch normalization avoids lifted critical points}

Batch normalization (BN) normalizes the layers' inputs by re-centering and re-scaling: $\mathrm{BN}(\vx)=\boldsymbol{\gamma} \circ \frac{\vx-\hat{\boldsymbol{\mu}}_{\mathcal{B}}}{ \hat{\boldsymbol{\sigma}}_{\mathcal{B}}}+\boldsymbol{\beta}$ with a default initialization $\vgamma=\boldsymbol{1}, \vbeta=\boldsymbol{0}$. Empirically, using BN can greatly speed up NN training. 
Intuitively, when a neuron's input range becomes too small, its nonlinear activation function behaves effectively like a linear function. In such cases, batch normalization can enhance the nonlinearity of each neuron by effectively rescaling its input range to $\fO(1)$ and thus suppress layer linearization.
As embedding principle in depth unravels a large family of lifted critical points with layer linearization, avoiding these critical points through suppressing layer linearization may be an important mechanism underlying the training efficiency of BN in practice. 

To verify this mechanism, we perform the following experiment shown in Fig.~\ref{fig:different-BN-init}. We train a 2-hidden-layer tanh NN (width $m=50$) to learn the data in Fig.~\ref{fig:3-hidden-nonlinear}(b), and compare the training trajectories without BN and with BN of scaling parameter $\vgamma$ initialized at $\boldsymbol{0.1}$, $\boldsymbol{1.0}$, or $\boldsymbol{1.5}$. Conforming with our intuition, a larger $\vgamma$ better suppresses layer linearization throughout the training as shown in Fig.~\ref{fig:different-BN-init}(b-d). Moreover, the stagnation is significantly alleviated with a larger $\vgamma$. There is even no stagnation for $\vgamma$ initialized at $\boldsymbol{1.5}$, signifying complete avoidance of lifted critical points as predicted by above mechanism.
\begin{figure*}[t]
\centering
\subfigure[Training loss]{
\includegraphics[height=0.27 \textwidth]{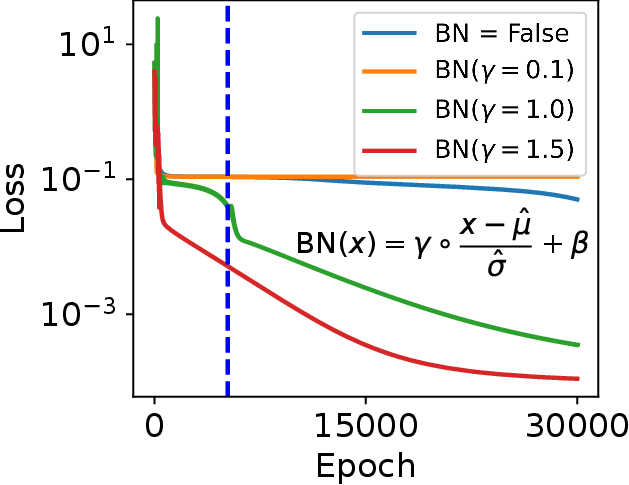}
}\hspace{1.5cm}
\subfigure[$\vgamma = 0.1$]{
\includegraphics[height=0.28 \textwidth]{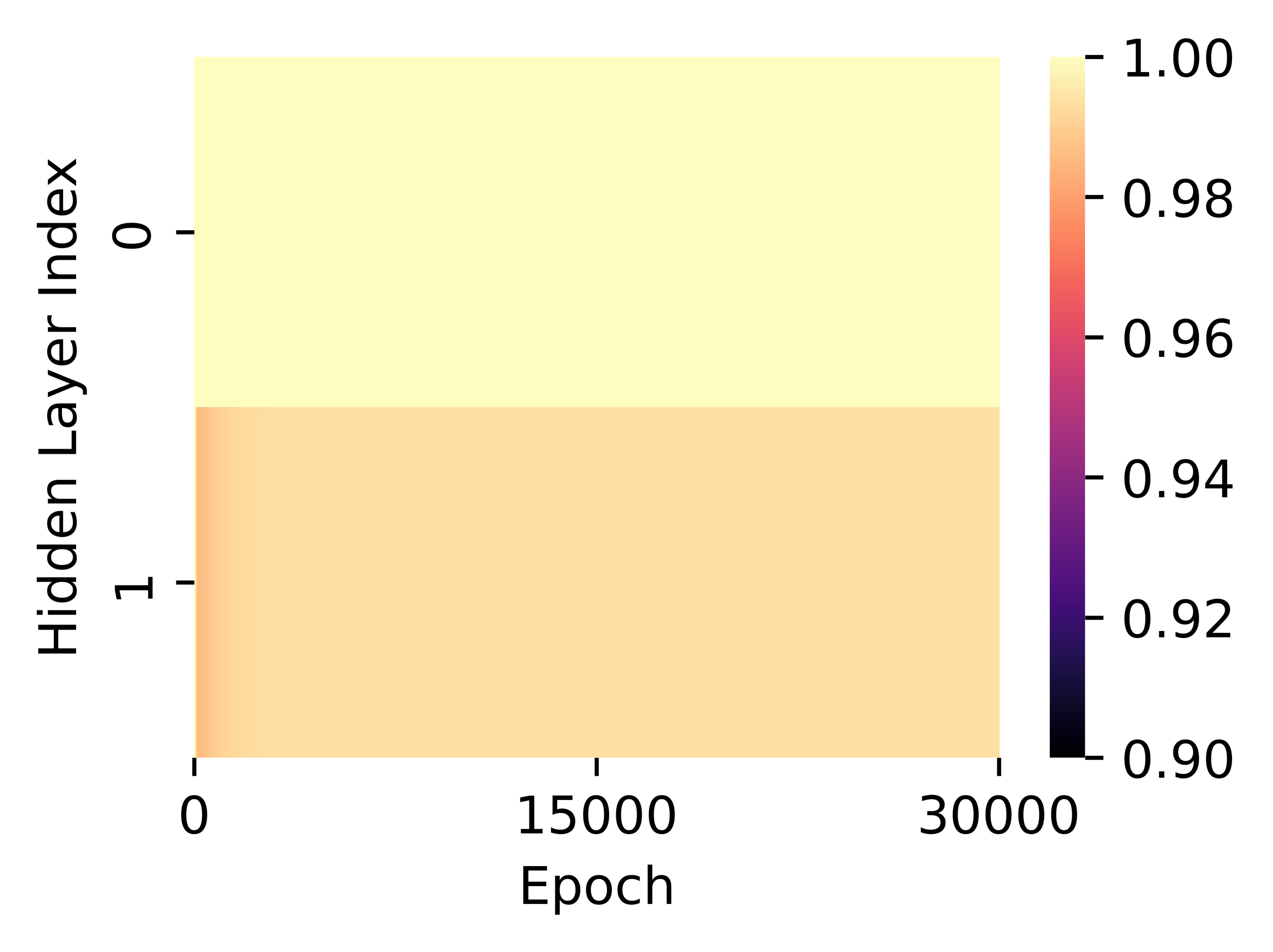}
}\\
\subfigure[$\vgamma = 1.0$]{
\includegraphics[height=0.28 \textwidth]{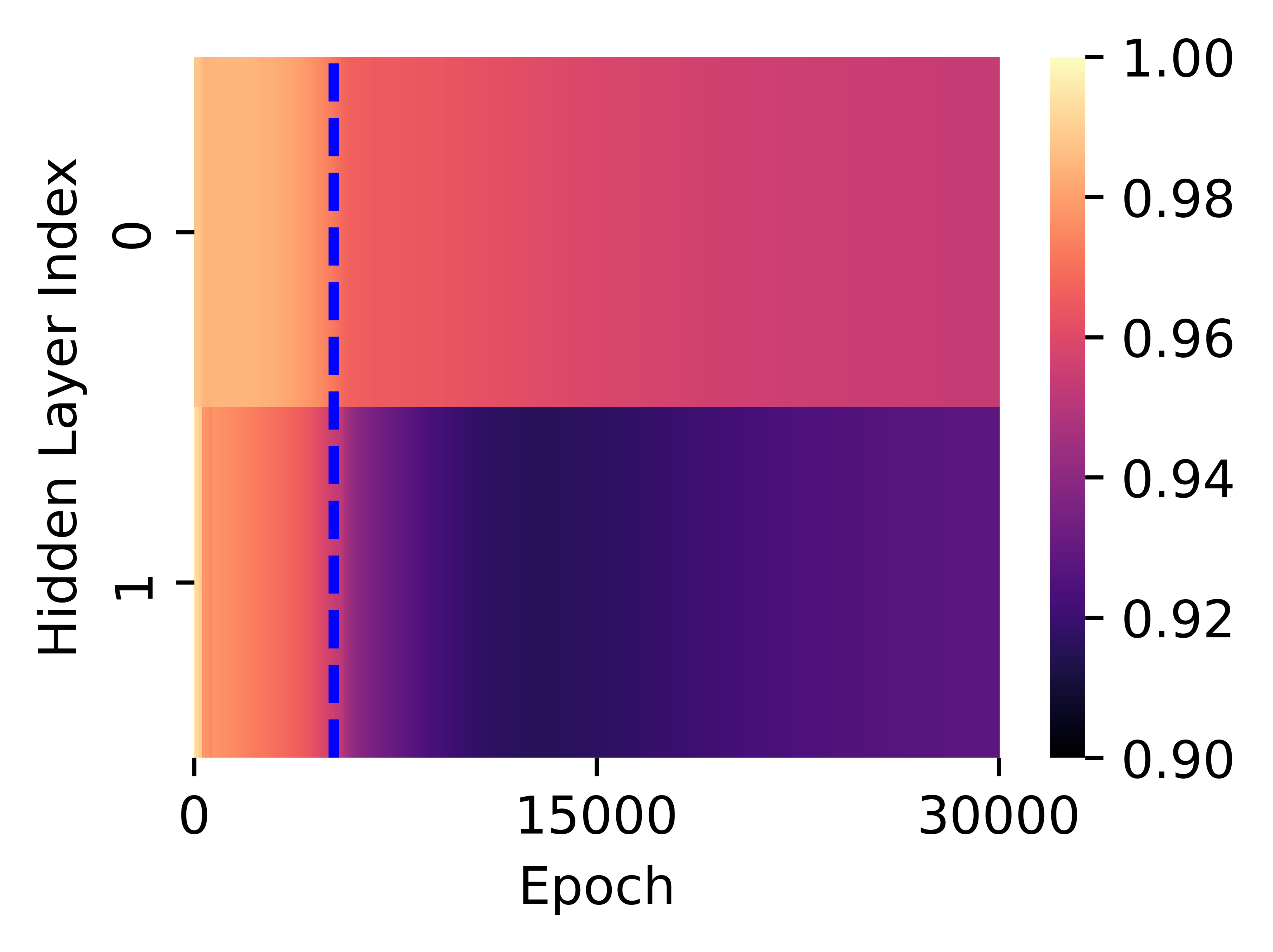}
}\hspace{1.5cm}
\subfigure[$\vgamma = 1.5$]{
\includegraphics[height=0.28 \textwidth]{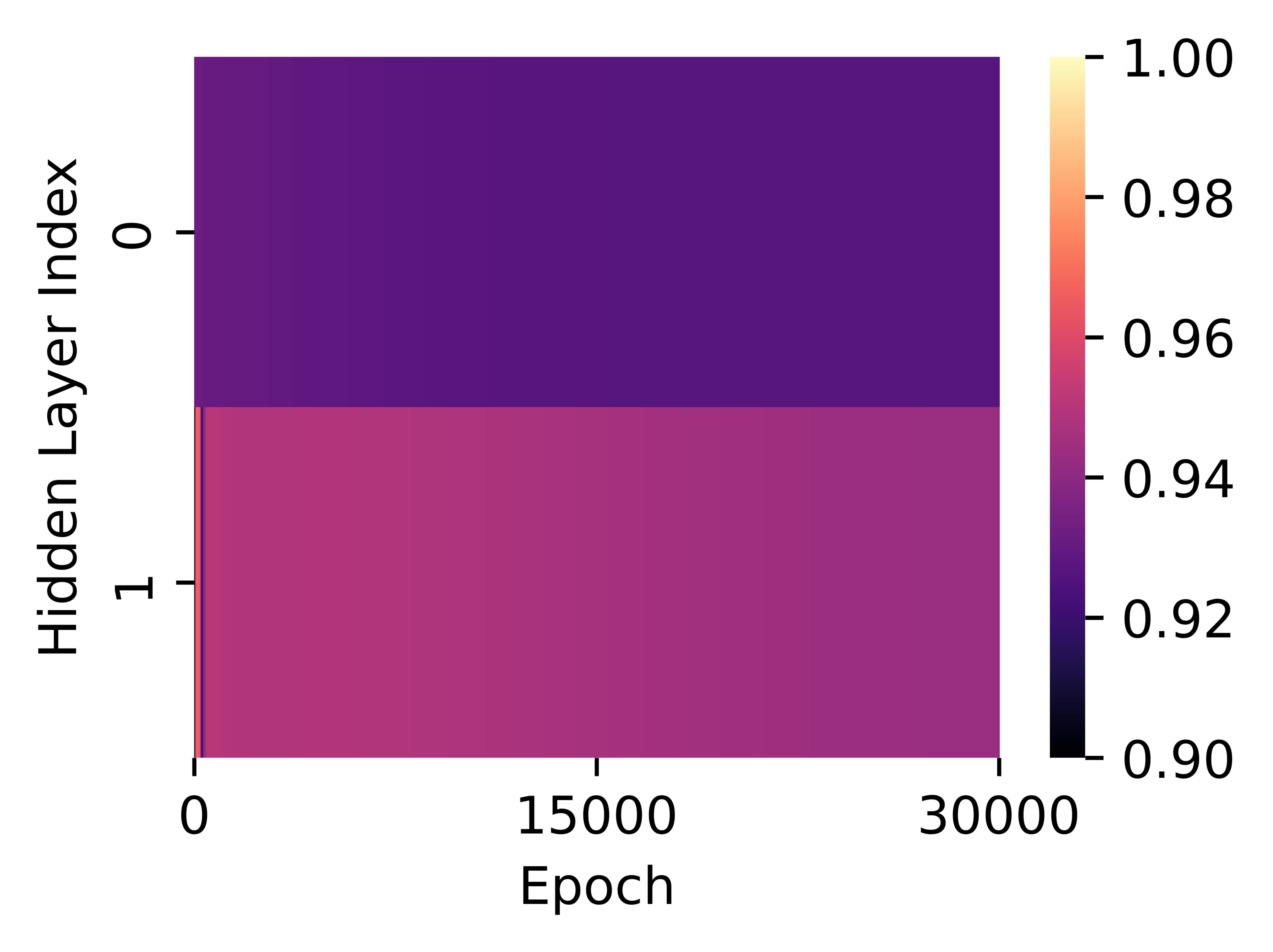}
}
\caption{\textbf{Batch normalization avoids lifted critical points during training.}
(a) Trajectories of training loss without BN and with BN of different initial values.  (b-d) The extent of layer linearization for all hidden layers with BN of scaling parameter $\vgamma$ initialized at $0.1$, $1.0$, or $1.5$, respectively. The auxiliary dash lines in (a) and (c) correspond to the same epoch.}
\label{fig:different-BN-init}
\end{figure*}

\subsubsection{Optimization benefit of larger dataset}

Prop.~\ref{thm:data-depen} characterizes the data dependency of critical lifting. The intuition is that a larger dataset increases the difficulty of layer linearization, thus helping reduce the critical manifolds.  This result provides a seemingly counter-intuitive prediction that larger dataset may be more easily fitted due to the reduced critical manifolds lifted from shallower NNs. This prediction is verified by the following experiment in Fig.~\ref{fig:data-dependency2}. 

We train a tanh NN with $3$ hidden layers (width $m=50$) to learn data (data size $n=70$) of Fig.~\ref{fig:3-hidden-nonlinear}(b) to a critical point (the red point in Fig.~\ref{fig:data-dependency2}(a)). We then continue (blue curve) or switch to larger datasets (orange and green curves) for training. As shown in Fig.~\ref{fig:data-dependency2}(a), more training data leads to faster escape from the lifted critical point. Fig.~\ref{fig:data-dependency2}(b-d) further trace the extent of layer linearization for each hidden layer on datasets of different sizes, respectively. From Fig.~\ref{fig:data-dependency2}(c-d), we can clearly see that more data facilitate the expression of nonlinearity of hidden layers (see the abrupt reductions of $\operatorname{MPC}$ at the red dot). This helps the network to escape the critical manifold lifted from a single-hidden-layer NN, which aligns with the implications of Prop.~\ref{thm:data-depen}.

\begin{figure}[t]
\centering
\subfigure[Training loss]{
\includegraphics[height=0.27 \textwidth]{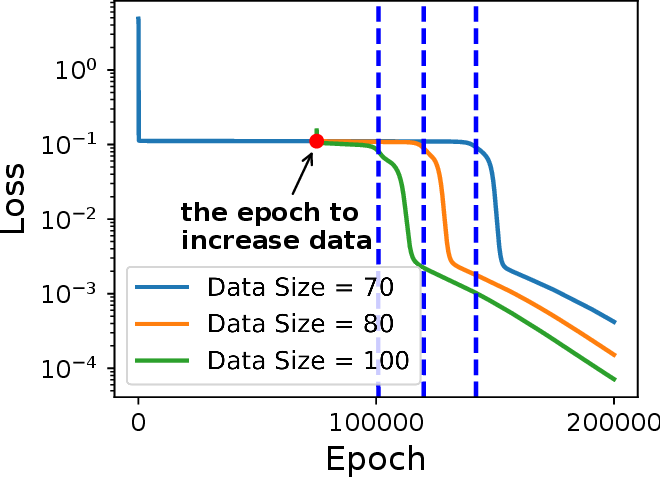}
}\hspace{1.5cm}
\subfigure[Data size $= 70$]{
\includegraphics[height=0.28 \textwidth]{figures/data-dependency-pear1.eps}
}\\
\subfigure[Data size $= 80$]{
\includegraphics[height=0.28 \textwidth]{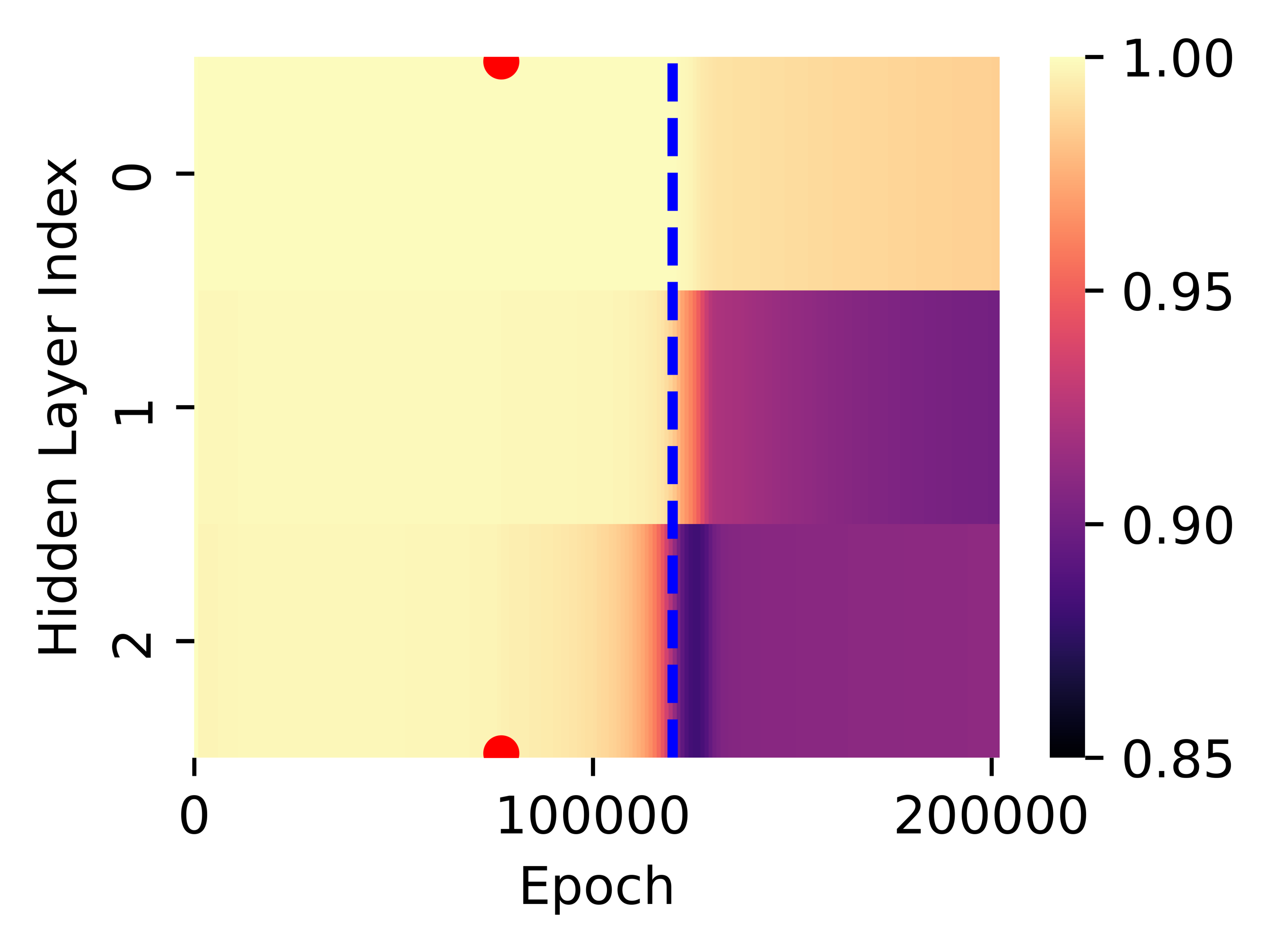}
}\hspace{1.5cm}
\subfigure[Data size $= 100$]{
\includegraphics[height=0.28 \textwidth]{figures/data-dependency-pear3.eps}
}
\caption{\textbf{Optimization benefit of larger dataset.} (a) The training loss of three-hidden-layer NN with width $m=50$ for learning data of Fig.~\ref{fig:3-hidden-nonlinear}(b). We sample 70, 80 and 100 data points equally spaced around 0, respectively. The red dot in (a) is selected for switching dataset.  (b-d) The extent of layer linearization for all hidden layers on datasets of different size, respectively. The red dots correspond to the epoch to switch dataset. The auxiliary dash lines correspond to the epoch where NNs escape from the lifted critical manifold.}
\label{fig:data-dependency2}
\end{figure}

\subsection{Network pruning}\label{sec:network-pruning}

\subsubsection{Layer pruning of DNNs with layer linearization} The embedding principle in depth predicts a family of critical points with layer linearization. These critical points intrinsically come from shallower NNs, thus possessing good layer pruning potential. To realize such pruning potential in practice, we propose the method of detecting and merging effectively linear layers, which works as follows. We train a deep $10$-hidden-layer tanh NN (width $m=50$) on the MNIST dataset. At the red dot in Fig.~\ref{fig:MNIST-exp-redo}(a), the training loss decreases very slowly, presumably is very close to a global minimum. As shown in Fig.~\ref{fig:MNIST-exp-redo}(b), there are $5$ effective linear layers ($\operatorname{MPC}>0.99$) at this point. We merge these effective linear layers by properly multiplying their weights, thereby pruning the NN of $10$ hidden layers to $5$ hidden layers. The parameters before reduction is denoted by $\vtheta_{\text{ori}}$ and after reduction by $\vtheta_{\text{redu}}$.
We further train the pruned NN from $\vtheta_{\text{redu}}$ as shown in Fig.~\ref{fig:MNIST-exp-redo}(c) which quickly fall into the same loss value as the red point in Fig.~\ref{fig:MNIST-exp-redo}(a). We then compare the prediction between original model and the pruned model at the corresponding red point on $10000$ test data as shown in Fig.~\ref{fig:MNIST-exp-redo}(d). Although our critical lifting does not preserve the output function over the entire domain of input, we still observe well agreement of these two models (overall $\sim 98.54\%$), which implies that this reduction can approximately preserve the generalization performance ($95.4\%$ to $95.27\%$).
\begin{figure}[t]
\centering
\subfigure[Initial NN]{
\includegraphics[height=0.26\textwidth]{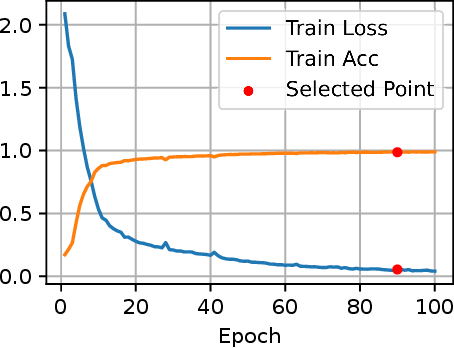}
}\hspace{1.5cm}
\subfigure[Evolution of $\operatorname{MPC}$]{
\includegraphics[height=0.28\textwidth]{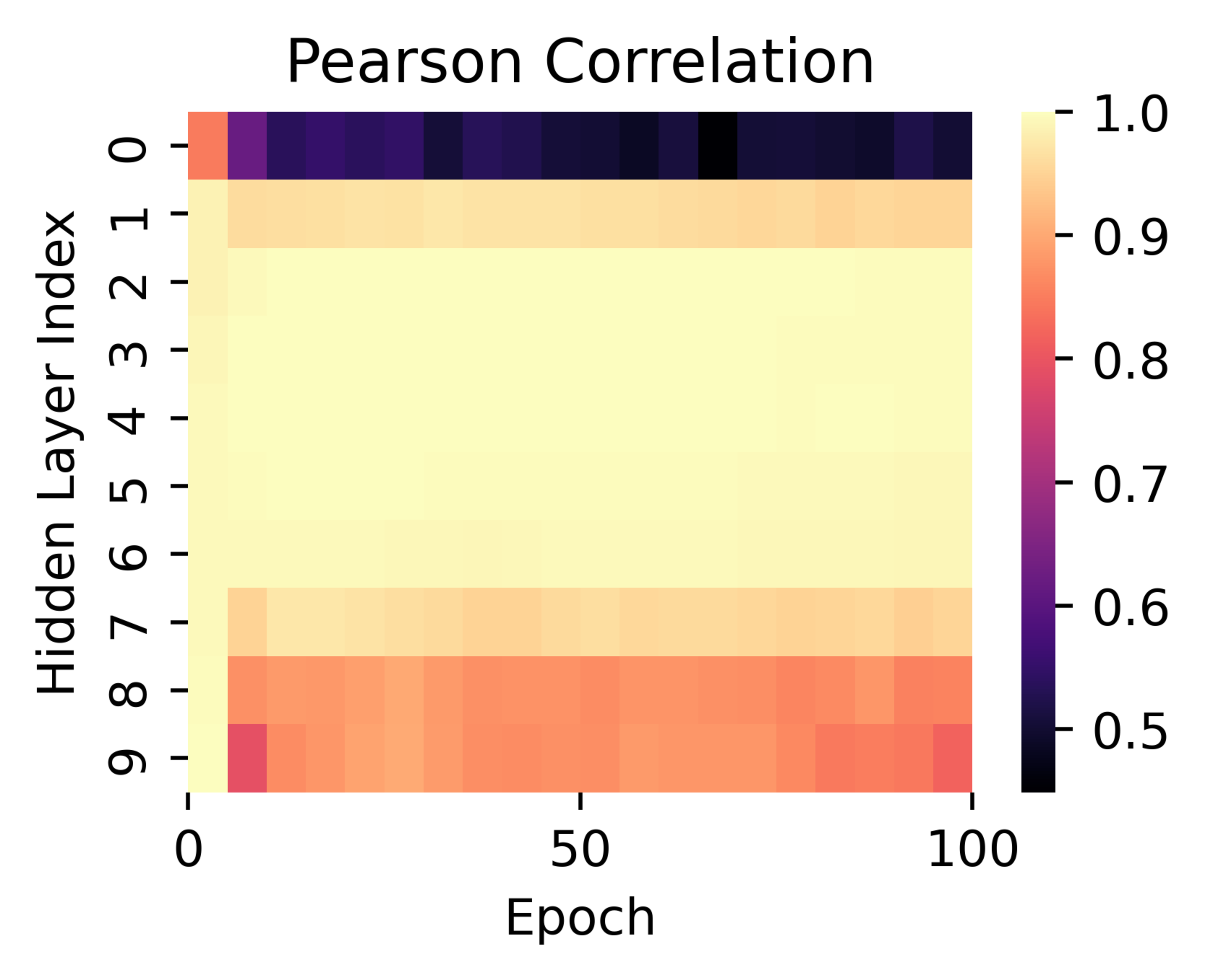}
}\\
\subfigure[Pruned NN]{
\includegraphics[height=0.26\textwidth]{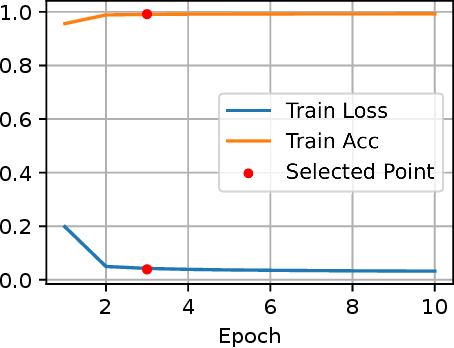}
}\hspace{1.5cm}
\subfigure[Prediction similarity]{
\includegraphics[height=0.28\textwidth]{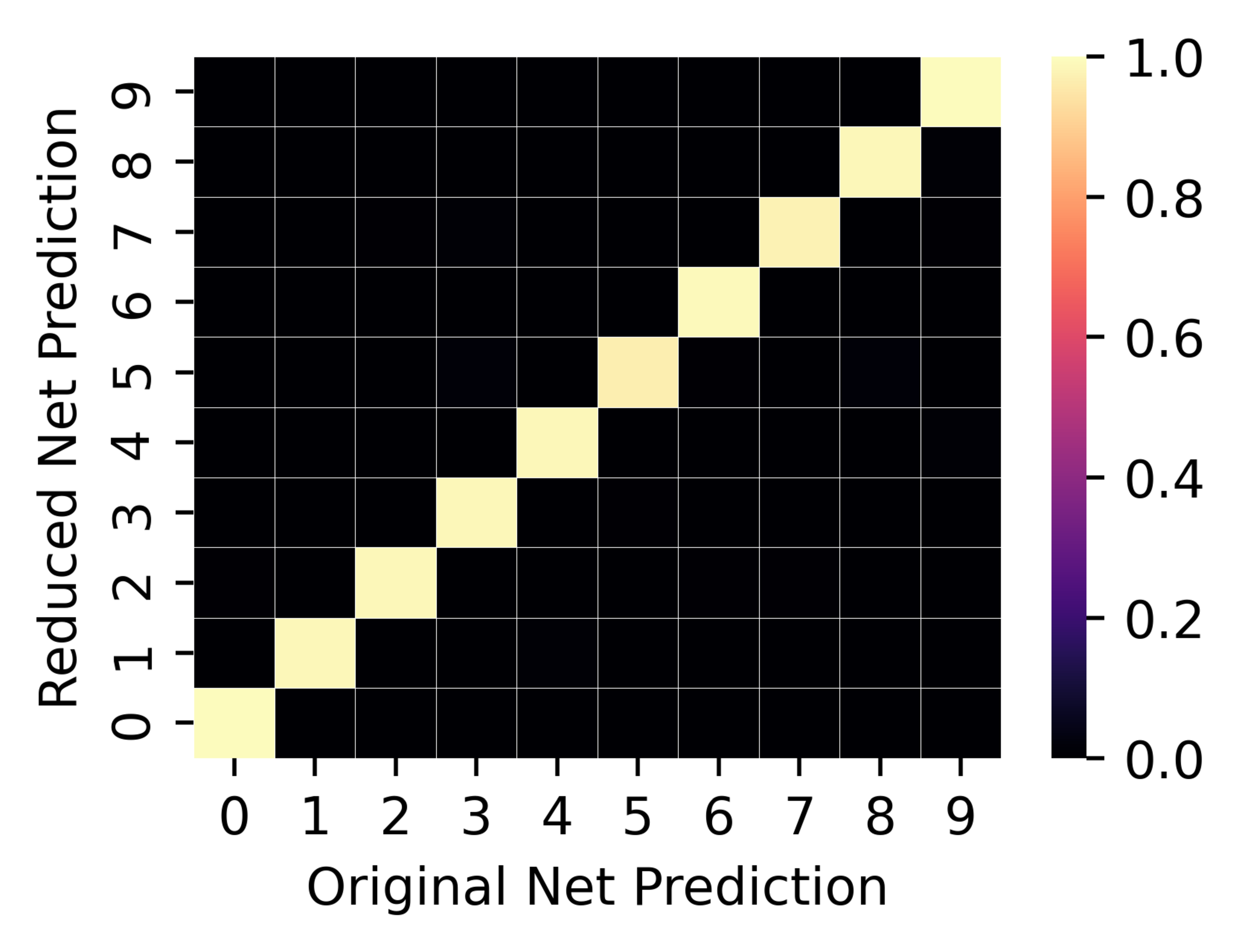}
}
\caption{\textbf{Layer pruning of DNNs with layer linearization.} (a) The training process of the original $10$-hidden-layer network on MNIST dataset. The red dot is selected for layer pruning. (b) The extent of layer linearization for all hidden layers during the training process. (c) The training process after layer pruning. (d) Prediction similarity between initial and reduced network on the test dataset. For each grid, color indicates the ratio of that prediction pair $(i,j)$ over all samples predicted as $j$ by the original $10$-hidden-layer NN.}
\label{fig:MNIST-exp-redo}
\end{figure}

\section{Discussion}
\label{sec:discussion}
\subsection{Differences between Embedding Principle in Width and in Depth}
Our work draws inspiration from the embedding principle in width~\cite{zhang2021embedding}, but, for the first time, addresses the embedding relation in depth for DNNs.  Though our depth-based critical lifting operator shares the same spirit with the width-based critical embedding operator in~\cite{zhang2021embedding}, it is important to note several key distinctions, as follows:

(i) \textbf{The target NN.} The depth-based critical lifting maps to the parameter space of a deeper NN whereas the width-based critical embedding maps to the parameter space of a wider NN.

(ii) \textbf{The requirement for the activation function.} The depth-based critical lifting requires a layer linearization condition, which can be satisfied for any activation with a non-constant linear segment, such as ReLU, leaky-ReLU, and ELU, and can be approximately satisfied for general smooth activations, including sigmoid, tanh, and gelu. On the other hand, the width-based critical embedding works for any activation function.

(iii) \textbf{The type of mapping.} The depth-based critical lifting is a set-valued function which maps any parameter vector to a manifold, whereas the width-based critical embedding is a vector-valued function.

(iv) \textbf{The output preserving property.} The depth-based critical lifting preserves the DNN outputs at the training dataset, whereas the width-based critical embedding preserves the DNN output function over the entire input domain.

(v) \textbf{The data dependency.} The depth-based critical lifting is data-dependent, whereas the width-based embedding operator is independent on data. For depth-based critical lifting, we prove that more data leads to reduced lifted manifolds in Prop.~\ref{thm:data-depen}, whereas the width-based critical embedding is data-independent.

The key distinctions (iii)-(v) stem from the intrinsic differences in expressiveness when adding width versus depth to a DNN. Specifically, the function space of a narrow NN is strictly a subset of the function space of any wider NN, as referenced in \cite{fukumizu2019semi,zhang2021embedding}. However, a similar (embedding) relation regarding the expressiveness can generally only be expected in the sense of approximation when comparing shallower and deeper NNs.

\subsection{Other network architectures}
The embedding principle in depth is derived from the inherent layer stacking nature of deep neural networks, where each layer can be linearly factorized into multiple layers. Therefore, our theoretical findings proven for fully-connected DNNs can be naturally extended to other neural network architectures, such as convolutional neural networks and residual-connected neural networks. For instance, when dealing with a residual-connected network, the only alteration required to derive its one-layer lifting operator is a modification of the output-preserving condition (see Appendix~\ref{app:proofs}: Def.~\ref{APP:def:one-layer-residual-lifting} for details).

It should be noted that for residual-connected networks, simply assigning zero values to the parameters of the inserted block can create a trivially criticality-preserving lifting, which naturally satisfies the layer linearization condition. However, it is important to underscore that the practical experience of encountering these lifted critical points during training is a more crucial determinant of the importance of different lifting techniques than the mere existence of such embeddings. 
Indeed, such trivially lifted critical points are seldom observed in practice. On the contrary, the lifted critical points we propose, which are accompanied by layer linearization, are frequently observable in experimental setups, as illustrated in Fig.~\ref{fig:res-nonlinear} in Appendix~\ref{app:sup-exp}. Therefore, the critical lifting operator proposed in this work is of special value for studying the practical training behavior of DNNs.

\subsection{Simplicity bias}
Our embedding principle in depth, combined with the prior embedding principle in width, explicitly characterizes the hierarchical structure of DNN loss landscape in both width and depth dimensions. This hierarchical structure highlights the potential for non-overfitting, even when a significantly large NN is employed to fit limited training data generated by a comparatively smaller (i.e., shallower and narrower) NN. The intuition here is that a large NN, guided by the hierarchy of "simple" critical points/manifolds lifted/embedded from shallower and narrower NNs, may learn a "simple" interpolation from a small NN through training. This potential is further corroborated by our numerical experiments shown in Figs.~\ref{fig:3-hidden-nonlinear}, \ref{fig:3-hidden-nonlinear-iris}, \ref{fig:different-BN-init}, \ref{fig:data-dependency2} and \ref{fig:MNIST-exp-redo}, which indicate that the "effective" depth, i.e., the number of nonlinear layers, often gradually increases during the training of deep NNs. In light of these findings, it is tempting to conjecture that, with proper initialization, a large DNN could adaptively increase its "effective" depth and width based on the complexity of the training data. We will examine this conjecture in our future works.
\section{Conclusion} \label{sec:conclusion}

In this paper, we discover an embedding principle in depth, establishing that the loss landscape of a deep NN inherits all critical points from shallower NNs. We introduce the critical lifting operator that serves to prove this principle and provide comprehensive details about it. Furthermore, we offer empirical evidence demonstrating the vast insights provided by this principle, which contribute to the highly degenerate critical points of deep networks, the acceleration effect of batch normalization and larger datasets, and the process of layer pruning. It should be noted that the experiments conducted in this work serve as proofs of concept for these novel insights. Further systematic experimental studies are needed to gain a full understanding of the practical significance of these insights.

Overall, our discovery of the embedding principle in depth, together with the previous embedding principle in width, provides a comprehensive picture of the intrinsic hierarchical structure of the DNN loss landscape. This picture strongly supports the empirically observed similarities in training and generalization between NNs of varying sizes, thereby shedding light on the non-overfitting mystery of large NNs.

\section*{Acknowledgments}
This work is sponsored by the National Key R\&D Program of China  Grant No. 2022YFA1008200 (T.L., Z.X., Y.Z.,), the National Natural Science Foundation of China Grant No. 92270001 (Z.X), No. 12101402 (Y.Z.), No. 12371511 (Z.X.), No. 12101401 (T.L.), Shanghai Municipal Science and Technology Key Project No. 22JC1401500 (T. L.), Shanghai Municipal of Science and Technology Project Grant No. 20JC1419500 (Y.Z.), the Lingang Laboratory Grant No. LG-QS-202202-08 (Y.Z.), Shanghai Municipal of Science and Technology Major Project No. 2021SHZDZX0102, the HPC of School of Mathematical Sciences and the Student Innovation Center, and the Siyuan-1 cluster supported by the Center for High Performance Computing at Shanghai Jiao Tong University, Key Laboratory of Marine Intelligent Equipment and System, Ministry of Education, P.R. China.



\appendix
\section{Proofs}
\label{app:proofs}
In this section, we give all proofs for our theoretical results mentioned in the main text.
\renewcommand\thefigure{A\arabic{figure}}    
\setcounter{figure}{0}    
\newtheorem{oldtheorem}{Theorem}
\renewcommand{\theoldtheorem}{A.\arabic{oldtheorem}}
\setcounter{oldtheorem}{0}

\newtheorem{oldlemma}{Lemma}
\renewcommand{\theoldlemma}{A.\arabic{oldlemma}}
\setcounter{oldlemma}{0}

\newtheorem{olddefinition}{Definition}
\renewcommand{\theolddefinition}{A.\arabic{olddefinition}}
\setcounter{olddefinition}{0}

\newtheorem{oldprop}{Proposition}
\renewcommand{\theoldprop}{A.\arabic{oldprop}}
\setcounter{oldprop}{0}

\newtheorem{oldcor}{Corollary}
\renewcommand{\theoldcor}{A.\arabic{oldcor}}
\setcounter{oldcor}{0}

\begin{olddefinition}[\textbf{one-layer lifting}]
\label{APP:def:one-layer-lifting}
Given data $S$, consider an $\mathrm{NN}\bigl(\left\{m_{l}\right\}_{l=0}^{L}\bigr)$ and its one-layer deeper counterpart, $\mathrm{NN}^{\prime}\big(\{m_l^\prime\}$, $l\in \{0, 1, 2, \cdots, q, \hat{q}, q+1, \cdots, L\}\big)$. The one-layer lifting, denoted as $\fT_S$,  is a function that transforms any parameter $\vtheta=\left(\boldsymbol{W}^{[1]}, \boldsymbol{b}^{[1]}, \cdots, \boldsymbol{W}^{[L]}, \boldsymbol{b}^{[L]}\right)$ of $\mathrm{NN}$ into a set $\fM$ within the parameter space of $\mathrm{NN}^{\prime}$. Formally, 
$\fM$ (where $\fM:= \fT_S(\vtheta))$ represents a collection of all possible parameters $\vtheta'$ of $\mathrm{NN}^{\prime}$
that satisfying the following three conditions:

(i) local-in-layer condition: weights of each layer in $\mathrm{NN}^{\prime}$ are inherited from $\mathrm{NN}$ except for layer $\hat{q}$ and $q+1$, i.e.,  
\begin{equation*}
   \left\{ \begin{aligned}
& \vtheta^{\prime}|_l = \boldsymbol{\theta} |_l,\quad \text{for}\quad  l \in [q]\cup [q+2:L],\\
& \vtheta^{\prime}|_{\hat{q}} = \bigl(\boldsymbol{W'}^{[\hat{q}]},  \boldsymbol{b'}^{[\hat{q}]}\bigr)\in\mathbb{R}^{m'_{\hat{q}}\times m'_{q-1}}\times \mathbb{R}^{m'_{\hat{q}}},\\
& \vtheta^{\prime}|_{q+1} = \bigl(\boldsymbol{W'}^{[q+1]},  \boldsymbol{b'}^{[q+1]}\bigr)\in\mathbb{R}^{m'_{q+1}\times m'_{\hat{q}}}\times \mathbb{R}^{m'_{q+1}}.\\
\end{aligned}
\right. 
\end{equation*}

(ii) layer linearization condition: for any $j\in [m_{\hat{q}}]$, there exists an affine subdomain $(a_j, b_j)$ of $\sigma$ associated with $\lambda_j, \mu_j$ such that the $j$-th component $\big(\boldsymbol{W'}^{[\hat{q}]}\boldsymbol{f}_{\boldsymbol{\theta}'}^{[q]}(\boldsymbol{x}) + \boldsymbol{b'}^{[\hat{q}]}\big)_j \in (a_j, b_j)$ for any $\vx \in S_{\vx}.$

(iii) output preserving condition:
\begin{equation*}
\left\{\begin{aligned}
&  \boldsymbol{W'}^{[q+1]}\operatorname{diag}(\boldsymbol{\lambda})\boldsymbol{W'}^{[\hat{q}]} = \boldsymbol{W}^{[q+1]},\\
& \boldsymbol{W'}^{[q+1]}\operatorname{diag}(\boldsymbol{\lambda})\boldsymbol{b'}^{[\hat{q}]}+\boldsymbol{W'}^{[q+1]}\boldsymbol{\mu} + \boldsymbol{b'}^{[q+1]} = \boldsymbol{b}^{[q+1]},   
\end{aligned}\right.  
\end{equation*}
where $\vlambda = [\lambda_1, \lambda_2, \cdots, \lambda_{m_{\hat{q}}}]^\top\in \mathbb{R}^{m^{\prime}_{\hat{q}}}, \vmu = [\mu_1, \mu_2, \cdots, \mu_{m_{\hat{q}}}]^\top\in\mathbb{R}^{m^{\prime}_{\hat{q}}}$, and $\operatorname{diag}(\boldsymbol{\lambda})$ denotes the diagonal matrix formed by vector $\vlambda$.

\end{olddefinition}
\subsection{Existence of one-layer lifting}
\begin{oldlemma}[\textbf{existence of one-layer lifting}]
\label{APP:existence}
Given data $S$, an $\mathrm{NN}\bigl(\left\{m_{l}\right\}_{l=0}^{L}\bigr)$ and its one-layer deeper counterpart,  $\mathrm{NN}^{\prime}\big(\{m_l^\prime\}$, $l\in \{0, 1, 2, \cdots, q, \hat{q}, q+1, \cdots, L\}\big)$, the one-layer lifting $\fT_S$ exists, i.e., 
$\fT_S(\vtheta_{\rshal})$ is not empty for any parameter $\vtheta_{\rshal}$ of $\mathrm{NN}$. 
\end{oldlemma}
\begin{proof}
We prove this lemma by construction. From the definition of one-layer deeper, we know that $m_1^\prime = m_1, \cdots, m_{q}^\prime = m_q, m_{\hat{q}}^\prime\geq \min\{m_{q}, m_{q+1}\}, m_{q+1}^\prime = m_{q+1}, \cdots, m_L^\prime = m_L$. Without loss of generality, we assume the width of the inserted layer $m^{\prime}_{\hat{q}}$ is equal to $\min\{m_q, m_{q+1}\}$. Our construction can be easily extended to the case with a wider inserted layer by adding zero-neurons, i.e., neurons whose input and output weights are all zero.

For any parameter $\vtheta_{\rshal} = \left(\boldsymbol{W}^{[1]}, \boldsymbol{b}^{[1]}, \cdots, \boldsymbol{W}^{[L]}, \boldsymbol{b}^{[L]}\right)$ of the $\mathrm{NN}$, we construct a  $\vtheta^{\prime}_{\rdeep}$ in $\fT_S(\vtheta_{\rshal})$ as follows. Since the activation function $\sigma$ has a non-constant linear segment, there exists an affine subdomain $(a, b)$ associated with $\lambda,\mu \in \mathbb{R}$ ($\lambda\neq0$) such that
$\sigma(x) = \lambda x + \mu$ for $x\in (a, b)$. Let $[x_{\rlow}, x_{\rup}]\subseteq (a, b)$ ($x_{\rlow}\neq x_{\rup}$) be a closed interval of the affine subdomain and $\boldsymbol{\lambda} = \lambda \boldsymbol{1}\in \sR^{m_{\hat{q}}}$, $\boldsymbol{\mu}=\mu\boldsymbol{{1}}\in \sR^{m_{\hat{q}}}$, where $\boldsymbol{1}\in \sR^{m_{\hat{q}}}$ is the all-ones vector. Now we discuss in two cases:

\begin{itemize}
    \item[(1)]  $\min\{m_q, m_{q+1}\} = m_{q+1}$.

Denote training data by $S = \{(\vx_i, \vy_i)\}_{i=1}^{n}$ and $\tilde{\vf}^{[q+1]}_{\vtheta_{\rshal}}=\mW^{[q+1]} \vf^{[q]}_{\vtheta_{\rshal}}+\vb^{[q+1]}  \in \sR^{m_{q+1}}$, 
\begin{align*}
x_{\rmin} = \min_{i\in[n], j\in[m_{q+1}]}\bigl\{\bigl(\tilde{\vf}^{[q+1]}_{\vtheta_{\rshal}}(\vx_i)\bigr)_j\bigr\}, \\
x_{\rmax} = \max_{i\in[n], j\in[m_{q+1}]}\bigl\{\bigl(\tilde{\vf}^{[q+1]}_{\vtheta_{\rshal}}(\vx_i)\bigr)_j\bigr\},
\end{align*}
where $\bigl(\tilde{\vf}^{[q+1]}_{\vtheta_{\rshal}}(\vx_i)\bigr)_j$ is the $j$-th component of $\tilde{\vf}^{[q+1]}_{\vtheta_{\rshal}}(\vx_i)$.

Now we transform the input range $[x_{\rmin}, x_{\rmax}]$ into the affine subdomain $[a, b]$ of the activation function $\sigma$ through an affine transformation. To this end, we further discuss in two cases:
\begin{itemize}
    \item[(i)] case 1: $x_{\rmin} \neq x_{\rmax}.$

Let $\xi = \dfrac{x_{\rup}-x_{\rlow}}{x_{\rmax}-x_{\rmin}}\in \sR$, $\vW^{\prime [\hat{q}]} = \xi\vW^{[q+1]}\in\sR^{m_{\hat{q}}\times m_q}$ and $\vb^{\prime[\hat{q}]}=\xi \vb^{[q+1]} + (x_{\rlow}-\xi x_{\rmin})\boldsymbol{1}\in\sR^{m_{\hat{q}}},$
where $\boldsymbol{1}$ is the all-ones vector.  And then we let $\vW^{\prime[q+1]} = \frac{1}{\lambda \xi} \boldsymbol{I}_d\in\sR^{m_{q+1}\times m_{\hat{q}}}$, where $\vI_d$ is the identity matrix, and $\boldsymbol{b}^{\prime[q+1]} = \boldsymbol{b}^{[q+1]}-\boldsymbol{W'}^{[q+1]}\operatorname{diag}(\boldsymbol{\lambda})\boldsymbol{b'}^{[\hat{q}]}-\boldsymbol{W'}^{[q+1]}\boldsymbol{\mu}\in\sR^{m_{q+1}}.$ Finally, we let $\vtheta^{\prime}_{\rdeep} = \big(\boldsymbol{W}^{[1]}, \boldsymbol{b}^{[1]}, \cdots, \boldsymbol{W}^{[q]}, \boldsymbol{b}^{[q]}, \boldsymbol{W'}^{[\hat{q}]}, \boldsymbol{b'}^{[\hat{q}]}, \boldsymbol{W'}^{[{q+1}]}$, $\boldsymbol{b'}^{[{q+1}]}$, $\cdots, \boldsymbol{W}^{[L]}, \boldsymbol{b}^{[L]}\big)$.
Now we verify that $\vtheta^{\prime}_{\rdeep}\in \fT_S(\vtheta_{\rshal})$, i.e., $\vtheta^{\prime}_{\rdeep}$ satisfies the three conditions of one-layer lifting. Firstly, by the construction of $\vtheta^{\prime}_{\rdeep}$, the local-in-layer condition is satisfied automatically.

Next, for any $j\in [m_{\hat{q}}]$, there exists an affine subdomain $(a, b)$ associated with $\lambda, \mu$ such that the $j$-th component $(\boldsymbol{W'}^{[\hat{q}]}\boldsymbol{f}_{\boldsymbol{\theta^{\prime}_{\rdeep}}}^{[q]}(\boldsymbol{x}) + \boldsymbol{b'}^{[\hat{q}]})_j =\bigl( \xi\vW^{[q+1]}\boldsymbol{f}_{\vtheta_{\rshal}}^{[q]}(\boldsymbol{x})+\xi \vb^{[q+1]} +(x_{\rlow}-\xi x_{\rmin})\boldsymbol{1}\bigr)_j = \bigl(\xi\tilde{\vf}^{[q+1]}_{\vtheta_{\rshal}}+(x_{\rlow}-\xi x_{\rmin})\boldsymbol{1}\bigr)_j \in (a, b)$ 
for any $\vx \in S_{\vx}.$ Thus the layer linearization condition holds.

Finally, by direct calculation, the output preserving condition holds:
\begin{align*}
\left\{\begin{aligned}
&  \boldsymbol{W'}^{[q+1]}\operatorname{diag}(\boldsymbol{\lambda})\boldsymbol{W'}^{[\hat{q}]} = \boldsymbol{W}^{[q+1]},\\
& \boldsymbol{W'}^{[q+1]}\operatorname{diag}(\boldsymbol{\lambda})\boldsymbol{b'}^{[\hat{q}]}+\boldsymbol{W'}^{[q+1]}\boldsymbol{\mu} + \boldsymbol{b'}^{[q+1]} = \boldsymbol{b}^{[q+1]}.   
\end{aligned}\right.   
\end{align*}

Collecting the above results, we prove that $\vtheta^{\prime}_{\rdeep}\in \fT_S(\vtheta_{\rshal})$, i.e., $\fT_S(\vtheta_{\rshal})$ is not empty.

\item[(ii)] case 2: $x_{\rmin} = x_{\rmax}.$

The layer linearization condition can be easily satisfied because the inputs to each neuron remain a constant. Therefore, by setting $\xi \neq 0$ to any nonzero constant, the above construction works for this case, i.e.,  the constructed $\vtheta^{\prime}_{\rdeep}\in \fT_S(\vtheta_{\rshal})$.
\end{itemize}
\item[(2)]  $\min\{m_q, m_{q+1}\} = m_{q}$.

Denote training data by $S = \{(\vx_i, \vy_i)\}_{i=1}^{n}$ and 
\begin{align*}
x_{\rmin} = \min_{i\in[n], j\in[m_{q}]}\bigl\{\bigl(\vf^{[q]}_{\vtheta_{\rshal}}(\vx_i)\bigr)_j\bigr\}, \\
x_{\rmax} = \max_{i\in[n], j\in[m_{q}]}\bigl\{\bigl(\vf^{[q]}_{\vtheta_{\rshal}}(\vx_i)\bigr)_j\bigr\},
\end{align*}
where $\bigl(\vf^{[q]}_{\vtheta_{\rshal}}(\vx_i)\bigr)_j$ is the $j$-th component of $\vf^{[q]}_{\vtheta_{\rshal}}(\vx_i)$.

Also, we can transform the input range $[x_{\rmin}, x_{\rmax}]$ into the affine subdomain $[a, b]$ of the activation function $\sigma$ through an affine transformation. To this end, we also further discuss in two cases:

\begin{itemize}
    \item[(i)] case 1: $x_{\rmin} \neq x_{\rmax}.$

Let $\xi = \dfrac{x_{\rup}-x_{\rlow}}{x_{\rmax}-x_{\rmin}}\in \sR$, $\vW^{\prime [\hat{q}]} = \xi\vI_d\in\sR^{m_{\hat{q}}\times m_q}$ and $\vb^{\prime[\hat{q}]}= (x_{\rlow}-\xi x_{\rmin})\boldsymbol{1}\in\sR^{m_{\hat{q}}}$.  And then we let $\vW^{\prime[q+1]} = \frac{1}{\lambda \xi} \vW^{[q+1]}\in\sR^{m_{q+1}\times m_{\hat{q}}}$, and $\boldsymbol{b}^{\prime[q+1]} = \boldsymbol{b}^{[q+1]}-\boldsymbol{W'}^{[q+1]}\operatorname{diag}(\boldsymbol{\lambda})\boldsymbol{b'}^{[\hat{q}]}-\boldsymbol{W'}^{[q+1]}\boldsymbol{\mu}\in\sR^{m_{q+1}}.$ Finally, we set $\vtheta^{\prime}_{\rdeep} = \big(\boldsymbol{W}^{[1]}, \boldsymbol{b}^{[1]}, \cdots, \boldsymbol{W}^{[q]}, \boldsymbol{b}^{[q]}, \boldsymbol{W'}^{[\hat{q}]}, \boldsymbol{b'}^{[\hat{q}]}, \boldsymbol{W'}^{[{q+1}]}, \boldsymbol{b'}^{[{q+1}]},$ $\cdots, \boldsymbol{W}^{[L]}, \boldsymbol{b}^{[L]}\big)$. We can verify that  $\vtheta^{\prime}_{\rdeep}\in \fT_S(\vtheta_{\rshal})$ similar to the previous case 1.
\item[(ii)] case 2: $x_{\rmin} = x_{\rmax}.$

By setting $\xi \neq 0$ to any nonzero constant, the above construction also works for this case, i.e.,  the constructed $\vtheta^{\prime}_{\rdeep}\in \fT_S(\vtheta_{\rshal})$.
\end{itemize}
\end{itemize}
Therefore, $\fT_{S}(\vtheta_{\rshal})$ is non-empty for any $\vtheta_{\rshal}$, i.e., one-layer lifting exists.
\end{proof}
\subsection{Output preserving and criticality preserving}
Now we prove the following lemma, of which Prop.~\ref{APP:prop-network-preserve} is a direct consequence.
\begin{oldlemma}[\textbf{computation of feature vectors, feature gradients and error vectors}]
\label{APP:lemma-gradients}
Given data $S$, consider an $\mathrm{NN}\bigl(\left\{m_{l}\right\}_{l=0}^{L}\bigr)$ and its one-layer deeper counterpart, $\mathrm{NN}^{\prime}\big(\{m_l^\prime\}$, $l\in \{0, 1, 2, \cdots, q, \hat{q}, q+1, \cdots, L\}\big)$. Let $\fT_S$ denote the one-layer lifting and $\vtheta_{\rshal}$
be any parameter of $\mathrm{NN}$. Then, for any lifted point
$\boldsymbol{\theta}'_{\rdeep} \in \fT_{S}\left(\boldsymbol{\theta}_{\rshal}\right)$, the following conditions hold: there exist $\vlambda, \vmu \in \mathbb{R}^{m^{\prime}_{\hat{q}}}$ such that for any $\vx\in S_{\vx}$,

(i) feature vectors in $\boldsymbol{F}_{\boldsymbol{\theta}'_{\rdeep}}: \boldsymbol{f}_{\boldsymbol{\theta}'_{\rdeep}}^{\left[l\right]}(\vx)=\boldsymbol{f}_{\boldsymbol{\theta}_{\rshal}}^{\left[l\right]}(\vx)$ for $l \in[L]$ and $ \boldsymbol{f}_{\boldsymbol{\theta}'_{\rdeep}}^{[\hat{q}]}(\vx) = \operatorname{diag}(\vlambda) (\boldsymbol{W}^{\prime[\hat{q}]}\boldsymbol{f}_{\vtheta_{\rshal}}^{[q]}(\vx) + \boldsymbol{b}^{\prime[\hat{q}]}) + \boldsymbol{\mu};$

(ii) feature gradients in $\boldsymbol{G}_{\boldsymbol{\theta}'_{\rdeep}}: \boldsymbol{g}_{\boldsymbol{\theta}'_{\rdeep}}^{\left[l\right]}(\vx)=\boldsymbol{g}_{\boldsymbol{\theta}_{\rshal}}^{\left[l\right]}(\vx)$ for $l \in[L]$ and $ \boldsymbol{g}_{\boldsymbol{\theta}'_{\rdeep}}^{[\hat{q}]}(\vx) = \boldsymbol{\lambda};$

(iii) error vectors in $\boldsymbol{Z}_{\boldsymbol{\theta}'_{\rdeep}}: \boldsymbol{z}_{\boldsymbol{\theta}'_{\rdeep}}^{\left[l\right]}(\vx)=\boldsymbol{z}_{\boldsymbol{\theta}_{\rshal}}^{\left[l\right]}(\vx)$ for $l \in [q-1]\cup[q+1: L]$ and $\boldsymbol{z}_{\boldsymbol{\theta}'_{\rdeep}}^{\left[\hat{q}\right]}(\vx) = \bigl(\boldsymbol{W}^{'[q+1]}\bigr)^{\top}\bigl(\boldsymbol{z}_{\boldsymbol{\theta}_{\rshal}}^{[q+1]}(\vx) \circ \boldsymbol{g}_{\boldsymbol{\theta}_{\rshal}}^{[q+1]}(\vx)\bigr),  \boldsymbol{z}_{\boldsymbol{\theta}'_{\rdeep}}^{\left[q\right]}(\vx) = \bigl(\boldsymbol{W}^{'[\hat{q}]}\bigr)^{\top}\bigl(\boldsymbol{z}_{\boldsymbol{\theta}'_{\rdeep}}^{[\hat{q}]}(\vx) \circ \boldsymbol{\lambda}\bigr).$
\end{oldlemma}

\begin{proof}
(i) By the construction of $\vtheta^{\prime}_{\rdeep}$, it is clear that $\vf_{\vtheta^{\prime}_{\rdeep}}^{[l]}(\vx)=\vf_{\vtheta_{\rshal}}^{[l]}(\vx)$ for any $l\in[q]$. And by the definition of one-layer lifting, layer linearization condition is satisfied, i.e., for any $j\in [m_{\hat{q}}]$, there exists an affine subdomain $(a_j, b_j)$ associated with $\lambda_j, \mu_j$ such that the $j$-th component $(\boldsymbol{W'}^{[\hat{q}]}\boldsymbol{f}_{\boldsymbol{\theta^{\prime}_{\rdeep}}}^{[q]}(\boldsymbol{x}) + \boldsymbol{b'}^{[\hat{q}]})_j \in (a_j, b_j)$ for any $\vx \in S_{\vx}.$ Therefore, there exist $\vlambda = [\lambda_1, \lambda_2, \cdots, \lambda_{m_{\hat{q}}}]^\top\in \mathbb{R}^{m^{\prime}_{\hat{q}}}, \vmu = [\mu_1, \mu_2, \cdots, \mu_{m_{\hat{q}}}]^\top\in\mathbb{R}^{m^{\prime}_{\hat{q}}}$ such that for any $\vx \in S_{\vx}$
\begin{equation*}
\begin{aligned}
\boldsymbol{f}_{\vtheta^{\prime}_{\rdeep}}^{[\hat{q}]}(\vx) 
& = \sigma\bigl(\boldsymbol{W}^{'[\hat{q}]}\boldsymbol{f}_{\vtheta^{\prime}_{\rdeep}}^{[q]}(\vx) + \boldsymbol{b}^{'[\hat{q}]}\bigr)\\
& = \boldsymbol{\lambda}\circ (\boldsymbol{W}^{\prime[\hat{q}]}\boldsymbol{f}_{\boldsymbol{\theta}^{\prime}_{\rdeep}}^{[q]}(\vx) + \boldsymbol{b}^{\prime[\hat{q}]}) + \boldsymbol{\mu}\\
& = \boldsymbol{\lambda}\circ (\boldsymbol{W}^{\prime[\hat{q}]}\boldsymbol{f}_{\boldsymbol{\theta}_{\rshal}}^{[q]}(\vx) + \boldsymbol{b}^{\prime[\hat{q}]}) + \boldsymbol{\mu}\\
& =\operatorname{diag}(\vlambda) (\boldsymbol{W}^{\prime[\hat{q}]}\boldsymbol{f}_{\vtheta_{\rshal}}^{[q]}(\vx) + \boldsymbol{b}^{\prime[\hat{q}]}) + \boldsymbol{\mu}.
\end{aligned}
\end{equation*}

By the forward propagation process and the output preserving condition of one-layer lifting, we have
\begin{equation*}
\begin{aligned}
\boldsymbol{f}_{\vtheta^{\prime}_{\rdeep}}^{[q+1]}(\vx) 
&= \sigma\bigl(\boldsymbol{W}^{'[q+1]}\boldsymbol{f}_{\vtheta^{\prime}_{\rdeep}}^{[\hat{q}]}(\vx)  + \boldsymbol{b}^{'[q+1]}\bigr)\\
& = \sigma\bigl(\boldsymbol{W}^{\prime[q+1]}\operatorname{diag}(\boldsymbol{\lambda})\boldsymbol{W}^{\prime[\hat{q}]}\boldsymbol{f}_{\vtheta^{\prime}_{\rdeep}}^{[q]}(\vx) + \boldsymbol{W}^{\prime[q+1]}\operatorname{diag}(\boldsymbol{\lambda})\boldsymbol{b}^{\prime[\hat{q}]}+\boldsymbol{W}^{\prime[q+1]}\boldsymbol{\mu} + \boldsymbol{b}^{\prime[q+1]} \bigr)\\
& = \sigma\bigl(\boldsymbol{W}^{[q+1]}\boldsymbol{f}_{\vtheta_{\rshal}}^{[q]}(\vx)  + \boldsymbol{b}^{[q+1]}\bigr)\\
& = \boldsymbol{f}_{\vtheta_{\rshal}}^{[q+1]}(\vx) .
\end{aligned}
\end{equation*}

And by recursion, we have $\boldsymbol{f}_{\vtheta^{\prime}_{\rdeep}}^{\left[l\right]}(\vx)  = \boldsymbol{f}_{\vtheta_{\rshal}}^{\left[l\right]}(\vx) $ for $l \in[q+1: L]$.

    (ii) By the continuity of the feature function, we know that for any $\vx \in S_{\vx}$ there exists at least a neighborhood of $\vx$ such that the layer linearization condition holds. Thus we have $\boldsymbol{g}_{\boldsymbol{\theta}'_{\rdeep}}^{[\hat{q}]}(\vx)  = \sigma^{(1)}\bigl(\boldsymbol{W}^{'[\hat{q}]}\boldsymbol{f}_{\vtheta^{\prime}_{\rdeep}}^{[q]}(\vx)  + \boldsymbol{b}^{'[\hat{q}]}\bigr) = \vlambda$.
    And the results for feature gradients $\vg_{\vtheta^{\prime}_{\rdeep}}^{[l]}(\vx)$, $l\in[L]$ can be recursively calculated in a similar way.
    
    (iii) 
    By the backpropagation and the above facts in (i), we have $\vz^{[L]}_{\vtheta^{\prime}_{\rdeep}}(\vx) = \nabla\ell(\vf^{[L]}_{\vtheta^{\prime}_{\rdeep}}(\vx),\vy)=\nabla\ell(\vf^{[L]}_{\vtheta_{\rshal}}(\vx) ,\vy)=\vz^{[L]}_{\vtheta_{\rshal}}(\vx)$. So for $l \in [q+1: L]$, it is clear that $\boldsymbol{z}_{\vtheta^{\prime}_{\rdeep}}^{\left[l\right]}(\vx) = \boldsymbol{z}_{\vtheta_{\rshal}}^{\left[l\right]}(\vx)$ and $\boldsymbol{z}_{\vtheta^{\prime}_{\rdeep}}^{\left[\hat{q}\right]}(\vx) = \bigl(\boldsymbol{W}^{'[q+1]}\bigr)^{\top}\bigl(\boldsymbol{z}_{\vtheta_{\rshal}}^{[q+1]}(\vx) \circ \boldsymbol{g}_{\vtheta_{\rshal}}^{[q+1]}(\vx)\bigr)$.
    
    Use the result in (ii), for $l = q$,
    \begin{equation*}
    \begin{aligned}
    \boldsymbol{z}_{\vtheta^{\prime}_{\rdeep}}^{\left[q\right]}(\vx) = \bigl(\boldsymbol{W}^{'[\hat{q}]}\bigr)^{\top}\bigl(\boldsymbol{z}_{\vtheta^{\prime}_{\rdeep}}^{[\hat{q}]}(\vx) \circ \boldsymbol{g}_{\vtheta_{\rdeep}}^{[\hat{q}]}(\vx)\bigr) = \bigl(\boldsymbol{W}^{'[\hat{q}]}\bigr)^{\top}\bigl(\boldsymbol{z}_{\vtheta^{\prime}_{\rdeep}}^{[\hat{q}]}(\vx) \circ \boldsymbol{\lambda}\bigr)
    \end{aligned}.
    \end{equation*}
    Since $\boldsymbol{f}_{\boldsymbol{\theta}'_{\rdeep}}^{\left[l\right]}(\vx)=\boldsymbol{f}_{\boldsymbol{\theta}_{\rshal}}^{\left[l\right]}(\vx)$ and $\boldsymbol{g}_{\boldsymbol{\theta}'_{\rdeep}}^{\left[l\right]}(\vx)=\boldsymbol{g}_{\boldsymbol{\theta}_{\rshal}}^{\left[l\right]}(\vx)$ for $l \in[L]$, we have $\boldsymbol{z}_{\vtheta^{\prime}_{\rdeep}}^{\left[l\right]}(\vx)=\boldsymbol{z}_{\vtheta_{\rshal}}^{\left[l\right]}(\vx)$ for $l\in[q-1]$.
\end{proof}

\begin{oldprop}[\textbf{network properties preserving}]
\label{APP:prop-network-preserve}
Given data $S$, consider an $\mathrm{NN}\bigl(\left\{m_{l}\right\}_{l=0}^{L}\bigr)$ and its one-layer deeper counterpart, $\mathrm{NN}^{\prime}\big(\{m_l^\prime\}$, $l\in \{0, 1, 2, \cdots, q, \hat{q}, q+1, \cdots, L\}\big)$. Let $\fT_S$ denote the one-layer lifting and $\vtheta_{\rshal}$
be any parameter of $\mathrm{NN}$. Then, for any lifted point
$\boldsymbol{\theta}'_{\rdeep} \in \fT_{S}\left(\boldsymbol{\theta}_{\rshal}\right)$, the following conditions hold:
    
    (i) outputs are preserved: $f_{\vtheta'_{\rdeep}}(\vx)=f_{\vtheta_{\rshal}}(\vx)$  for $\vx\in S_{\vx}$;
    
    (ii) empirical risk is preserved: $\RS(\vtheta'_{\rdeep})=\RS(\vtheta_{\rshal})$;
    
        (iii) the network representations are preserved for all layers:  
    $
    \operatorname{span}\bigl\{\bigl\{\bigl(\vf_{\vtheta_{\rdeep}^\prime}^{[\hat{q}]}(\vX)\bigr)_j\bigr\}_{j \in[m^{\prime}_{\hat{q}}]}\cup \{ \boldsymbol{1}\} \bigr\} = \operatorname{span}\bigl\{\bigl\{\bigl(\vf_{\vtheta_{\rshal}}^{[q]}(\vX)\bigr)_j\bigr\}_{j \in[m_{q}]}\cup \{\boldsymbol{1}\} \bigr\},    
    $
    and for the other index $l\in [L]$,
    $
    \operatorname{span}\bigl\{\bigl\{\bigl(\vf_{\vtheta_{\rdeep}^\prime}^{[l]}(\vX)\bigr)_j\bigr\}_{j \in[m^{\prime}_l]}\cup \bigl\{\boldsymbol{1}\bigr\} \bigr\}=\operatorname{span}\bigl\{\bigl\{\bigl(\vf_{\vtheta_{\rshal}}^{[l]}(\vX)\bigr)_j\bigr\}_{j \in[m_l]}\cup \bigl\{\boldsymbol{1}\bigr\} \bigr\},    
    $ where $\vf_{\vtheta}^{[l]}(\vX)=\bigl[\vf_{\vtheta}^{[l]}(\vx_1), \vf_{\vtheta}^{[l]}(\vx_2), \cdots, \vf_{\vtheta}^{[l]}(\vx_n)\bigr]^\top \in \sR^{n\times m^{\prime}_{\hat{q}}}$ and $\boldsymbol{1}\in\sR^{n}$ is the all-ones vector.
\end{oldprop}

\begin{proof}
The properties (i) and (ii) are direct consequences of Lem.~\ref{APP:lemma-gradients}.

(iii) It is clear that for $l\in[L]$
\begin{align*}
    \operatorname{span}\bigl\{\bigl\{\bigl(\vf_{\vtheta_{\rdeep}^\prime}^{[\hat{q}]}(\vX)\bigr)_j\bigr\}_{j \in[m^{\prime}_{\hat{q}}]}\cup \{ \boldsymbol{1}\} \bigr\} = \operatorname{span}\bigl\{\bigl\{\bigl(\vf_{\vtheta_{\rshal}}^{[q]}(\vX)\bigr)_j\bigr\}_{j \in[m_{q}]}\cup \{\boldsymbol{1}\} \bigr\}.
\end{align*}

Since for any $\vx \in S_{\vx}$,
\begin{equation*}
\boldsymbol{f}_{\vtheta^{\prime}_{\rdeep}}^{[\hat{q}]}(\vx) = \boldsymbol{\lambda}\circ (\boldsymbol{W}^{\prime[\hat{q}]}\boldsymbol{f}_{\boldsymbol{\theta}_{\rshal}}^{[q]}(\vx) + \boldsymbol{b}^{\prime[\hat{q}]}) + \boldsymbol{\mu},
\end{equation*}
we have
\begin{align*}
    \operatorname{span}\bigl\{\bigl\{\bigl(\vf_{\vtheta_{\rdeep}^\prime}^{[\hat{q}]}(\vX)\bigr)_j\bigr\}_{j \in[m^{\prime}_{\hat{q}}]}\cup \{ \boldsymbol{1}\} \bigr\} = \operatorname{span}\bigl\{\bigl\{\bigl(\vf_{\vtheta_{\rshal}}^{[q]}(\vX)\bigr)_j\bigr\}_{j \in[m_{q}]}\cup \{\boldsymbol{1}\} \bigr\}.
\end{align*}
Thus we finish the proof.
\end{proof}

\begin{oldprop}[\textbf{criticality preserving}]
\label{APP:criticality-preserve}
Given data $S$, consider an $\mathrm{NN}\bigl(\left\{m_{l}\right\}_{l=0}^{L}\bigr)$ and its one-layer deeper counterpart, $\mathrm{NN}^{\prime}\big(\{m_l^\prime\}$, $l\in \{0, 1, 2, \cdots, q, \hat{q}, q+1, \cdots, L\}\big)$. Let $\fT_S$ denote the one-layer lifting and $\vtheta_{\rshal}$
be any parameter of $\mathrm{NN}$. If $\vtheta_{\rshal}$ of $\mathrm{NN}$ satisfies $\nabla_{\boldsymbol{\theta}} R_{S}\left(\boldsymbol{\theta}_{\rshal}\right)=\boldsymbol{0}$, then $\nabla_{\boldsymbol{\theta}^{\prime}} R_{S}\bigl(\boldsymbol{\theta}'_{\rdeep}\bigr)=\boldsymbol{0}$ for any lifted point $\boldsymbol{\theta}'_{\rdeep} \in \fT_{S}\left(\boldsymbol{\theta}_{\rshal}\right)$.
\end{oldprop}
\begin{proof}
Gradient of loss with respect to network parameters of each layer can be computed from $\boldsymbol{F}$, $\boldsymbol{G}$, and $\boldsymbol{Z}$ as follows
\begin{equation*}
\begin{aligned}
\nabla_{\boldsymbol{W}^{\left[l\right]}} R_{S}(\boldsymbol{\theta}) &=\nabla_{\boldsymbol{W}^{\left[l\right]}} \mathbb{E}_{S} \ell\left(\boldsymbol{f}_{\boldsymbol{\theta}}(\boldsymbol{x}), \boldsymbol{y}\right) =\mathbb{E}_{S}\bigl(\bigl(\boldsymbol{z}_{\boldsymbol{\theta}}^{\left[l\right]}(\vx) \circ \boldsymbol{g}_{\boldsymbol{\theta}}^{\left[l\right]}(\vx)\bigr)\bigl(\boldsymbol{f}_{\boldsymbol{\theta}}^{\left[l-1\right]}(\vx)\bigr)^{\top}\bigr), \\
\nabla_{\boldsymbol{b}^{\left[l\right]}} R_{S}(\boldsymbol{\theta}) &=\nabla_{\boldsymbol{b}^{[l]}} \mathbb{E}_{S} \ell\left(\boldsymbol{f}_{\boldsymbol{\theta}}(\boldsymbol{x}), \boldsymbol{y}\right)=\mathbb{E}_{S}\left(\boldsymbol{z}_{\boldsymbol{\theta}}^{\left[l\right]}(\vx) \circ \boldsymbol{g}_{\boldsymbol{\theta}}^{\left[l\right]}(\vx)\right).
\end{aligned}   
\end{equation*}
Then we have for $l \neq q, \hat{q}, q+1$,
\begin{equation*}
  \nabla_{\boldsymbol{W}'^{\left[l\right]}} R_{S}\left(\boldsymbol{\theta}'_{\rdeep}\right)=\nabla_{\boldsymbol{W}^{\left[l\right]}} R_{S}\left(\boldsymbol{\theta}'_{\rdeep}\right) = \nabla_{\boldsymbol{W}^{\left[l\right]}} R_{S}\left(\boldsymbol{\theta}_{\rshal}\right)=\boldsymbol{0} 
\end{equation*}

and
\begin{equation*}
 \nabla_{\boldsymbol{b}'^{\left[l\right]}} R_{S}\left(\boldsymbol{\theta}'_{\rdeep}\right)= \nabla_{\boldsymbol{b}^{\left[l\right]}} R_{S}\left(\boldsymbol{\theta}'_{\rdeep}\right) = \nabla_{\boldsymbol{b}^{\left[l\right]}} R_{S}\left(\boldsymbol{\theta}_{\rshal}\right)=\boldsymbol{0}.   
\end{equation*}
Also, for $l = q + 1,$
\begin{small}
\begin{align*}
\begin{aligned}
& \nabla_{\boldsymbol{W}'^{\left[q+1\right]}} R_{S}\left(\boldsymbol{\theta}'_{\rdeep}\right) = \mathbb{E}_{S}\left(\left(\boldsymbol{z}_{\boldsymbol{\theta}'_{\rdeep}}^{\left[q+1\right]}(\vx) \circ \boldsymbol{g}_{\boldsymbol{\theta}'_{\rdeep}}^{\left[q+1\right]}(\vx)\right)\left(\boldsymbol{f}_{\boldsymbol{\theta}'_{\rdeep}}^{\left[\hat{q}\right]}(\vx)\right)^{\top}\right)\\
& = \mathbb{E}_{S}\left(\left(\boldsymbol{z}_{\boldsymbol{\theta}_{\rshal}}^{\left[q+1\right]}(\vx) \circ \boldsymbol{g}_{\boldsymbol{\theta}_{\rshal}}^{\left[q+1\right]}(\vx)\right)\left[\sigma\left(\boldsymbol{W'}^{[\hat{q}]}\boldsymbol{f}_{\boldsymbol{\theta}'_{\rdeep}}^{[q]}(\vx) + \boldsymbol{b'}^{[\hat{q}]}\right) \right]^{\top}\right)\\
& = \mathbb{E}_{S}\left(\left(\boldsymbol{z}_{\boldsymbol{\theta}_{\rshal}}^{\left[q+1\right]}(\vx) \circ \boldsymbol{g}_{\boldsymbol{\theta}_{\rshal}}^{\left[q+1\right]}(\vx)\right)\left(\boldsymbol{\lambda} \circ (\boldsymbol{W'}^{[\hat{q}]}\boldsymbol{f}_{\boldsymbol{\theta}_{\rshal}}^{[q]}(\vx) + \boldsymbol{b'}^{[\hat{q}]}) + \boldsymbol{\mu}\right)^{\top} \right)\\
& = \mathbb{E}_{S}\left(\left(\boldsymbol{z}_{\boldsymbol{\theta}_{\rshal}}^{\left[q+1\right]}(\vx) \circ \boldsymbol{g}_{\boldsymbol{\theta}_{\rshal}}^{\left[q+1\right]}(\vx)\right)\left(\operatorname{diag}(\boldsymbol{\lambda}) (\boldsymbol{W'}^{[\hat{q}]}\boldsymbol{f}_{\boldsymbol{\theta}_{\rshal}}^{[q]}(\vx) + \boldsymbol{b'}^{[\hat{q}]}) + \boldsymbol{\mu}\right)^{\top} \right)\\
& = \mathbb{E}_{S}\left(\left(\boldsymbol{z}_{\boldsymbol{\theta}_{\rshal}}^{\left[q+1\right]}(\vx) \circ \boldsymbol{g}_{\boldsymbol{\theta}_{\rshal}}^{\left[q+1\right]}(\vx)\right)\left(\boldsymbol{f}_{\boldsymbol{\theta}_{\rshal}}^{[q]}(\vx)\right)^{\top}\left(\operatorname{diag}(\boldsymbol{\lambda}) (\boldsymbol{W'}^{[\hat{q}]}\right)^{\top}\right)\\
& + \mathbb{E}_{S}\left(\left(\boldsymbol{z}_{\boldsymbol{\theta}_{\rshal}}^{\left[q+1\right]}(\vx) \circ \boldsymbol{g}_{\boldsymbol{\theta}_{\rshal}}^{\left[q+1\right]}(\vx)\right)\left( \operatorname{diag}(\boldsymbol{\lambda})\boldsymbol{b'}^{[\hat{q}]} + \boldsymbol{\mu}\right)^{\top}\right)\\
& = \boldsymbol{0},
\end{aligned}
\end{align*}
\end{small}
and

\begin{align*}
\nabla_{\boldsymbol{b}^{'\left[q+1\right]}} R_{S}(\boldsymbol{\theta}'_{\rdeep})=\mathbb{E}_{S}\left(\boldsymbol{z}_{\boldsymbol{\theta}'_{\rdeep}}^{\left[q+1\right]}(\vx) \circ \boldsymbol{g}_{\boldsymbol{\theta}'_{\rdeep}}^{\left[q+1\right]}(\vx)\right) = \mathbb{E}_{S}\left(\boldsymbol{z}_{\boldsymbol{\theta}_{\rshal}}^{\left[q+1\right]}(\vx) \circ \boldsymbol{g}_{\boldsymbol{\theta}_{\rshal}}^{\left[q+1\right]}(\vx)\right) = \boldsymbol{0}.
\end{align*}

For $l = \hat{q},$
\begin{small}
\begin{equation*}
\begin{aligned}
& \nabla_{\boldsymbol{W}'^{\left[\hat{q}\right]}} R_{S}\big(\boldsymbol{\theta}'_{\rdeep}\big) = \mathbb{E}_{S}\big(\big(\boldsymbol{z}_{\boldsymbol{\theta}'_{\rdeep}}^{\left[\hat{q}\right]}(\vx) \circ \boldsymbol{g}_{\boldsymbol{\theta}'_{\rdeep}}^{\left[\hat{q}\right]}(\vx)\big)\big(\boldsymbol{f}_{\boldsymbol{\theta}'_{\rdeep}}^{\left[q\right]}(\vx)\big)^{\top}\big)\\
& = \mathbb{E}_{S}\big(\big(\operatorname{diag}(\boldsymbol{\lambda})\boldsymbol{z}_{\boldsymbol{\theta}'_{\rdeep}}^{\left[\hat{q}\right]}(\vx)\big)\big(\boldsymbol{f}_{\boldsymbol{\theta}_{\rshal}}^{\left[q\right]}(\vx)\big)^{\top}\big)\\
& = \mathbb{E}_{S}\big(\big(\operatorname{diag}(\boldsymbol{\lambda})(\vW^{'[q+1]})^{\top}\big(\boldsymbol{z}_{\boldsymbol{\theta}'_{\rdeep}}^{\left[q+1\right]}(\vx) \circ \boldsymbol{g}_{\boldsymbol{\theta}'_{\rdeep}}^{\left[q+1\right]}(\vx)\big)\big)\big(\boldsymbol{f}_{\boldsymbol{\theta}_{\rshal}}^{\left[q\right]}(\vx)\big)^{\top}\big)\\
& = \operatorname{diag}(\boldsymbol{\lambda})(\vW^{'[q+1]})^{\top}\mathbb{E}_{S}\big(\big(\boldsymbol{z}_{\boldsymbol{\theta}'_{\rdeep}}^{\left[q+1\right]}(\vx) \circ \boldsymbol{g}_{\boldsymbol{\theta}'_{\rdeep}}^{\left[q+1\right]}(\vx)\big)\big(\boldsymbol{f}_{\boldsymbol{\theta}_{\rshal}}^{\left[q\right]}(\vx)\big)^{\top}\big)\\
& = \boldsymbol{0},
\end{aligned}
\end{equation*}
\end{small}

and
\begin{equation*}
\begin{aligned}
& \nabla_{\boldsymbol{b}^{'\left[\hat{q}\right]}} R_{S}(\boldsymbol{\theta}'_{\rdeep})=\mathbb{E}_{S}\left(\boldsymbol{z}_{\boldsymbol{\theta}'_{\rdeep}}^{\left[\hat{q}\right]}(\vx) \circ \boldsymbol{g}_{\boldsymbol{\theta}'_{\rdeep}}^{\left[\hat{q}\right]}(\vx)\right)\\
& = \operatorname{diag}(\boldsymbol{\lambda})(\vW^{'[q+1]})^{\top}\mathbb{E}_{S}\left(\boldsymbol{z}_{\boldsymbol{\theta}'_{\rdeep}}^{\left[q+1\right]}(\vx) \circ \boldsymbol{g}_{\boldsymbol{\theta}'_{\rdeep}}^{\left[q+1\right]}(\vx)\right)\\
& = \operatorname{diag}(\boldsymbol{\lambda})(\vW^{'[q+1]})^{\top}\mathbb{E}_{S}\left(\boldsymbol{z}_{\boldsymbol{\theta}_{\rshal}}^{\left[q+1\right]}(\vx) \circ \boldsymbol{g}_{\boldsymbol{\theta}_{\rshal}}^{\left[q+1\right]}(\vx)\right)\\
& = \boldsymbol{0}.   
\end{aligned}
\end{equation*}

For $l = q,$
\begin{small}
\begin{equation*}
\begin{aligned}
& \nabla_{\boldsymbol{W}'^{\left[q\right]}} R_{S}\bigl(\boldsymbol{\theta}'_{\rdeep}\bigr) = \mathbb{E}_{S}\bigl(\bigl(\boldsymbol{z}_{\boldsymbol{\theta}'_{\rdeep}}^{\left[q\right]}(\vx) \circ \boldsymbol{g}_{\boldsymbol{\theta}'_{\rdeep}}^{\left[q\right]}(\vx)\bigr)\bigl(\boldsymbol{f}_{\boldsymbol{\theta}'_{\rdeep}}^{\left[q-1\right]}(\vx)\bigr)^{\top}\bigr)\\
& = \mathbb{E}_{S}\bigl(\bigl((\vW^{'[\hat{q}]})^{\top}\bigl(\boldsymbol{z}_{\boldsymbol{\theta}'_{\rdeep}}^{\left[\hat{q}\right]}(\vx) \circ \boldsymbol{g}_{\boldsymbol{\theta}'_{\rdeep}}^{\left[\hat{q}\right]}(\vx)\bigr) \circ \boldsymbol{g}_{\boldsymbol{\theta}'_{\rdeep}}^{\left[q\right]}(\vx)\bigr)\bigl(\boldsymbol{f}_{\boldsymbol{\theta}'_{\rdeep}}^{\left[q-1\right]}(\vx)\bigr)^{\top}\bigr)\\
& = \operatorname{diag}(\boldsymbol{\lambda})(\vW^{'[\hat{q}]})^{\top}(\vW^{'[q+1]})^{\top}\mathbb{E}_{S}\bigl(\bigl(\bigl(\boldsymbol{z}_{\boldsymbol{\theta}'_{\rdeep}}^{\left[q+1\right]}(\vx) \circ \boldsymbol{g}_{\boldsymbol{\theta}'_{\rdeep}}^{\left[q+1\right]}(\vx)\bigr) \circ \boldsymbol{g}_{\boldsymbol{\theta}'_{\rdeep}}^{\left[q\right]}(\vx)\bigr)\bigl(\boldsymbol{f}_{\boldsymbol{\theta}'_{\rdeep}}^{\left[q-1\right]}(\vx)\bigr)^{\top}\bigr)\\
& = (\vW^{[q+1]})^{\top}\mathbb{E}_{S}\bigl(\bigl(\bigl(\boldsymbol{z}_{\boldsymbol{\theta}'_{\rdeep}}^{\left[q+1\right]}(\vx) \circ \boldsymbol{g}_{\boldsymbol{\theta}'_{\rdeep}}^{\left[q+1\right]}(\vx)\bigr) \\ & \quad \circ \boldsymbol{g}_{\boldsymbol{\theta}'_{\rdeep}}^{\left[q\right]}(\vx)\bigr)\bigl(\boldsymbol{f}_{\boldsymbol{\theta}'_{\rdeep}}^{\left[q-1\right]}(\vx)\bigr)^{\top}\bigr)\\
& = (\vW^{[q+1]})^{\top}\mathbb{E}_{S}\bigl(\bigl(\bigl(\boldsymbol{z}_{\boldsymbol{\theta}_{\rshal}}^{\left[q+1\right]}(\vx) \circ \boldsymbol{g}_{\boldsymbol{\theta}_{\rshal}}^{\left[q+1\right]}(\vx)\bigr) \\ & \quad \circ \boldsymbol{g}_{\boldsymbol{\theta}_{\rshal}}^{\left[q\right]}(\vx)\bigr)\bigl(\boldsymbol{f}_{\boldsymbol{\theta}_{\rshal}}^{\left[q-1\right]}(\vx)\bigr)^{\top}\bigr)\\
& = \mathbb{E}_{S}\bigl(\bigl(\boldsymbol{z}_{\boldsymbol{\theta}_{\rshal}}^{\left[q\right]}(\vx) \circ \boldsymbol{g}_{\boldsymbol{\theta}_{\rshal}}^{\left[q\right]}(\vx)\bigr)\bigl(\boldsymbol{f}_{\boldsymbol{\theta}_{\rshal}}^{\left[q-1\right]}(\vx)\bigr)^{\top}\bigr)\\
& = \boldsymbol{0},
\end{aligned}
\end{equation*}
\end{small}

and
\begin{small}
\begin{equation*}
\begin{aligned}
& \nabla_{\boldsymbol{b}^{'\left[q\right]}} R_{S}(\boldsymbol{\theta}'_{\rdeep})=\mathbb{E}_{S}\left(\boldsymbol{z}_{\boldsymbol{\theta}'_{\rdeep}}^{\left[q\right]}(\vx) \circ \boldsymbol{g}_{\boldsymbol{\theta}'_{\rdeep}}^{\left[q\right]}(\vx)\right)\\
& =\mathbb{E}_{S}\left((\vW^{'[\hat{q}]})^{\top}\left(\boldsymbol{z}_{\boldsymbol{\theta}'_{\rdeep}}^{\left[\hat{q}\right]}(\vx) \circ \boldsymbol{g}_{\boldsymbol{\theta}'_{\rdeep}}^{\left[\hat{q}\right]}(\vx)\right) \circ \boldsymbol{g}_{\boldsymbol{\theta}'_{\rdeep}}^{\left[q\right]}(\vx)\right)\\
& = \operatorname{diag}(\boldsymbol{\lambda})(\vW^{'[\hat{q}]})^{\top}(\vW^{'[q+1]})^{\top}\mathbb{E}_{S}\big(\big(\boldsymbol{z}_{\boldsymbol{\theta}'_{\rdeep}}^{\left[q+1\right]}(\vx) \circ \boldsymbol{g}_{\boldsymbol{\theta}'_{\rdeep}}^{\left[q+1\right]}(\vx)\big) \circ \boldsymbol{g}_{\boldsymbol{\theta}'_{\rdeep}}^{\left[q\right]}(\vx)\big)\\
& = (\vW^{[q+1]})^{\top}\mathbb{E}_{S}\left(\left(\boldsymbol{z}_{\boldsymbol{\theta}'_{\rdeep}}^{\left[q+1\right]}(\vx) \circ \boldsymbol{g}_{\boldsymbol{\theta}'_{\rdeep}}^{\left[q+1\right]}(\vx)\right) \circ \boldsymbol{g}_{\boldsymbol{\theta}'_{\rdeep}}^{\left[q\right]}(\vx)\right)\\
& = (\vW^{[q+1]})^{\top}\mathbb{E}_{S}\left(\left(\boldsymbol{z}_{\boldsymbol{\theta}_{\rshal}}^{\left[q+1\right]}(\vx) \circ \boldsymbol{g}_{\boldsymbol{\theta}_{\rshal}}^{\left[q+1\right]}(\vx)\right) \circ \boldsymbol{g}_{\boldsymbol{\theta}_{\rshal}}^{\left[q\right]}(\vx)\right)\\
& = \mathbb{E}_{S}\left(\boldsymbol{z}_{\boldsymbol{\theta}_{\rshal}}^{\left[q\right]}(\vx) \circ \boldsymbol{g}_{\boldsymbol{\theta}_{\rshal}}^{\left[q\right]}(\vx)\right)\\
& = \boldsymbol{0}.   
\end{aligned}
\end{equation*}
\end{small}

Collecting all the above relations, we obtain that $\nabla_{\boldsymbol{\theta^{\prime}}} R_S(\boldsymbol{\theta}'_{\rdeep}) = \boldsymbol{0}$.
\end{proof}

\subsection{Embedding Principle in Depth}
\begin{oldtheorem}[\textbf{embedding principle in depth}]
\label{APP:thm:embedding-principle}
Given data $S$ and an $\mathrm{NN'}\bigl(\left\{m'_{l}\right\}_{l=0}^{L'}\bigr)$, for any parameter $\vtheta_{\rc}$ of any  shallower $\mathrm{NN}\bigl(\left\{m_{l}\right\}_{l=0}^{L}\bigr)$ satisfying $\nabla_{\boldsymbol{\theta}} R_{S}\left(\boldsymbol{\theta}_{\rc}\right)=\boldsymbol{0}$, there exists parameter $\vtheta'_{\rc}$ in the loss landscape of  $\mathrm{NN'}\bigl(\left\{m'_{l}\right\}_{l=0}^{L'}\bigr)$ satisfying the following conditions:

\begin{itemize}
    \item[(i)] $\vf_{\vtheta'_{\rc}}(\vx)=\vf_{\vtheta_{\rc}}(\vx)$ for $\vx\in S_{\vx}$;
    \item[(ii)] $\nabla_{\boldsymbol{\theta}'} R_{S}\left(\boldsymbol{\theta'}_{\rc}\right)=\boldsymbol{0}$.
\end{itemize}
\end{oldtheorem}

\begin{proof}
We prove this theorem by construction using the critical liftings. Let $J = L - L^{\prime}$. The $J$-layer lifting $\fT_S$ is the $J$-step composition of one-layer liftings, say $\fT_S = \fT^J_S\cdots\fT^2_S\fT^1_S$. From Lem.~\ref{APP:existence}, we know one-layer lifting always exists, which leads to the existence of $J$-layer lifting $\fT_S$, i.e., $\fT_S(\vtheta_{c})\neq\emptyset$ for any $\vtheta_{c}$. Now we prove by induction that $J$-layer lifting $\fT_S$ satisfies the properties of output preserving and criticality preserving.

For $J=1$, Prop.~\ref{APP:prop-network-preserve} and Prop.~\ref{APP:criticality-preserve} show that the one-layer lifting satisfies the properties of output preserving and criticality preserving.

Assume that the $(J-1)$-layer lifting satisfies the properties of output preserving and criticality preserving, we want to show that so does the $J$-layer lifting.

From the induction hypothesis, we only need to show that if given two critical liftings $\fT^1_S$ and $\fT^2_S$, then $\fT^2_S\fT^1_S$ also satisfies the properties of output preserving and criticality preserving.
\begin{itemize}
    \item[(i)] $\fT_S^{2} \fT_S^{1}$ satisfies the property of output preserving:

Since $\fT_S^{1}$ satisfies the property of output preserving, then for any $\vx\in S_{\vx}$ and $\vtheta^{\prime}\in\fT_S^1(\vtheta), \boldsymbol{f}_{\vtheta^{\prime}}(\vx)=\boldsymbol{f}_{\boldsymbol{\theta}}(\vx)$. Similarly for $\fT_S^{2}$, we have for any $\vtheta''\in\fT_S^2\fT_S^1(\vtheta)$,  $\boldsymbol{f}_{\vtheta^{''}}(\vx)=\boldsymbol{f}_{\vtheta'}(\vx)$, hence $\boldsymbol{f}_{\vtheta''}(\vx)=\boldsymbol{f}_{\vtheta}(\vx)$, for any $\vx \in S_{\vx}$.
\item[(ii)] $\fT_{2} \fT_{1}$ satisfies the property of criticality preserving:

Since $\vtheta$ is a critical point of $R_{S}(\vtheta)$, for any $\vtheta'\in \fT_S^1(\vtheta)$, $\vtheta'$ is also a critical point of $R_{S}(\vtheta)$. Similarly, for any $\vtheta^{''}\in\fT_S^2(\vtheta')$, $\vtheta''$ is also a critical point of $R_{S}(\vtheta)$, hence for any $\vtheta''\in\fT_S^2\fT_S^1(\vtheta)$, $\vtheta''$ is also a critical point of $R_{S}(\vtheta)$.
\end{itemize}
Therefore, for any $\vtheta'_c\in \fT_S(\vtheta_{c})$, conditions (i) and (ii) are satisfied.
\end{proof}

\subsection{Data dependency}
\begin{oldprop}[\textbf{data dependency of lifting}]
\label{APP:thm:data-depen}
Given data $S$ and $S'$, consider an $\mathrm{NN}\bigl(\left\{m_{l}\right\}_{l=0}^{L}\bigr)$ and its deeper counterpart, $\mathrm{NN}^{\prime}\big(\{m_l^\prime\}_{l=0}^{L'}$. Let $\fT_S$ and $\fT_{S'}$ denote the respective critical liftings and $\vtheta_{\rshal}$
be any parameter of $\mathrm{NN}$. If data $S' \subseteq S$, then $\fT_{S}(\vtheta_{\rshal})\subseteq\fT_{S'}(\vtheta_{\rshal})$.
\end{oldprop}

\begin{proof}
We first assume $\fT_S$ is a one-layer lifting. If $S^\prime \subseteq S$, i.e., dataset $S'$ is a subset of dataset $S$, then we have $S^{\prime}_{\vx}\subseteq S_{\vx}$.
For any $\vtheta'_{\rdeep}\in \fT_S(\vtheta_{\rshal})$ , local-in-layer condition is satisfied regardless of input data. 
Regarding the data-dependent layer linearization condition, we have that for any $j\in [m_{\hat{q}}]$, there exists an affine subdomain $(a_j, b_j)$ associated with $\lambda_j, \mu_j$ such that the $j$-th component $(\boldsymbol{W'}^{[\hat{q}]}\boldsymbol{f}_{\boldsymbol{\theta}_{\rdeep}}^{[q]}(\boldsymbol{x}) + \boldsymbol{b'}^{[\hat{q}]})_j \in (a_j, b_j)$ for any $\vx \in S_{\vx}$, where $\mW'^{[\hat{q}]}$ and $\vb'^{[\hat{q}]}$ are weight and bias of $\vtheta'_{\rdeep}$ at layer $\hat{q}$. Since $S_{\vx}^{\prime}\subseteq S_{\vx}$, for any $\vx \in S^{\prime}_{\vx}$, naturally, $(\boldsymbol{W'}^{[\hat{q}]}\boldsymbol{f}_{\boldsymbol{\theta}_{\rdeep}}^{[q]}(\boldsymbol{x}) + \boldsymbol{b'}^{[\hat{q}]})_j \in (a_j, b_j)$, i.e., $\vtheta'_{\rdeep}$ satisfies layer linearization condition for $S'$ with the same $\vlambda$ and $\vmu$ as for $S$. Therefore, $\vtheta'_{\rdeep}$ also satisfies the output preserving condition for $S'$.
Then, we have $\vtheta_{\rdeep}\in \fT_{S^{\prime}}(\vtheta_{\rshal})$, which leads to $\fT_{S}(\vtheta_{\rshal})\subseteq\fT_{S^{\prime}}(\vtheta_{\rshal})$.

If $\fT_S$ is a composition of critical liftings that satisfy this corollary, say $\fT_S=\fT^2_S\fT^1_S$. We have for any  $\vtheta'_{\rdeep}\in \fT_S(\vtheta_{\rshal})$, there exists $\vtheta_{\mathrm{mid}}\in\fT^1_S(\vtheta_{\rshal})$ such that $\vtheta'_{\rdeep}\in \fT^2_S(\vtheta_{\mathrm{mid}})$. Then $\vtheta'_{\rdeep}\in \fT^2_{S'}(\vtheta_{\mathrm{mid}})$ and $\vtheta_{\mathrm{mid}}\in\fT^1_{S'}(\vtheta_{\rshal})$. Therefore, $\vtheta_{\rdeep}\in\fT^2_{S'}\fT^1_{S'}(\vtheta_{\rshal})=\fT_{S^{\prime}}(\vtheta_{\rshal})$, which leads to $\fT_{S}(\vtheta_{\rshal})\subseteq\fT_{S^{\prime}}(\vtheta_{\rshal})$.
\end{proof}

\begin{oldprop}\label{APP:prop-generalization}
Given data $S$ and an $\mathrm{NN'}\bigl(\left\{m'_{l}\right\}_{l=0}^{L'}\bigr)$, for any parameter $\vtheta_{\rc}$ of any  shallower $\mathrm{NN}\bigl(\left\{m_{l}\right\}_{l=0}^{L}\bigr)$, there exists parameter $\vtheta'_{\rc}$ in the loss landscape of  $\mathrm{NN'}\bigl(\left\{m'_{l}\right\}_{l=0}^{L'}\bigr)$ satisfying that: for any $\vx_i\in S_{\vx}$, there exists a neighbourhood $N(\vx_i)$ of $\vx_i$ such that $\vf_{\vtheta'_{\rc}}(\vx)=\vf_{\vtheta_{\rc}}(\vx)$ for any $\vx\in N(\vx_i)$.
\end{oldprop}

\begin{proof}
We only prove this result for one-layer lifting and similar to Thm.~\ref{APP:thm:embedding-principle}, the result of multi-layer lifting can be easily obtained by induction.
Let $\fT_S$ be any one-layer lifting and $\vtheta'_{\rc}\in \fT_S(\vtheta_{\rc})$.  By the definition of one-layer lifting, layer linearization condition is satisfied, i.e., for any $j\in [m_{\hat{q}}]$, there exists an affine subdomain $(a_j, b_j)$ associated with $\lambda_j, \mu_j$ such that the $j$-th component $(\boldsymbol{W'}^{[\hat{q}]}\boldsymbol{f}_{\boldsymbol{\theta^{\prime}_{\rc}}}^{[q]}(\boldsymbol{x}_i) + \boldsymbol{b'}^{[\hat{q}]})_j \in (a_j, b_j)$ for any $\vx_i \in S_{\vx}.$ Let $\vg(\vx) = \boldsymbol{W'}^{[\hat{q}]}\boldsymbol{f}_{\boldsymbol{\theta^{\prime}_{\rc}}}^{[q]}(\boldsymbol{x}) + \boldsymbol{b'}^{[\hat{q}]}$, and 
    $$\varepsilon = \min{ \big\{\vg(\vx_i)_j - a_j}, b_i - \vg(\vx_i)_j\big\}.
$$
By the continuity of the function $\vg(\vx)$, there exists a $\delta$ neighborhood $N_\delta(\vx_i)$ such that $|\vg(\vx_i)_j - \vg(\vx)_j| < \varepsilon$ for any $\vx\in N_\delta(\vx_i)$, which implies that $\vg(\vx)_j\in(a_j, b_j)$. Therefore, the layer linearization condition indeed holds not only for each training input but also at least a neighbourhood of each training input.

Similar to Lem.~\ref{APP:lemma-gradients}, by recursive we can get the NN output function is actually preserved over a broader area of input space including at least a neighbourhood of each training input. Hence if the training dataset is sufficiently large and representative, then our lifting operator effectively preserves the generalization performance. 
\end{proof}

\begin{oldprop}[\textbf{positive and negative index of inertia preserving}]\label{app:inertia-preserving}
Given data $S$, consider an $\mathrm{NN}\bigl(\left\{m_{l}\right\}_{l=0}^{L}\bigr)$, and its deeper counterpart $\mathrm{NN'}\bigl(\left\{m'_{l}\right\}_{l=0}^{L'}\bigr)$. Let $\fT_S$ denote the corresponding critical lifting and $\vtheta_{\rshal}$
be a critical point of $\mathrm{NN}$. 
For any critical embedding $\fE: \sR^M \to \sR^{M'}$ resulting from $\fT_S$ (i.e., $\mathcal{E}$ is a differentiable point-to-point critical mapping with full column rank Jacobian $\boldsymbol{J}_{\mathcal{E}(\boldsymbol{\theta}_{\text{shal}})}$). Denote $\boldsymbol{\theta}_{\text{deep}}:=\mathcal{E}(\boldsymbol{\theta}_{\text{shal}})$. Then the number of positive and negative eigenvalues of the Hessian matrix $\boldsymbol{H}_S(\boldsymbol{\theta}_{\text{deep}})$ equals the counterparts of $\boldsymbol{H}_S(\boldsymbol{\theta}_{\text{shal}})$.
\end{oldprop}
\begin{proof}
On one hand, because $\fE$ is a critical embedding resulting from critical lifting $\fT_S$, by the output preserving property, the loss value is preserved:
\begin{equation*}
R_S(\vtheta_{\rshal}) = R_S(\fE(\vtheta_{\rshal})) = R_S(\vtheta_{\rdeep}).
\end{equation*}
Given that $\boldsymbol{\theta}_{\text{shal}}$ is a critical point, we have:
\begin{equation*}
\boldsymbol{H}_S(\vtheta_{\rshal}) = \vJ_{\fE(\vtheta_{\rshal})}^\top \boldsymbol{H}_S(\vtheta_{\rdeep})\vJ_{\fE(\vtheta_{\rshal})}.
\end{equation*}
If $\boldsymbol{H}_S\left(\boldsymbol{\theta}_{\rshal}\right)$ has $k_1$ negative eigenvalues $\left\{\lambda_j^{\text {neg }}\right\}_{j=1}^{k_1}$ with associated orthonormal eigenvectors $\left\{\boldsymbol{e}_j^{\text {neg }}\right\}_{j=1}^{k_1}$, then $\left\{\vJ_{\fE(\vtheta_{\rshal})} \boldsymbol{e}_j^{\text {neg }}\right\}_{j=1}^{k_1}$ satisfies, for any $\boldsymbol{e}_j^{\text {neg }}$,
\begin{equation*}
\left(\vJ_{\fE(\vtheta_{\rshal})} \boldsymbol{e}_j^{\mathrm{neg}}\right)^{\top} \boldsymbol{H}_S\left(\vtheta_{\rdeep}\right) \vJ_{\fE(\vtheta_{\rshal})} \boldsymbol{e}_j^{\mathrm{neg}}=\left(\boldsymbol{e}_j^{\mathrm{neg}}\right)^{\top} \boldsymbol{H}_S(\boldsymbol{\theta}_{\rshal}) \boldsymbol{e}_j^{\mathrm{neg}}=\lambda_j^{\mathrm{neg}}<0.
\end{equation*}
By full rankness of $\vJ_{\fE(\vtheta_{\rshal})}$, we have
\begin{equation*}
\operatorname{dim}\left(\operatorname{span}\left(\left\{\vJ_{\fE(\vtheta_{\rshal})} \boldsymbol{e}_j^{\mathrm{neg}}\right\}_{j=1}^{k_1}\right)\right)=k_1.
\end{equation*}
Thus, $\boldsymbol{H}_S\left(\vtheta_{\rdeep}\right)$ has at least $k_1$ negative eigenvalues. 

On the other hand, for any lifted critical point $\boldsymbol{\theta}_{\text{deep}}$, by the definition of critical lifting $\mathcal{T}_S$ (see Definition~\ref{APP:def:one-layer-lifting}), any critical embedding $\mathcal{E}$ resulting from $\mathcal{T}_S$ is injective and there exists a neighborhood $\mathcal{N}(\boldsymbol{\theta}_{\text{deep}})$ of $\boldsymbol{\theta}_{\text{deep}}$ where the layer linearization condition holds. Consequently, within this neighborhood, there exists a differentiable output-preserving merge operator $\mathcal{P}: \mathcal{N}(\boldsymbol{\theta}_{\text{deep}})\subset \mathbb{R}^{M'} \to \mathbb{R}^{M}$ such that $\mathcal{P}(\mathcal{E}(\boldsymbol{\theta}_{\text{shal}})) = \boldsymbol{\theta}_{\text{shal}}$ for any parameter $\boldsymbol{\theta}_{\text{shal}}$. Since $\boldsymbol{J}_{\mathcal{P}(\boldsymbol{\theta}_{\text{deep}})}\boldsymbol{J}_{\mathcal{E}(\boldsymbol{\theta}_{\text{shal}})} = \boldsymbol{I}_{M}$, where $\boldsymbol{I}_{M}$ is the identity matrix, the Jacobian $\boldsymbol{J}_{\mathcal{P}(\boldsymbol{\theta}_{\text{deep}})}$ must have full row rank.

By the output-preserving property of $\mathcal{P}$, for the $\mathcal{E}$-embedded critical point $\boldsymbol{\theta}_{\text{deep}} = \mathcal{E}(\boldsymbol{\theta}_{\text{shal}})$, we have:
$$
R_S(\vtheta_{\rdeep}) = R_S(\fP(\vtheta_{\rdeep})) = R_S(\fP(\fE(\vtheta_{\rshal}))) = R_S(\vtheta_{\rshal}).
$$
Hence,
$$
\boldsymbol{H}_S(\vtheta_{\rdeep}) = \vJ_{\fP(\vtheta_{\rdeep})}^\top \boldsymbol{H}_S(\vtheta_{\rshal})\vJ_{\fP(\vtheta_{\rdeep})}.
$$
Let $\boldsymbol{H}_S(\boldsymbol{\theta}_{\text{shal}}) = \vP^\top \boldsymbol{\Sigma} \vP$ be the eigendecomposition and $\vU = \vP \vJ_{\vP(\boldsymbol{\theta}_{\text{deep}})}\in\sR^{M\times M'}$ have full row rank. Then:
\begin{equation*}
\boldsymbol{H}_S(\vtheta_{\rdeep}) = \vU^\top \Sigma \vU.
\end{equation*}

By augmenting $\boldsymbol{\Sigma}$ with $(M'-M)$ zeros on the diagonal, we obtain $\boldsymbol{\Sigma}'\in\mathbb{R}^{M'\times M'}$. Similarly, by extending $\boldsymbol{U}$ to an invertible matrix $\boldsymbol{U}'\in\mathbb{R}^{M'\times M'}$, we have:
\begin{equation*}
\boldsymbol{H}_S(\vtheta_{\rdeep}) = \vU'^\top \Sigma' \vU'.
\end{equation*}

This represents a congruence transformation of $\boldsymbol{\Sigma}'$. According to Sylvester's law of inertia, congruent matrices have the same inertia signature. Therefore, $\boldsymbol{H}_S(\boldsymbol{\theta}_{\text{deep}})$ and $\boldsymbol{\Sigma}'$ have identical counts of positive and negative eigenvalues. Since $\boldsymbol{\Sigma}'$ has the same inertia signature as $\boldsymbol{H}_S(\boldsymbol{\theta}_{\text{shal}})$ by construction, we conclude that $\boldsymbol{H}_S(\boldsymbol{\theta}_{\text{deep}})$ and $\boldsymbol{H}_S(\boldsymbol{\theta}_{\text{shal}})$ have exactly the same number of positive and negative eigenvalues.

\end{proof}

\begin{oldcor}[\textbf{incremental degeneracy of critical point through lifting}]\label{app:cor:incremental_degeneracy}
Given data $S$, consider an $\mathrm{NN}\bigl(\left\{m_{l}\right\}_{l=0}^{L}\bigr)$, and its deeper counterpart $\mathrm{NN'}\bigl(\left\{m'_{l}\right\}_{l=0}^{L'}\bigr)$. Let $\fT_S$ denote the corresponding critical lifting and $\vtheta_{\rshal}$
be a critical point of $\mathrm{NN}$. 
For any critical embedding $\fE: \sR^M \to \sR^{M'}$ resulting from $\fT_S$ (i.e., $\mathcal{E}$ is a differentiable point-to-point critical mapping with full column rank Jacobian $\boldsymbol{J}_{\mathcal{E}(\boldsymbol{\theta}_{\text{shal}})}$). Denote $\boldsymbol{\theta}_{\text{deep}}:=\mathcal{E}(\boldsymbol{\theta}_{\text{shal}})$.
Then, $\vtheta_{\rdeep}$ possesses $M'- M$ additional degrees of degeneracy in comparison to $\vtheta_{\rshal}$.
\end{oldcor}
\begin{proof}
Given that deeper networks have more parameters, it follows from Prop~\ref{app:inertia-preserving} that the embedded critical point $ \vtheta_{\rdeep}$ exhibits $M' - M$ extra degrees of degeneracy compared to $\vtheta_{\rshal}$.
\end{proof}

\subsection{One-layer residual lifting}
We give the rigorous definition of one-layer residual lifting, which is very similar to one-layer lifting. The only difference is that there is one more item in the output preserving condition due to \emph{the skip connection} (see Fig.~\ref{fig:one-layer_residual_lifting} for illustration.)

\begin{figure*}[htbp]
    \centering
    \includegraphics[width=1.0\textwidth]{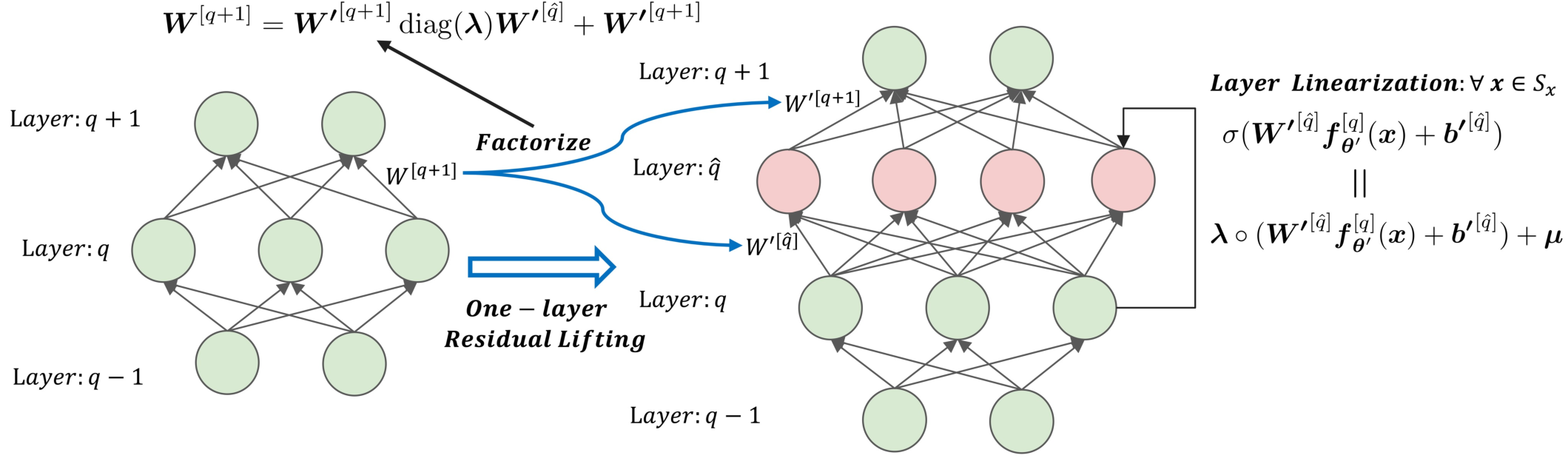} 
\caption{\textbf{Illustration of one-layer residual lifting.} The pink layer is inserted into the left network to get the right network. The input parameters $\vW^{\prime[\hat{q}]}$ and output parameters $\vW^{\prime[q+1]}$ of the inserted layer are obtained by factorizing the input parameters $\vW^{[ q+1]}$ of $(q+1)$-th layer in the left network to satisfy layer linearization and output preserving conditions. }
\label{fig:one-layer_residual_lifting}
\end{figure*}

\begin{olddefinition}[\textbf{one-layer residual lifting}]
\label{APP:def:one-layer-residual-lifting}
Given data $S$, consider an $\mathrm{NN}\bigl(\left\{m_{l}\right\}_{l=0}^{L}\bigr)$ and its one-layer deeper residual counterpart, $\mathrm{NN}^{\prime}\big(\{m_l^\prime\}$, $l\in \{0, 1, 2, \cdots, q, \hat{q}, q+1, \cdots, L\}\big)$, which has a skip connection at the $\hat{q}$-th layer. The one-layer residual lifting, denoted as $\fT_S$,  is a function that transforms any parameter $\vtheta=\left(\boldsymbol{W}^{[1]}, \boldsymbol{b}^{[1]}, \cdots, \boldsymbol{W}^{[L]}, \boldsymbol{b}^{[L]}\right)$ of $\mathrm{NN}$ into a set $\fM$ within the parameter space of $\mathrm{NN}^{\prime}$. Formally, 
$\fM$ (where $\fM:= \fT_S(\vtheta))$ represents a collection of all possible parameters $\vtheta'$ of $\mathrm{NN}^{\prime}$
that satisfying the following three conditions:

(i) local-in-layer condition: weights of each layer in $\mathrm{NN}^{\prime}$ are inherited from $\mathrm{NN}$ except for layer $\hat{q}$ and $q+1$, i.e.,  
\begin{equation*}
   \left\{ \begin{aligned}
& \vtheta^{\prime}|_l = \boldsymbol{\theta} |_l,\quad \text{for}\quad  l \in [q]\cup [q+2:L],\\
& \vtheta^{\prime}|_{\hat{q}} = \bigl(\boldsymbol{W'}^{[\hat{q}]},  \boldsymbol{b'}^{[\hat{q}]}\bigr)\in\mathbb{R}^{m'_{\hat{q}}\times m'_{q-1}}\times \mathbb{R}^{m'_{\hat{q}}},\\
& \vtheta^{\prime}|_{q+1} = \bigl(\boldsymbol{W'}^{[q+1]},  \boldsymbol{b'}^{[q+1]}\bigr)\in\mathbb{R}^{m'_{q+1}\times m'_{\hat{q}}}\times \mathbb{R}^{m'_{q+1}},\\
\end{aligned}
\right. 
\end{equation*}

(ii) layer linearization condition: for any $j\in [m_{\hat{q}}]$, there exists an affine subdomain $(a_j, b_j)$ of $\sigma$ associated with $\lambda_j, \mu_j$ such that the $j$-th component $(\boldsymbol{W'}^{[\hat{q}]}\boldsymbol{f}_{\boldsymbol{\theta}'}^{[q]}(\boldsymbol{x}) + \boldsymbol{b'}^{[\hat{q}]})_j \in (a_j, b_j)$ for any $\vx \in S_{\vx}.$

(iii) output preserving condition: 
\begin{align*}
\left\{\begin{aligned}
&  \boldsymbol{W'}^{[q+1]}\operatorname{diag}(\boldsymbol{\lambda})\boldsymbol{W'}^{[\hat{q}]} + \boldsymbol{W'}^{[q+1]} = \boldsymbol{W}^{[q+1]},\\
& \boldsymbol{W'}^{[q+1]}\operatorname{diag}(\boldsymbol{\lambda})\boldsymbol{b'}^{[\hat{q}]}+\boldsymbol{W'}^{[q+1]}\boldsymbol{\mu} + \boldsymbol{b'}^{[q+1]} = \boldsymbol{b}^{[q+1]}.
\end{aligned}\right.   
\end{align*}
where $\vlambda = [\lambda_1, \lambda_2, \cdots, \lambda_{m_{\hat{q}}}]^\top\in \mathbb{R}^{m^{\prime}_{\hat{q}}}, \vmu = [\mu_1, \mu_2, \cdots, \mu_{m_{\hat{q}}}]^\top\in\mathbb{R}^{m^{\prime}_{\hat{q}}}$, and $\operatorname{diag}(\boldsymbol{\lambda})$ denotes the diagonal matrix formed by vector $\vlambda$.
\end{olddefinition}
As with one-layer lifting, the properties of output preserving and criticality preserving as well as data-dependency are the same for one-layer residual lifting.

\subsection{Centered kernel alignment}
Consider $\vX \in \mathbb{R}^{n \times m_{1}}$ and $\vY \in \mathbb{R}^{n \times m_{2}}$, which represent two layers each containing $m_{1}$ and $m_{2}$ neurons respectively, mapped to the identical set of $n$ instances. The Gram matrices, $\boldsymbol{K}=\boldsymbol{X} \boldsymbol{X}^{\top}$ and $\boldsymbol{L}=\boldsymbol{Y} \boldsymbol{Y}^{\top}$, are $n \times n$ in dimension and each of their elements signifies the similarity between two instances based on the representations in $\boldsymbol{X}$ or $\boldsymbol{Y}$.

The centering matrix is given by $\boldsymbol{H}=\boldsymbol{I}_{n}-\frac{1}{n} \boldsymbol{1 1}^{\top}$. Consequently, the matrices $\boldsymbol{K}^{\prime}=\boldsymbol{H} \boldsymbol{K} \boldsymbol{H}$ and $\boldsymbol{L}^{\prime}=\boldsymbol{H} \boldsymbol{L} \boldsymbol{H}$ correspond to similarity matrices where column and row means have been subtracted.

The Hilbert-Schmidt Independence Criterion (HSIC) quantifies the similarity of these centered similarity matrices by converting them into vectors and computing the dot product between these vectors, $\mathrm{HSIC}_{0}(\boldsymbol{K}, \boldsymbol{L})=\operatorname{vec}\left(\boldsymbol{K}^{\prime}\right) \cdot \operatorname{vec}\left(\boldsymbol{L}^{\prime}\right) /(n-$ $1)^{2}$. HSIC remains invariant under orthogonal transformations of the representations and, consequently, under permutation of neurons, but lacks invariance to scaling of the original representations.

Centered Kernel Alignment (CKA) is used to further normalize HSIC to yield a similarity metric in the range of 0 to 1 that remains invariant to isotropic scaling, $$ \mathrm{CKA}(\boldsymbol{K}, \boldsymbol{L})=\frac{\operatorname{HSIC}_{0}(\boldsymbol{K}, \boldsymbol{L})}{\sqrt{\mathrm{HSIC}_{0}(\boldsymbol{K}, \boldsymbol{K}) \mathrm{HSIC}_{0}(\boldsymbol{L}, \boldsymbol{L})}}. $$

The research by~\cite{kudugunta2019investigating} demonstrated that linear CKA is a reliable measure for identifying architecturally corresponding layers when measured between layers of architecturally identical networks that were trained from distinct random initializations. Indeed, linear CKA reflects the degree of linear correlation between representations across layers.
\begin{oldprop}[\textbf{CKA and layer linearization}]
\label{APP-prop-CKA}
Let $\vX \in \mathbb{R}^{n \times m_{1}}$ and $\vY \in \mathbb{R}^{n \times m_{2}}$ contain representations of two layers, one with $m_{1}$ neurons and another $m_{2}$ neurons, to the same set of $n$ examples. If the linear CKA between the two layers equals 1, then there exists $\vW\in\sR^{m_1\times m_2}, \vb\in \sR^{1\times m_2}$ such that
\begin{equation*}
    \vY = \vX \vW + \vb.
\end{equation*}
\end{oldprop}

\begin{proof}
Denote $\vX^\prime = \vH\vX, \vY^\prime = \vH\vY$, where $\vH = \vI_n - \frac{1}{n}\boldsymbol{1}\boldsymbol{1}^{\top}$. Then $\vK^\prime = \vX^\prime \vX'^{\top}, \vL^\prime = \vY^\prime \vY'^{\top}$.
Notice that 
$$
\left\langle\operatorname{vec}\left(\vX' \vX'^{\mathrm{T}}\right), \operatorname{vec}\left(\vY' \vY'^{\mathrm{T}}\right)\right\rangle=\operatorname{tr}\left(\vX' \vX'^{\mathrm{T}} \vY' \vY'^{\mathrm{T}}\right)=\operatorname{tr}\left(\vX'^{\mathrm{T}} \vY' \vY'^{\mathrm{T}}\vX'\right)\\
=\left\|\vY'^{\mathrm{T}} \vX'\right\|_{\mathrm{F}}^{2}.
$$
Linear CKA is equivalent to the cosine similarity:
$$
\operatorname{CKA}(\vX\vX^T, \vY\vY^T) = \frac{\left\|\vY'^{\mathrm{T}} \vX'\right\|_{\mathrm{F}}^{2}} {\left\|\vX'^{\mathrm{T}} \vX'\right\|_{\mathrm{F}}\left\|\vY'^{\mathrm{T}} \vY'\right\|_{\mathrm{F}}} = \frac{\left\langle\operatorname{vec}\left(\vX' \vX'^{\mathrm{T}}\right), \operatorname{vec}\left(\vY' \vY'^{\mathrm{T}}\right)\right\rangle}{\sqrt{\left\langle\operatorname{vec}\left(\vX' \vX'^{\mathrm{T}}\right), \operatorname{vec}\left(\vX' \vX'^{\mathrm{T}}\right)\right\rangle} \sqrt{\left\langle\operatorname{vec}\left(\vY' \vY'^{\mathrm{T}}\right), \operatorname{vec}\left(\vY' \vY'^{\mathrm{T}}\right)\right\rangle}}.
$$
If $\operatorname{CKA}(\vX\vX^T, \vY\vY^T) = 1$, then there exists $\alpha\geq 0$ such that
$$
\vY'\vY'^\top = \alpha \vX'\vX'^\top.
$$
From this, we can conclude $\vX'$ and $\vY'$ share the same column space. Therefore, there exists a matrix $\vW\in \sR^{m_1\times m_2}$ such that 
$$\vY' = \vX'\vW.$$ 
As $\vX' = \vX - \frac{1}{n}\boldsymbol{1}\boldsymbol{1}^T \vX$ and $\vY' = \vY - \frac{1}{n}\boldsymbol{1}\boldsymbol{1}^T \vY$, we can write $\vY' = \vX'\vW$ as:
$$
\vY - \frac{1}{n}\boldsymbol{1}\boldsymbol{1}^T \vY = \vX\vW - \frac{1}{n}\boldsymbol{1}\boldsymbol{1}^T \vX\vW.
$$
Denote 
$$\vb = \frac{1}{n}\boldsymbol{1}^\top(\vY-\vX\vW)\in \sR^{1\times m_2},$$
then we have
$$\vY = \vX\vW + \vb,$$
which completes the proof.
\end{proof}

\section{Supplementary experiments}\label{app:sup-exp}
\renewcommand\thefigure{B\arabic{figure}} 
\setcounter{figure}{0} 
In this section, we present the supplementary experiments mentioned in the main text.

\begin{figure}[h]
	\centering
	\subfigure[Training loss]{\includegraphics[height=0.26\textwidth]{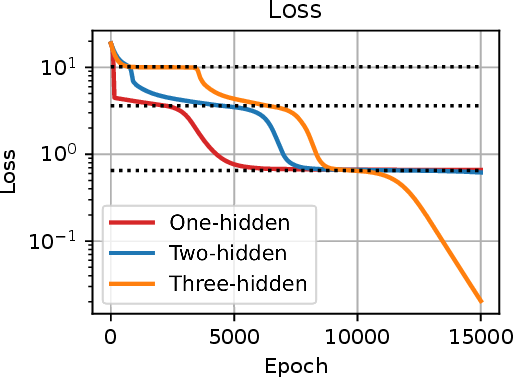}}\hspace{1.5cm}
	\subfigure[Output corresponding to  the second dotted line in (a)]{\includegraphics[height=0.26\textwidth]{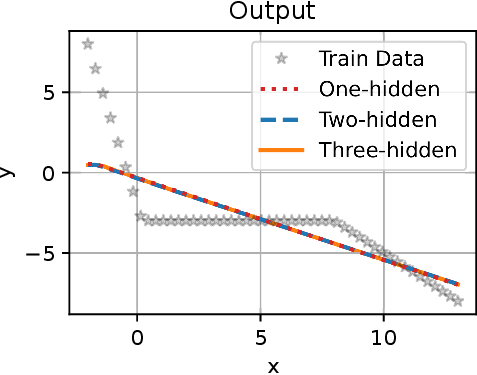}}\hspace{1.5cm}
	\subfigure[Output corresponding to the third dotted line in (a)]{\includegraphics[height=0.26\textwidth]{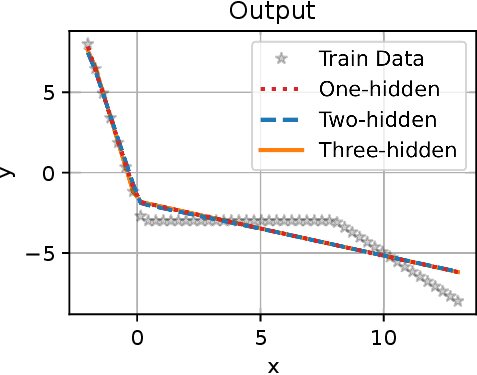}}\hspace{1.5cm}
	\subfigure[Evolution of $\operatorname{MPC}$]{\includegraphics[height=0.26\textwidth]{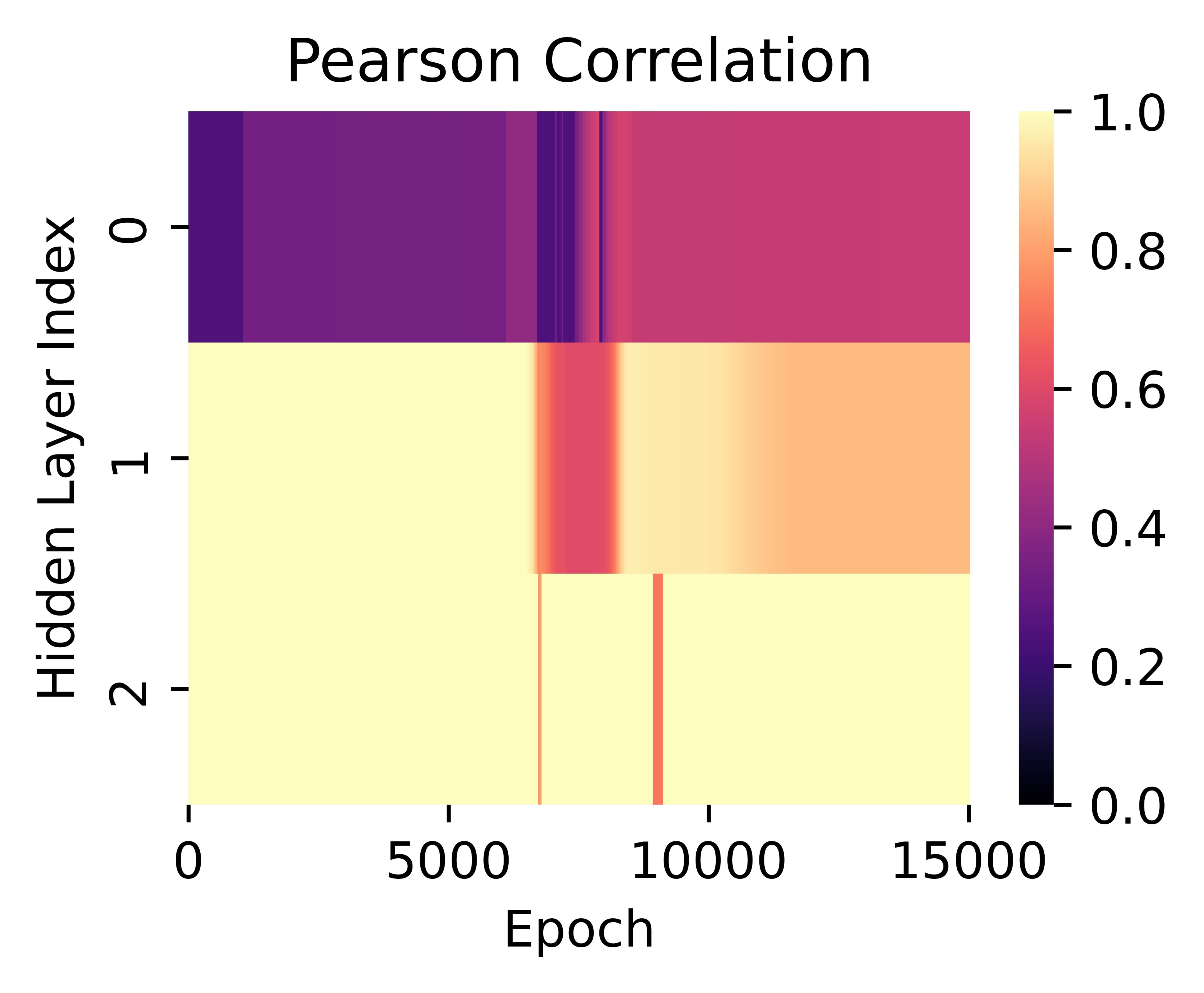}}
	\caption{\textbf{Deep ReLU neural networks encounter lifted critical points during training.} (a) The training loss trajectory for ReLU NN of different depths with $50$ neurons in each hidden layer  on training data in (b). (b, c) The output functions of NNs with different depths at the same loss values indicated by (b) the second horizontal dotted line or (c) the third horizontal dotted line in (a). (d) The extent of layer linearization for all hidden layers during the training process of three-hidden-layer NN. 
	}
	\label{fig:ReLU}
\end{figure} 

\begin{figure}[h]
\centering
\subfigure[Training loss]{
\includegraphics[height=0.26\textwidth]{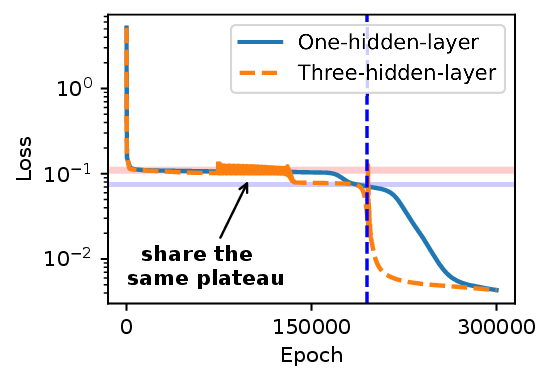}
}$\hspace{1.5cm}$
\subfigure[Output corresponding to the red span in (a)]{
\includegraphics[height=0.26\textwidth]{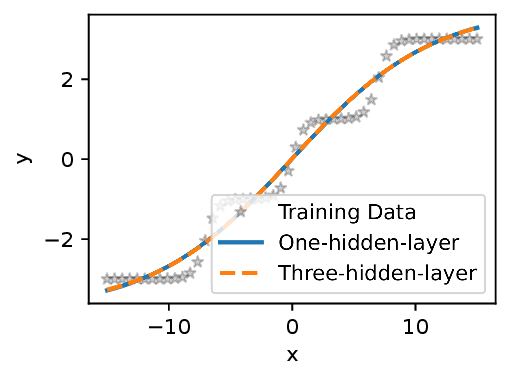}
}
$\hspace{1.5cm}$
\subfigure[Output corresponding to the blue span in (a)]{
\includegraphics[height=0.26\textwidth]{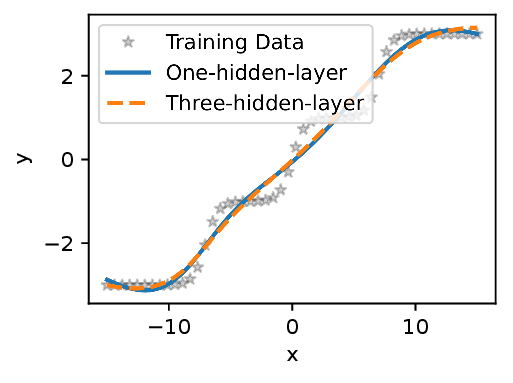}
}
$\hspace{1.5cm}$
\subfigure[Evolution of $\operatorname{MPC}$]{
\includegraphics[height=0.26\textwidth]{figures/ResNet_pear.eps}
}
\caption{\textbf{Deep residual-connected neural networks encounter lifted critical points during training.} (a) The training loss for NNs of different depths with $50$ neurons in each hidden layer for training data in (b). (b, c) The output functions of NNs with different depths at the same loss values indicated by (b) the first horizontal dotted line or (c) the second horizontal dotted line. (d) The extent of layer linearization for all hidden layers during the training process of three-hidden-layer residual NN.}
\label{fig:res-nonlinear}
\end{figure}

\paragraph{Dependence of layer linearization on initialization scale and training data size.} The layer linearization of a network is influenced by its initialization scale and the size of the training data. When initialized with a small enough scale, the network is likely to operate in the linear region during the early stages of training, leads to often encounter the lifted critical point. To investigate the impact of initialization scale on layer linearization, we train a 10-layer network with different initializations on the Fashion-MNIST dataset and measure the extent of linearization of each hidden layer post-training. As depicted in Fig.~\ref{fig:Initialization}(a), with increasing initialization scale from left to right, the degree of non-linearity within each hidden layer also rises. Additionally, as shown in Fig.~\ref{fig:Initialization}(b), for a fixed initialization scale, the degree of nonlinearity across the network's hidden layers increases as the size of the training dataset expands from 500 to 5,000 to 50,000. This finding suggest that larger dataset adds to the difficulty of layer linearization and potentially facilitate a reduction in critical manifolds, thereby enhancing optimization. This is further corroborated by our experiment on simpler datasets, as depicted in Fig.~\ref{fig:data-dependency2} in the main text.

\begin{figure}[htbp]
\centering
\subfigure[Initialization scale]{
\includegraphics[height=0.26\textwidth]{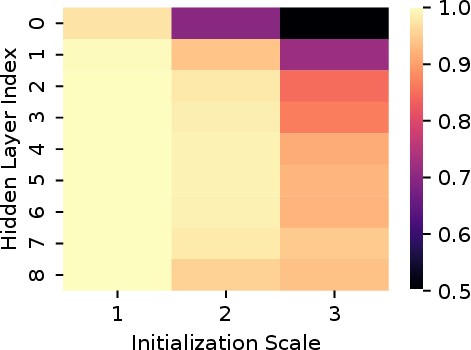}
}
$\hspace{1.5cm}$
\subfigure[Training data size]{
\includegraphics[height=0.26\textwidth]{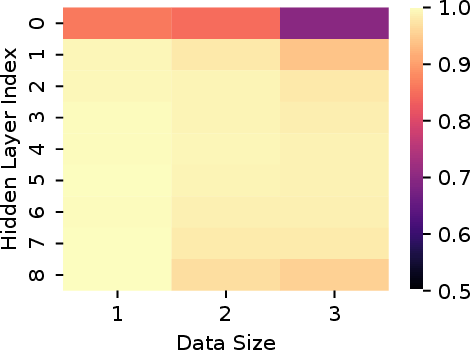}
}
\caption{\textbf{Dependence of layer linearization on initialization scale and training data size.} (a) The extent of layer linearization across all hidden layers corresponding to initialization scales drawn from a Gaussian distribution with a mean of 0 and standard deviations of 0.003, and 0.005, 0.02, arranged from left to right. (b) The extent of layer linearization across all hidden layers in relation to training data sizes of 500, 5,000, and 50,000, arranged from left to right.}
\label{fig:Initialization}
\end{figure}

\section{Details of experiments}
\label{app:tra-details}
For the experiment of Iris dataset (Fig.~\ref{fig:syntraining}(a, b)), we use ReLU as the activation function and the mean square error (MSE) as the loss function. We use full-batch gradient descent with learning rate $0.001$ to train NNs for $100000$ epochs. The width is $50$ for each hidden layer. The initial distribution of all parameters follows a Gaussian distribution with a mean of $0$ and a variance of $0.07$. For the MNIST dataset shown in Fig.~\ref{fig:syntraining}(c, d), we randomly select 500 images to constitute the training set, employing full batch gradient training, with MSE serving as the loss function. Remark that, the phenomenon in Fig.~\ref{fig:syntraining} are similar for different activation functions.

For the 1-D experiments in Fig.~\ref{fig:3-hidden-nonlinear}, Fig.~\ref{fig:different-BN-init}, Fig.~\ref{fig:data-dependency2}, and Fig.~\ref{fig:res-nonlinear}, we use tanh as the activation function and MSE as the loss function. We use full-batch gradient descent with learning rate $0.01$ to train NNs with width $50$ for each hidden layer. The initial distribution of all parameters follows a Gaussian distribution with a mean of $0$ and a variance of $0.01$.

For the experiment of MNIST classification (Fig.~\ref{fig:MNIST-exp-redo}), we use tanh as the activation function and the cross-entropy as the loss function. We use stochastic gradient descent with batch size $1000$ and learning rate $0.001$ to train NNs for 100 epochs. The width is $50$ for each hidden layer. The initial distribution of all parameters follows a Gaussian distribution with a mean of 0 and a variance of 0.05.

For the experiment of Fashion-MNIST classification (Fig.~\ref{fig:Initialization}), we use tanh as the activation function and the cross-entropy as the loss function. We use stochastic gradient descent with batch size $512$ and learning rate $0.01$ to train NNs for 50 epochs. The width is $100$ for each hidden layer.

For the 1-D experiments in Figs.~\ref{fig:ReLU}, we use ReLU as the activation function and MSE as the loss function. We use full-batch gradient descent with learning rate $0.001$ to train NNs with width $50$ for each hidden layer. The initial distribution of all parameters follows a Gaussian distribution with a mean of $0$ and a variance of $0.005$.

For the experiments in Fig.~\ref{fig:eigenvalues-new}, we employ a comprehensive methodology to ascertain the eigenvalues of the Hessian matrix at empirical critical points, comprising the following steps:
(1) Initially, we approximate the probable interval of critical points by observing regions where the loss diminishes very slowly. We then choose the point with the least parameter derivative (using the $L_1$ norm) as our empirical critical point. At this empirical critical point, the $L_1$ norm of the derivative of the loss function hovers around $10^{-4}$, which is acceptably small.
(2) To accurately determine the eigenvalues of a Hessian matrix with a large condition number, we perform 100 random orthogonal similarity transformations on the matrix. We derive a more reliable set of eigenvalues by averaging the outcomes from these 100 trials.
(3) For a clearer distinction between significant and non-significant eigenvalues, we identify locations where there are evident gaps in eigenvalue magnitudes to differentiate between zero and non-zero eigenvalues. (4) To further ensure that the empirical degeneracy is valid, we meticulously examined the state of each neuron and the effective rank of the Hessian matrix. For instance, in Fig.~\ref{fig:eigenvalues-new} (a), for the ReLU network with a single hidden layer containing only two neurons, one of the neurons remains inactive throughout the training dataset. Consequently, the eigendirections associated with its parameters should be null. We observed that the effective rank of the Hessian matrix is 3, implying that there indeed are three non-zero eigenvalues. This aligns with the number of non-zero eigenvalues suggested by the evident gaps identified in step (3).

For activation functions with strong nonlinearity near zero (e.g. ReLU), we first removes those ``zero-neurons'' whose input and output weights are reasonably small to avoid their interference to the measure of layer linearization.

Remark that, although Fig.~\ref{fig:syntraining}, Fig.~\ref{fig:3-hidden-nonlinear} and Fig.~\ref{fig:3-hidden-nonlinear-iris} are case studies each based on a random trial, similar phenomenon can be easily observed as long as the initialization variance is properly small, i.e., far away from the linear/kernel/NTK regime.

\end{document}